\documentclass{article}
\usepackage[dvipsnames]{xcolor}

\usepackage{graphicx}
\usepackage[export]{adjustbox}
\usepackage[numbers,sort&compress]{natbib}
\usepackage{iclr2027_conference,times}
\usepackage{tikz}
\usepackage{harpoon}
\usepackage{hyperref}
\usetikzlibrary{shapes.geometric, arrows.meta}
\usepackage[utf8]{inputenc}
\usepackage[T1]{fontenc}
\usepackage{amsmath,amssymb,amsthm,mathtools}
\usepackage{hyperref}
\usepackage{graphicx}
\usepackage{booktabs}
\usepackage{algorithm}
\usepackage{algpseudocode}
\usepackage{bm}
\usepackage{xcolor}
\usepackage{empheq}
\usepackage{enumitem}
\usepackage{placeins}
\usepackage{multirow}
\usepackage{dsfont}
\usepackage{microtype}
\hypersetup{colorlinks=true, linkcolor=blue!50!black, citecolor=blue!50!black}
\usepackage{makecell}
\usepackage{wrapfig2}

\newtheorem{theorem}{Theorem}

\newtheorem{lemma}{Lemma}
\theoremstyle{remark}
\newtheorem{remark}{Remark}
\theoremstyle{plain}

\newif\ifarxiv
\arxivfalse

\newcommand{\R}{\mathbb{R}}
\newcommand{\E}{\mathbb{E}}
\newcommand{\N}{\mathcal{N}}
\newcommand{\Var}{\mathrm{Var}}
\newcommand{\dd}{\mathrm{d}}
\newcommand{\ESS}{\mathrm{ESS}}

\newcommand{\ind}[1]{\mathbf{1}\{#1\}}
\newcommand{\inds}[1]{\mathbf{1}_{#1}}

\newcommand{\Ev}{E}

\newcommand{\pev}[1]{p_0[#1]}
\newcommand{\hpev}[1]{\hat{p}_0[#1]}

\newcommand{\pE}{\pev{\Ev}}

\newcommand{\hpE}{\hpev{\Ev}}

\newcommand{\pz}{p_0}
\newcommand{\pt}{p_t}
\newcommand{\piz}{\pi_0}
\newcommand{\Zz}{Z_0}

\newcommand{\qz}{q_0}
\newcommand{\pit}{\pi_t}

\makeatletter
\newcommand{\dirharp@gap}{0ex}
\newcommand{\dirharp@scale}{0.8}
\newcommand{\dirharp}[2]{\mathpalette{\dirharp@{#1}}{#2}}
\newcommand{\dirharp@}[3]{%
  \sbox\z@{$\m@th#2#3$}%
  \ooalign{%
    \hfil$\m@th#2#3$\hfil\cr
    \hfil\raisebox{\dimexpr\ht\z@+\dirharp@gap\relax}{%
      \scalebox{\dirharp@scale}{$\m@th\scriptstyle#1$}}\hfil\cr}}
\makeatother
\newcommand{\fwdh}[1]{\dirharp{\rightharpoonup}{#1}}
\newcommand{\bwdh}[1]{\dirharp{\leftharpoonup}{#1}}

\newcommand{\mut}{\mu_t}

\newcommand{\Wb}{\bwdh{W}}
\newcommand{\sth}{s_\theta}

\newcommand{\xht}{\hat{x}_{0|t}}
\newcommand{\xh}{\xht}
\newcommand{\Sig}{\Sigma_t}
\newcommand{\sbar}{\bar\sigma}

\newcommand{\pbwdk}[2]{\bwdh{P}_{#1\mid #2}}

\newcommand{\pdata}{p_{\mathrm{data}}}
\newcommand{\pref}{p_{\mathrm{ref}}}
\newcommand{\pEdata}{\pdata[\Ev]}
\newcommand{\hpMC}{\hat{p}^{\mathrm{MC}}_0[\Ev]}
\newcommand{\Pk}{\bwdh{P}_{k-1\mid k}}
\newcommand{\Qk}{\bwdh{Q}_{k-1\mid k}}
\newcommand{\Pkr}{\bwdh{P}^{\star}_{k-1\mid k}}

\newcommand{\xhk}{\hat{x}_{0\mid k}}
\newcommand{\Sigk}{\Sigma_k}

\newcommand{\om}{\zeta}

\newcommand{\rz}{r_0}
\newcommand{\rt}{r_t}
\newcommand{\rstar}{r^{\star}_t}
\newcommand{\rs}[1]{r^{\star}_{#1}}
\newcommand{\rsoft}{r_{0,\delta}}
\newcommand{\pisoft}{\pi_{0,\delta}}
\newcommand{\Zsoft}{Z_{0,\delta}}

\newcommand{\rd}{\rsoft}

\newcommand{\hg}[1]{h_{#1}}

\newcommand{\Fk}{\fwdh{F}_{k\mid k-1}}
\newcommand{\Bk}{\bwdh{B}_{k-1\mid k}}
\newcommand{\gk}{g_{k-1\mid k}}

\newcommand{\Pm}{\mathbb{P}}
\newcommand{\Qm}{\mathbb{Q}}

\newcommand{\Pmr}{\mathbb{P}_{\rd}}

\newcommand{\Pmk}{\mathbb{P}_{0:K}}
\newcommand{\Qmk}{\mathbb{Q}_{0:K}}
\newcommand{\Pms}{\mathbb{P}^{\star}_{0:K}}
\newcommand{\PmE}{\mathbb{P}^{\Ev}_{0:K}}
\newcommand{\pis}[1]{\pi^{\star}_{#1}}

\newcommand{\dPQ}{\frac{\dd\Pm}{\dd\Qm}}

\newcommand{\taul}{\tau_\ell}
\newcommand{\tauh}{\tau_h}

\newcommand{\pfwd}{\fwdh{p}}
\newcommand{\pbwd}{\bwdh{p}}
\newcommand{\qfwd}{\fwdh{q}}
\newcommand{\qbwd}{\bwdh{q}}
\newcommand{\Ez}{\mathcal{E}_0}
\newcommand{\Eg}{\mathcal{E}_\gamma}
\newcommand{\Dmap}{D_{\mathrm{map}}}
\arxivtrue
\iclrfinalcopy
\let\cite\citep
\usepackage{comment}
\ificlrfinal
  \newenvironment{ack}{\section*{Acknowledgments}}{}
\else
  \excludecomment{ack}
\fi

\title{Steering Diffusion Models to Rare Events with Sequential Monte Carlo}
\author{%
  Aavash Subedi\thanks{Correspondence to \texttt{subedi@stats.ox.ac.uk}.} \\
  Department of Statistics \\
  University of Oxford \\
  \And
  Tim Reichelt \\
  Department of Physics \\
  University of Oxford \\
  \And
  Christopher Williams \\
  Department of Statistics \\
  University of Oxford \\
  \And
  Philip Stier \\
  Department of Physics \\
  University of Oxford \\
  \And
  Yee Whye Teh \\
  Department of Statistics \\
  University of Oxford \\
  \And
  Saifuddin Syed \\
  Department of Statistics \\
  University of British Columbia \\
}

\begin{document}

\addtocontents{toc}{\protect\setcounter{tocdepth}{-5}}

\maketitle
\lhead{Preprint.}
\begin{abstract}

Diffusion models are increasingly used as surrogates for expensive simulators in weather prediction, molecular dynamics, and materials design. In these models, computing the probability $p_0[E]$ of an event $E$ is difficult, especially when the event of interest is \emph{rare}. A stable estimate using Monte Carlo becomes computationally intractable, requiring a growing sample size $\propto\!1/p_0[E]$ to compensate for an increasing rarity. In this paper, we present \textbf{Di}ffusion \textbf{r}are \textbf{e}vent sampling with \textbf{S}equential \textbf{M}onte \textbf{C}arlo (\textbf{DireSMC}), a sequential Monte Carlo scheme that guides a population of weighted samples towards the rare event, giving access not only to samples but also to a calibrated estimate of its probability. We set up our guidance using an analytical relaxation of the event set, allowing the method to easily extend to a wide range of user-defined rare events. We validate our method on a toy problem with analytical solutions and on a score-based climate emulator, where we obtain accurate rare-event probabilities on a range of rarities from $10^{-3}$ to $10^{-5}$, achieving net speed-ups of $9\times$ to $1413\times$ over Monte Carlo.
\end{abstract}
\section{Introduction}
Score-based diffusion models \cite{song2021scorebasedgenerativemodelingstochastic} are a powerful tool for generative modelling, and have become standard in the quantitative sciences as cheap surrogates for expensive simulators in numerical weather prediction \cite{couairon2026archesweathergen, price2024gencastdiffusionbasedensembleforecasting}, molecular simulation \cite{abramson2024accurate} and materials design \cite{zeni2025generative}. Trained on data from a distribution $\pz$, they generate new representative samples by running a denoising process $\pt$ that anneals from a Gaussian reference $p_1$ down to the data $p_0$, driven by a learned score $\nabla_x\log\pt$. 

Often, scientists are interested in events $\Ev$ lying in regions where
the data distribution assigns little mass $\pE$. 
Such \emph{rare events} carry outsized downstream impacts, and arise across a wide range of domains, from the environmental sciences \cite{manshausen2026accurateextremeeventlikelihoods, b1gc-9c2q, peard2025simulating} and bioinformatics \cite{xie2026enhanced} to finance \cite{glasserman2005importance}, engineering \cite{au2014engineering} and autonomous driving \cite{o2018scalable}.
Practitioners typically need two things: representative samples of the event, to characterise possible failure modes, and a calibrated estimate of its probability $\pE$, to inform planning and resource allocation \cite{kaplan1981quantitative, rubino2009rare}. 
The simplest approach is to draw $N$ independent samples and form a Monte Carlo (MC) average to estimate $\pE$, which is difficult precisely because the event is rare. The relative error is $1/\sqrt{N\pE}$, hence stabilising the variance requires  $N\!\propto\!1/\pE$ samples, which is prohibitively expensive \cite{morio2014survey, rubino2009rare}. This has motivated a large body of research designed specifically to tackle the poor scaling of standard MC  \cite{Kahn1951SplittingParticleTransmission, morio2014survey, glasserman2005importance}. 

Diffusion model steering offers a natural method for generating samples from $E$ without  the high  costs of simple MC. A reward $r_0$ defines a tilted target $\pi_0 \propto p_0\, e^{r_0}$, and inference-time guidance steers a pre-trained sampler towards it. At each denoising step, the gradient of a time-dependent reward $r_t$ is added, so that the sampler targets $\pi_t \propto p_t\, e^{r_t}$ \cite{uehara2025inference, chung2023diffusion, song2023pseudoinverseguided}. 
However, designing a good reward is difficult as the ideal $r_t$ is often intractable, and any approximations used can guide the sampler towards a different distribution to $\pi_0$. For rare-event sampling specifically, setting the tilting term $e^{r_0} \equiv\mathbf{1}_E$ recovers the rare-event tilted distribution $p_0(\cdot |E) \propto p_0 \mathbf{1}_E$ that we wish to sample from. However, this reward is a poor object for steering, as its gradient cannot be used to steer particles towards $E$. 

\begin{figure}
    \centering
    \vspace{-1em}
    \includegraphics[width=0.9\linewidth]{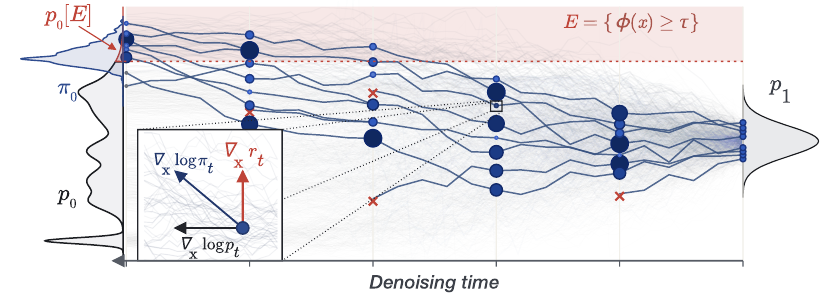}
    \vspace{-1em}
    \caption{We steer particles (\textcolor{Blue}{$\bullet$}) from $\pref$ towards \textcolor{Blue}{$\pi_{0}$}, which concentrates mass in the rare-event region of interest, using a guidance drift $\textcolor{red}{\nabla_x r_k}$. Particle size reflects the importance weights accumulated along the trajectory, which are used to resample (\textcolor{red}{$\times$}) and form a rare event estimate, $\pE$.}
    \label{fig:begining_shot_placeholders}
    \vspace{-14pt}
\end{figure}
Recent works have attempted to address these problems.
The error of using an imperfect reward $r_t$ can be addressed using sequential Monte Carlo (SMC) by tracking importance weights along the denoising trajectory and resampling particles to focus on high probability regions \cite{wu2023practical, he2026rne, singhal2025general, skreta2025feynman}.
In diffusion-based molecular simulation, \citet{xie2026enhanced} adapt ideas from \citet{abrams2013enhanced,shirts2008statistically} to estimate the occupancy of a rare metastable state using collective variables and a sequence of events $E_i$ converging to $E$. 
\citet{manshausen2026accurateextremeeventlikelihoods} instead steer a diffusion model using the probability-flow ODE (PF-ODE) \cite{song2020denoising}, and estimate $p_0[E]$ using importance sampling.
These methods have high compute costs, which become especially problematic in the rare event settings of
interest here. 

\textbf{Contributions.} 
\emph{First}, we present DireSMC where we use a sequence of tilted marginals $\pi_t$ to construct an unbiased sequential Monte Carlo estimator of $p_0[E]$ that runs within the denoising process, obtaining both samples and probability estimates. We show that the variance is governed by the quality of our steering rather than the rarity. \emph{Second}, we show how to design a reward for rare-event simulation, and in particular for threshold-exceedance events $\Ev\!=\!\{x : \phi(x)\ge\tau\}$, giving a smooth surrogate whose look-ahead guidance is closed-form under a Tweedie posterior approximation \cite{song2023pseudoinverseguided, chung2023diffusion}. \emph{Third}, we empirically validate on: (1) a Gaussian mixture where ground truth is available and we obtain probability estimates that do not degrade across $10^{-3}$--$10^{-5}$ orders of rarity. (2) a score-based climate emulator, where we draw samples and estimate probabilities of  compound hot-and-dry extremes over the region of the 2021 Pacific Northwest heatwave, with net speed-ups of $9\times$ to $1413\times$ over Monte Carlo across rarities from $10^{-3}$ to $10^{-5}$.

\section{Background}
\label{sec:background}

\textbf{Diffusion models.}
Let $\pdata$ be the true data generating distribution. A diffusion model \cite{song2021scorebasedgenerativemodelingstochastic} progressively noises the data along a path $p_t$ interpolating from $\pdata$ at $t = 0$ to a Gaussian reference $\pref$ at $t=1$. Sampling reverses this process: a $X_1 \sim \pref$ draw is evolved to $t=0$, under the reverse-time SDE,
\begin{equation}
\label{eq:reverse_sde_sigma}
    dX_t = \big[\mut(X_t)-\sigma_t^2\, \nabla_x \log \pt(X_t)\big]\,dt + \sigma_t\,d\Wb_t,\qquad t\in[0, 1],
\end{equation}
to generate from $\pdata$. The drift $\mut$ and the diffusion coefficient $\sigma_t$ are fixed by the noise schedule, and the score $\nabla_x \log \pt$ is approximated by a pre-trained network. In practice, Eq.~\eqref{eq:reverse_sde_sigma} cannot be simulated directly. Instead, we fix a discretisation schedule $0=t_0<\cdots<t_K=1$ and a numerical integrator whoese $k$-th step $t_k\!\rightarrow\!t_{k-1}$ defines a Markov transition kernel $\Pk(\dd x_{k-1}\mid x_k)$. Sampling corresponds to composing the $K$-step chain,
\begin{equation}
\label{eq:base_chain}
    X_K\sim \pref, \qquad X_{k-1}\mid X_k \sim \Pk(\dd x_{k-1}\mid X_k), \qquad k=K,\dots,1.
\end{equation}
Let $p_k$ denote the marginal law of $X_k$ under this chain. Hence $p_K\!=\!\pref$ and $p_{k-1}(x_{k-1})\!=\!\int p_k(x_k)\,\Pk(x_{k-1}\mid x_k)\,\dd x_k$. Henceforth, we refer to the end-point of the discretised chain as $p_0$, with $X_0\! \sim\!p_0$. As the discretisation is refined both $p_k\!\rightarrow\!p_{t= t_{k}}$ and $p_0\!\rightarrow\!\pdata$ at a rate given by the weak-order convergence of the integrator \cite{milstein2004stochastic, chen2023sampling, benton2024nearly}. Throughout, we assume that we can simulate from the transition kernels and have access to the score, but not to the densities $p_t$ and $p_k$, which are generally intractable.

\textbf{Rare Event Probabilities.}
Given a rare-event set $\Ev \subset \R^d$ and the endpoint law of the chain $p_0$, we seek to estimate its probability $\pE\!=\!\int \mathbf{1}_E(x)  p_0(x) \dd x $. A natural estimator of $\pE$ is obtained by drawing $N$ independent samples, $X^n_0 \sim p_0$ and forming the Monte Carlo (MC) estimate,
\begin{equation}
\label{eq:mc_estimator}
    \hpMC := \frac1N\sum_{n=1}^N \inds{\Ev}\big(X_0^n\big), \qquad \E\big[\hpMC\big] = \pE.
\end{equation}
 However, its relative variance is $(1-\pE)/(N\pE)$, so when the event is rare, an impractical number $N\propto 1/\pE$ of runs are needed for a fixed relative error \cite{rubino2009rare}.

\vspace{-5pt}
\section{Rare-event sampling and probability estimation}
\vspace{-5pt}
\label{sec:method}
\textbf{Notation.} For a measure $\mu$, $\mu[f]\!:=\!\int f(x)\,\mu(\dd x)$ for a function $f$ and $\mu[A]\!:=\!\mu[\inds{A}]$ for a set $A$. Square brackets denote action on a function or set, and parentheses a density or kernel evaluation. 

We seek to construct a low variance estimator of $p_0[E]$ with the same expectation as Eq.~\eqref{eq:mc_estimator}.
We do so by reward tilting \cite{rubino2009rare, uehara2025inference}. 
A user-specified reward $\rz:\R^d \rightarrow \R$ biases the base density to $\gamma_0(x) \coloneqq \pz(x)\,e^{\rz(x)}$ with a normalised form given by $\pi_0 := \gamma_0/\gamma_0[1]$. 
Since $\pz = \gamma_0\,e^{-\rz}$, both samples and estimates are functionals of the tilted law,
\begin{equation}
\label{eq:tilt-identity}
    \pE = \gamma_0\big[e^{-\rz}\inds{\Ev}\big],
    \qquad
    \pz(x\mid\Ev) = e^{-\rz(x)}\inds{\Ev}(x)\,\pi_0(x) \big/ \pi_0\big[e^{-\rz}\inds{\Ev}\big].
\end{equation}
In the remainder of the section, we detail how to efficiently sample from $p_0(\cdot |E)$, construct an estimator $\hat{p}_0[E]$ using Sequential Monte Carlo (SMC), and design a reward that reduces its variance. While we present the following for discretised transition kernels, DireSMC holds automatically for the continuous time construction of a diffusion model.

\subsection{Steering The diffusion Models towards Rare Events}
\label{sec:sequential}
 The marginals of the diffusion model, $p_k$, tilted by the reward $r_k$, define a sequence of densities,
\begin{equation}
\label{eq:target_sequence}
    \gamma_k(x) := p_k(x)\,e^{r_k(x)}, \qquad \pi_k(x) := {\gamma_k(x)}/{\gamma_k[1]}, \qquad k = K, \dots, 0, 
\end{equation}
pinned at both ends. At $k=0$ (data), we set $r_0$ to be the reward of Eq.~\eqref{eq:tilt-identity}, and at $k=K$ (noise) we impose $r_K = 0$, such that $\gamma_K = \pref$. We defer reward path $r_{K-1},\dots, r_1$ selection to Section~\ref{sec:reward_selection}.

\textbf{Sequential Monte Carlo Sampler.} We sample from the sequence of marginals of Eq.~\eqref{eq:target_sequence} using a sequential Monte Carlo sampler \cite{del2006sequential, doucet2022scorebased, vargas2023denoising}. We refer readers to App.~\ref{app:generic_smc_algorithm} for its general formulation; below we give a version adapted to our setting. 
Firstly, initialise $N$ particles, $X_K^n\sim \pref$ with weights $w_K^n=1$. At each step $k$, (i) propagate each particle through a backward proposal kernel, $\Qk$, and (ii) update its weight by,
\begin{equation}
\label{eq:smc_incremental_weight_bg}
  X_{k-1}^n\sim \Qk(\cdot \mid X_k^n),\qquad w_{k-1}^n = w_k^n\,\gk\big(X_{k-1}^n, X_k^n\big),
\end{equation}
where $\gk$ is the incremental weight function built from an auxiliary forward kernel $\Fk$,
\begin{equation}
\label{eq:generic_weight}
    \gk(x_{k-1}, x_k)
= \Bigg(\frac{\exp\big(r_{k-1}(x_{k-1})\big)}
       {\exp\big(r_{k}(x_k)\big)} \Bigg)
\cdot
\Bigg(
\underbrace{
\frac{p_{k-1}(x_{k-1})}
     {p_{k}(x_k)}
}_{\text{intractable}}
\cdot
\frac{\Fk(x_k\mid x_{k-1})}
     {\Qk(x_{k-1}\mid x_k)} \Bigg).
\end{equation}
The first factor is the ratio of reward-tilts at consecutive steps; the second is a likelihood ratio of the base process, involving the marginals $p_{k}$ we cannot evaluate. The choice of both kernels is flexible, provided we can sample from $\Qk$ and evaluate $\gk$ pointwise \cite{del2006sequential}. Optionally, when the effective sample size, $\ESS(w_{k-1}) := \|w_{k-1}\|_1^2 / \|w_{k-1}\|_2^2$ drops below a threshold $\eta N$ for $\eta \in [0, 1]$, we can resample, replacing each particle by an ancestor drawn in proportion to its weight and resetting the weights to the average weights of the ancestors. This is done through setting $X_{k-1}^n \gets X_{k-1}^{a^n}$ and $w_{k-1}^n \gets \frac1N\sum_n w^n_{k-1}$ with the ancestor index $a^n$ drawn, e.g. systematically, from $\Pr[a^n = m \mid w_{k-1}] \propto w_{k-1} ^m$ \cite{chopin2022resampling}.
As resampling duplicates high-weight particles and depletes diversity, we follow it with a rejuvenation step, $X_{k-1}^n \sim M_{k-1}(\cdot \mid X_{k-1}^n)$, using a $\pi_{k-1}$-invariant MCMC kernel $ M_{k-1}$. We use the Metropolis-adjusted diffusion model (MADM) kernel of \cite{lam2026metropolis}, which approximates a Metropolis-adjusted Langevin kernel using only the score and $r_k$; both quantities that we already have access to (App.~\ref{sec:madm_details}).

\textbf{Choice of Kernels.} We construct our kernels to, (i) circumvent the intractable marginals and (ii) concentrate samples on $E$. For the first, we set $\Fk$ as the time reversal of the base chain, 
\begin{equation}
    \Fk(x_k\mid x_{k-1}) := p_k(x_k)\,\Pk(x_{k-1}\mid x_k)/p_{k-1}(x_{k-1}),
\end{equation}
where $\Pk$ is the reverse diffusion kernel. This removes the intractable ratio and is a Markov kernel by the marginal recursion of $p_{k-1}$ of Section~\ref{sec:background}. Our backward proposal, $\Qk$ is defined by adding the gradient of the reward, or the \emph{guidance}, to the reverse drift of Eq.~\eqref{eq:reverse_sde_sigma},
\begin{equation}
\label{eq:guided_sde}
    dX_t = \big[\mut(X_t)-\sigma^2_t\,\big(\nabla_x\log \pt(X_t) + \nabla_x\rt(X_t)\big)\big]\,dt + \sigma_t\,d\Wb_t,
\end{equation}
so that steering is exactly a tilt of the score, $ \nabla_x \log p_{k} \rightarrow \nabla_x \log p_{k} + \nabla_{x} r_k$ and $\Qk$ and $\Pk$ share the \emph{same} integration steps but differ in the score. In our construction the \emph{guidance strength} is defined within $r_k$ \cite{uehara2025inference}. Substituting into Eq.~\eqref{eq:generic_weight} gives the update function,
\begin{equation}
\label{eq:kernel_ratio_weight}
    \gk(x_{k-1}, x_k)
    = \frac{e^{r_{k-1}(x_{k-1})}}{e^{r_k(x_k)}} \cdot \frac{\Pk(x_{k-1}\mid x_k)}{\Qk(x_{k-1}\mid x_k)}.
\end{equation}

\textbf{Computing the Incremental Weights.} Eq.~\eqref{eq:kernel_ratio_weight} requires the ratio $\Pk/\Qk$ pointwise. For one-step \emph{explicit} integrators \cite{milstein2004stochastic} with Gaussian transitions and a mean affine in the score, guidance shifts the mean but leaves the covariance invariant. So, for integrator dependent scalars $c_k, v_k>0$,
\begin{equation}
\label{eq:gaussian_kernels}
\begin{aligned}
    \Pk(\dd x_{k-1}| x_k) &= \N\big(m_k(x_k),\, v_k I\big), \\
    \Qk(\dd x_{k-1}| x_k) &= \N\big(m_k(x_k) + c_k\nabla_x r_k(x_k),\, v_k I\big).
\end{aligned}
\end{equation}
This holds for first-order stochastic integrators such as Euler-Maruyama or SEEDS-1 of \cite{gonzalez2023seeds}, but not for higher-order schemes that evaluate the score at a noisy intermediate point, e.g. the stochastic Heun integrator of \cite{karras2022elucidatingdesignspacediffusionbased}.

\begin{lemma}
\label{prop:omega_sufficent}
Under Eq.~\eqref{eq:gaussian_kernels}, let $\om_k := c_k/\sqrt{v_k}$, $\epsilon_k\sim\N(0,I)$ for sampling $x_{k-1}$ from $\Qk$, then
\begin{equation}
\label{eq:dlogw}
    \log \gk(x_{k-1},x_k) =  r_{k-1}(x_{k-1}) - r_k(x_k) - (\om_k^2/2)\big\|\nabla_x r_k(x_k)\big\|^2 - \om_k \big\langle \nabla_x r_k(x_k), \epsilon_k\big\rangle .
\end{equation}
\end{lemma}
The proof is given in App.~\ref{app:proofs}; we use the analytical form for a ratio of Gaussians with a shared covariance alongside a discrete Cameron-Martin correction for the guidance term \cite{oksendal2003stochastic}. The lemma states that the choice of numerical integrator only enters the weights via the scalar $\om_k$. Additionally, tracking the weights requires no additional network calls as both $\nabla_x r_k$ and $\epsilon_k$ are readily available when the particle is steered. Under this formulation, for the Euler-Maruyama integrator with a VP-SDE noise schedule: $c_k = v_k = \sigma_{t_k}^2\Delta t_k$ gives $\om_k = \sigma_{t_k}\sqrt{\Delta t_k}$, and Eq.~\eqref{eq:dlogw} recovers the weights of \cite{he2026rne}. App.~\ref{app:gamma_derieved} derives $\om_k$ for different integrators.

\textbf{Estimators.} With the weights reset to the average at each resampling, the particle system at $k=0$ is a weighted approximation of $\gamma_0$, and substituting the weights into Eq.~\eqref{eq:tilt-identity} gives the two outputs we are after: the probability estimate $\hpE = \hat\gamma_0[e^{-\rz}\inds{\Ev}]$ and weighted samples from the conditional,
\begin{equation}
\label{eq:pE_estimator}
    \hpE = \,\frac{1}{N}\sum_{n=1}^N w^n_0\,e^{-\rz(X^n_0)}\inds{\Ev}(X^n_0),
    \,
    \hat p_0[\dd x\mid\Ev] = \frac{\sum_{n=1}^N w^n_0\,e^{-\rz(X^n_0)}\inds{\Ev}(X^n_0)\,\delta_{X^n_0}[\dd x]}{\sum_{n=1}^N w^n_0\,e^{-\rz(X^n_0)}\inds{\Ev}(X^n_0)}.
\end{equation}

\begin{algorithm}[t]
\caption{DireSMC: Steering diffusion models to rare events.}
\label{alg:rne_ais}
\begin{algorithmic}[1]
\Require pretrained diffusion model, transition kernels, $\Pk$ and $\Qk$, rare event $\Ev$, reward path $r_k$, particles $N$, discretisation schedule $1 = t_K > \cdots > t_0 = 0$, $\ESS$ threshold $\eta\in[0,1]$.

\State $X^n_K \sim \pref$, $w^n_K \gets 1$ for $n=1,\dots,N$,
\For{$k = K, \dots, 1$}
    \State $X^n_{k-1}\sim \Qk(\dd x_{k-1}\mid X^n_k)$ \Comment{One integrator step with the guided score Eq.~\eqref{eq:gaussian_kernels}}
    \State $w^n_{k-1} \gets w^n_{k}\, \gk(X^n_{k-1},X^n_{k})$  \Comment{Update weights using Eq.~\eqref{eq:dlogw}}
    \If{$\ESS(w_{k-1}) < \eta N$}
        \State $\bar w \gets \frac{1}{N}\sum_{m} w^m_{k-1}$;
        \State
               $X^n_{k-1}\gets X^{a^n}_{k-1}$, $w^n_{k-1}\gets \bar w$
               \Comment{Draw $a^n$ using systematic resampling}
        \State $X^n_{k-1} \sim M_{k-1}(\dd x_{k-1}\mid X^n_{k-1})$ \Comment{Rejuvenate with a $\pi_{k-1}$ invariant MCMC kernel}
    \EndIf
\EndFor
\State \Return $\hpE$ and $\hat p_0[\dd x\mid\Ev]$ from Eq.~\eqref{eq:pE_estimator}
\end{algorithmic}
\end{algorithm}

\begin{theorem}
\label{thm:main} Suppose $\Pk\ll\Qk$ for all $k$, and the rejuvenation kernels $M_k$ are $\pi_k$-invariant. Then, for any $K$,  $N$,  rewards $r_0,\dots,r_K$ with $r_K=0$ and $\gamma_0[1]<\infty$, and any adaptive resampling schedule, $\E[\hpE] = \pE = \E[\hpMC]$, and $\hat p_0[\cdot\mid\Ev]$ converges almost surely to $\pz[\cdot\mid\Ev]$ as $N\to\infty$. Moreover, if the  score is exact and $\pdata[\partial\Ev]=0$, then $\E[\hpE]\to\pEdata$ as $K\to\infty.$ 
\end{theorem}
The proof is given in App.~\ref{app:proofs} and follows from the unbiasedness of the SMC normalisation constant estimator under adaptive resampling and $\pi_{k-1}$ invariant moves \cite{del2012adaptive, syed2026optimized}. The estimator is well suited to the rare-event setting as it remains unbiased for \emph{any} reward, however poorly it may steer. The normalisation constant, $\hat{p}_0[E]$ is accumulated along the denoising process rather than estimated separately \cite{vargas2023denoising, doucet2022scorebased, del2006sequential} (App.~\ref{sec:why_not_snis}). The rewards only affects the variance, which is the subject of Section~\ref{sec:reward_selection} below.

\subsection{Reward selection}
\label{sec:reward_selection}
Theorem~\ref{thm:main} shows $\hat{p}_0[E]$ is unbiased independent of reward choice, however, a poorly chosen steering inflates the variance of $\hpE$ to an impractical number of required samples \cite{chatterjee2018sample}. 
We construct a reward path and associated transition kernels that reduce the variance of $\hpE$. 
Let $\Pmk$ and $\Qmk$ be the joint laws over the base and guided chains, $(X_0,\dots,X_K)$. The ideal law to simulate from is the base chain reweighted by $e^{r_0(X_0)}/\gamma_0[1]$, since its endpoint is $\pi_0$. 
This is Doob's \emph{h}-transform, $\Pms$ of $\Pmk$  \cite{dai1991stochastic, denker2024deft}. 
Writing $\pbwdk{0}{k}$ as the denoising posterior of $X_0$ given $X_k$, the optimal reward and transitions are
\begin{equation}
\label{eq:doob_objects}
\begin{aligned}
    \rs{k}(x_k) &:= \log \int \pbwdk{0}{k}(\dd x_0\mid x_k)\, e^{\rz(x_0)}, \\
    \Pkr(\dd x_{k-1}\mid x_k) &:= \Pk(\dd x_{k-1}\mid x_k)\, e^{\rs{k-1}(x_{k-1}) - \rs{k}(x_k)},
\end{aligned}
\end{equation}
and its marginal at step $k$ is $\pis{k}\propto p_k\,e^{\rs{k}}$, with $\pis{0}=\pi_0$ (App.~\ref{app:proofs}). While we refer to $\rs{k}$ as a reward for simplicity, it is also known as the value function in stochastic optimal control; see \cite{berner2024an}. The discrepancy between $\Pms$ and $\Qmk$ determines the variance of our estimator. 
\begin{theorem}
\label{thm:variance}
Suppose there is no resampling ($\eta=0$), $\Pk\ll\Qk$ for all $k$, and $e^{\rz}\ge c$ on $\Ev$ for some $c\in(0,1]$. Then
\begin{equation}
\label{eq:variance_bound}
    \frac{\Var\big[\hpE\big]}{\pE^2} \;\le\; \frac{1}{N}\left[\Big(\frac{\gamma_0[1]}{c\,\pE}\Big)^2\Big(1+\chi^2\big(\Pms\,\|\,\Qmk\big)\Big)-1\right],
\end{equation}
with equality when $e^{\rz}=c\,\inds{\Ev}$, in which case $\Pms$ is $\Pmk$ conditioned on $X_0\in\Ev$. Moreover, the Doob transitions minimise Pearson's $\chi^2$ divergence: for any $\Qk$,
$\chi^2(\Pms\,\|\,\Qmk) \;\ge\; \chi^2(\pis{K}\,\|\pref )$ with equality when $\Qk=\Pkr$ for all $k$. The minimum is at most $\chi^2(\pi_0\|\pz)$, and it is zero if $X_K$ is independent of $X_0$ under $\Pmk$.
\end{theorem}
The proof is in App.~\ref{app:proofs} and follows from the chain rule of measures and the data-processing inequality. 
The bound separates two design choices: (1) the endpoint reward affects $\gamma_0[1]/c\pE$, which equals to 1 when $e^{r_0} \equiv 1_E$; 
(2) the second factor is minimised when the optimal Doob's transitions of Eq.~\eqref{eq:doob_objects} are used throughout the generation process. To elaborate, suppose $e^{r_0} = c1_E$, then the two distinct effects materialise. 
First, turning the steering off, the bound is exact with a relative variance equalling the Monte Carlo estimator keeping the first term. 
Second, if the Doob's transition kernels are instead used, then the relative variance is the initialisation error between the ideal starting distribution, $\pis{K}$ and the reference $\pref$. 
The bound also gives us a practical target. 
With resampling turned off, we have $\text{ESS}/{N}$$ \rightarrow
1/(1+\chi^2(\Pms\|\Qmk))$ as $N \to \infty$ \cite{agapiou2017importance} which allows us to evaluate the quality of the simulated trajectories. 
Under resampling and sufficient rejuvenation, a more elaborate error analysis is needed, see Theorem 1(a) of \cite{syed2026optimized}. 
Intuitively, errors from a poor transition kernel $\Qk$ can be pruned by resampling instead of compounding.

\textbf{Rewards for threshold exceedance events.}
We now construct $r_0$ and $r_1,\dots,r_{K-1}$ for events of the form $\Ev = \{x : \phi(x)\ge\tau\}$, for an observable $\phi:\R^d\to\R$ and threshold $\tau$; App.~\ref{app:multiple_guidances} extends $r_0$ and $r_k$ to interval and multi-constraint events. 
Theorem~\ref{thm:variance} identifies the ideal terminal reward $e^{r_0} \!=\!c1_E$. However, it lacks the regularity needed to be steered with. We require a finite continuously differentiable reward so the guidance $\nabla_x \rt$ is well defined, and whoese gradient is non-zero outside $\Ev$, so particles outside the event set are actively pushed towards it. 
Further, we require that $r_0$ is bounded below $\Ev$, so that the factor $\gamma_0[1]/cp_0[E]$ in Theorem~\ref{thm:variance} is controlled. Any CDF $\Phi$ with an everywhere-positive density gives such a regularity; writing $\delta>0$ for a smoothing width,
\begin{equation}
  \rsoft(x) \;:=\; \log\Phi\!\left((\phi(x)-\tau) \, \big/ \, \delta\right),
  \label{eq:soft-reward}
\end{equation}
with $e^{\rsoft}\ge\Phi(0)=\tfrac12$ on $\Ev$. So $c=\tfrac12$ in Theorem~\ref{thm:variance}, and the induced $\pi_0$ tends to $\pz[\cdot\mid\Ev]$ as $\delta\downarrow 0$. We adopt the Gaussian CDF because, for a linear observable $\phi(x) = a^\top x + b$, it is closed under convolution with a Gaussian \cite{10.7551/mitpress/3206.001.0001}. See Fig.~\ref{fig:reward_guidance_term} for an illustration.

\textbf{The reward sequence.} Fixing $\rsoft$ leaves the Doob's rewards of Eq.~\eqref{eq:doob_objects} intractable through the posterior $\pbwdk{0}{k}$. Following \cite{chung2023diffusion, song2023pseudoinverseguided}, we approximate it by a Gaussian centred at its mean $\xhk(x_k) = \E[X_0\mid X_k=x_k]$,
\begin{equation}
\label{eq:pigdm_gaussian}
    \pbwdk{0}{k}(\dd x_0\mid x_k) \;\approx\; \N\big(\xhk(x_k),\, \Sigk\big),
\end{equation}
with $\xhk$ approximated by Tweedie's formula from the score at $t_k$ \cite{Efron01122011} (Eq.~\eqref{eq:tweedie_appendix}) and $\Sigk$ is approximated, e.g. via a diagonal $\Pi$GDM posterior-covariance proxy \cite{song2023pseudoinverseguided}; we give  the explicit forms of both in App.~\ref{app:diffusion_details}. For the Gaussian-CDF reward and a linear observable $\phi$, the resulting integral of Eq.~\eqref{eq:doob_objects} is closed-form, $\rs{k}(x_k) \approx \log \Phi\!( \, (\phi(\xhk(x_k)) - \tau) \, \big/ \, \sqrt{\delta^2 + a^\top \Sigk a} \, )$,
with the smoothing width inflated by the convolution (App.~\ref{app:guidance_gradient}). The approximation is accurate near the data end and poor at high noise, where $\xhk$ carries little information about $X_0$, and it does not vanish at $k=K$ as the boundary condition requires. We therefore temper $\rs{k}$
by a guidance strength $\lambda_k\ge0$, which switches the guidance on gradually as the look-ahead becomes reliable, and optionally approach the event in stages through a threshold $\tau_k$. This gives us our reward sequence,
\begin{equation}
\label{eq:lookahead_reward}
    r_k(x_k) = \lambda_k\,\log\Phi\!\left( \, ( \phi(\xhk(x_k))-\tau_k) \big/ \sqrt{\delta^2 + a^\top\Sigk a} \, \right),
\end{equation}
pinned at both ends as Eq.~\eqref{eq:target_sequence} requires: $\lambda_0=1$, $\tau_0=\tau$ give $r_0=\rsoft$ and $\lambda_K=0$ gives $r_K=0$. The schedules $\lambda_k$ and $\tau_k$ are the two parameters of our method; App.~\ref{app:annealing_splitting} relates them to classical rare-event samplers. By Theorem~\ref{thm:main}, none of these approximations affect what $\hpE$ estimates; by Theorem~\ref{thm:variance}, they determine the variance of our estimator.
\subsection{Related work}
Several works use a pre-trained diffusion sampler to target a tilted distribution without retraining \cite{chung2023diffusion, song2023pseudoinverseguided, denker2024deft, berner2024an, uehara2025inference}, but these do not compute probabilities. 
The weight construction combining diffusion models and SMC samplers \cite{skreta2025feynman, wu2023practical, singhal2025general, he2026rne} is common, but our contribution is to apply it to rare-event sampling and probability estimation. 
In diffusion-based molecular simulations, enhanced rare-event sampling ideas such as metadynamics and umbrella sampling have been adapted \cite{xie2026enhanced} to estimate the occupancy of a rare metastable state. Both methods track importance weights however, \citet{xie2026enhanced} assemble $p_0[E]$ by combining a sequence of events $E_i$ with MBAR \cite{shirts2008statistically}, rather than estimating $p_0[E]$ directly. 
\citet{manshausen2026accurateextremeeventlikelihoods} use a score-based climate emulator guidance to generate Tropical Cyclone (TC) events and use importance sampling. An ODE sampler samples the target rare event, followed by two passes of a guided and unguided Probability Flow ODE to estimate its importance weights \cite{song2021scorebasedgenerativemodelingstochastic}.  We instead estimate $p_0[E]$ by accumulating weights from a single denoising pass and benefit from resampling, which is unavailable to them; see App.~\ref{app:manshausen_comparison} for a comparison.

\section{Experiments}
\label{sec:experiments}
We validate DireSMC in two settings: a correlated two-dimensional Gaussian mixture model (GMM) and a pretrained score-based climate emulator. 
For the GMM, the score and guidance are known analytically so our method may be evaluated in the absence of approximation errors in addition to a learned score, see App.~\ref{app:corr_gmm}. 
The second is a pretrained score-based climate emulator for estimating extreme compound weather events, demonstrating the scalability and applicability of our method. Estimates are evaluated against the sampled law, $p^\star$: closed-form $\pE$ for the analytical score or a large MC for the learned score; the same references define the sample-quality metrics. The ratio $\hpE/p^*$ measures the bias and the estimator coefficient of variation, $\mathrm{CV} = \sqrt{\Var(\hpE)}/\,p^\star$ measures the relative variance computed over independent runs. 

\subsection{Correlated Gaussian Mixture Model (GMM) }
\label{sec:exp_2dgmm}

\begin{figure}[!h]
  \centering
  \includegraphics[width=0.95\textwidth]{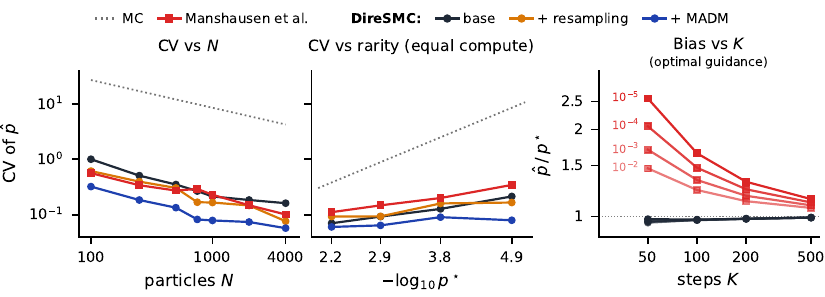}
  \vspace{-10pt}
  \caption{\emph{Left:}  CV plotted against $N$ for a  $p_0[x_1 >10]$ problem at a fixed $K$. Rejuvenation tracks the lowest CV across all $N$. \emph{Middle:} we compare our estimator against \citet{manshausen2026accurateextremeeventlikelihoods} across rarities at a compute-normalised $N$. We consistently track a lower CV. \emph{Right:} measured bias scaling with rarity using exact Doob's steering and analytical scores, highlighting the bias of \citet{manshausen2026accurateextremeeventlikelihoods}, DireSMC tracks the floor, see Tables~\ref{tab:us_v_manshausen_comparision} and App.~\ref{app:detailed_comparision_fig1}. }
  \label{fig:table_1_replacement}
  \vspace{-5pt}
\end{figure}
Let $p$ be the seven-mode GMM shown in Figure~\ref{fig:corr_gmm_base} and defined in Table~\ref{tab:corr_gmm}.
Let $\phi(x)\!=\!x_1$ and define $E=\!\{x_1\ge\tau\}$ for $\tau\!\in\!\{7,8,9,10\}$. 
This varies the analytic mass $\pE$ from $6.9\!\times\!10^{-3}$ to $1.4\!\times\!10^{-5}$, defining a sequence of increasingly rare events. 
Unless stated otherwise $N\!=\!1000$, $\delta\!=\!0.01$, $\lambda(t)\!=\!1\!-\!(t/0.9)^2\!$ with $\lambda(t)\!=\!0$ for $t >0.9$, $\tau(t)\!=\!\tau\!$ and $K=1000$, using the SEEDS-1 integrator guided by the reward of Eq.~\eqref{eq:lookahead_reward}; where $\Sigma_t$ is the $\Pi$GDM proxy (App.~\ref{app:diffusion_details}). We run Algorithm~\ref{alg:rne_ais} with adaptive resampling at $\ESS/N\!<\!0.5$, followed by three MADM rejuvenation sweeps \cite{lam2026metropolis} gated to a time $t\! \ge\!0.30$. Experiments are replicated over $24$ random seeds and validated against a 200 million MC reference pool (App.~\ref{app:more_2dgmm}). 

Our method recovers the reference exceedance probability for both analytical and learned scores (full per-threshold results and sample-quality metrics in 
Table~\ref{tab:step1_recovery}, App.~\ref{app:threshold_extra}). For a fixed  $\tau = 10$ threshold, Fig.~\ref{fig:table_1_replacement} shows that the relative variance decreases as $O(1/\sqrt{N})$ \cite{chatterjee2018sample}. Resampling and rejuvenation prevent accumulation of guidance mismatch \cite{del2006sequential, del2012adaptive} and the CV stays flat across rarities for a fixed $N$. 

DireSMC compares favourably against the importance sampling estimator of \citet{manshausen2026accurateextremeeventlikelihoods}.
Fig.~\ref{fig:table_1_replacement} (left) shows that DireSMC with resampling and MADM leads to lower CV for the same number of particles compared to their method, even though their method requires $3.5\times$ as many function evaluations as us. In this setting, rejuvenation incurs an additional $1.6\%$ network calls over the base, which we also discount. On a compute normalised particle budget, all variants of our estimator achieve a lower variance across all rarities as seen in Fig.~\ref{fig:table_1_replacement} (cost detailed in App.~\ref{app:manshausen_comparison} and Table~\ref{tab:cost}). Additionally, the estimator of \citet{manshausen2026accurateextremeeventlikelihoods} assumes that the law induced by an explicit integrator is reversible. 
However, this often does not hold in practice, leading to a discretisation bias which decreases with step count $K$.
To isolate the effects of discretisation bias, we ran experiments with the optimal Doob's transition kernels of Eq.~\eqref{eq:doob_objects} such that the true $p^*$ is known and any artefacts due to approximate guidance or score estimates are removed. We also provide the PF-ODE estimator with the true Jacobian trace term, which would otherwise have to be estimated.
The results in Fig.~\ref{fig:table_1_replacement} (right) show their discretisation bias is significantly larger than DireSMC's with the bias decreasing with increasing $K$ (cf. Table~\ref{tab:us_v_manshausen_comparision}).
We acknowledge that our method also incurs a discretisation bias which we characterise below.

\textbf{Ablations.} 
We isolate both the estimator bias
and variance 
introduced by SDE integrators (Euler-Maruyama, SEEDS-1 \cite{gonzalez2023seeds}) and the DDIM ODE integrator \cite{song2020denoising} used in \citet{manshausen2026accurateextremeeventlikelihoods}. Evaluating with both analytical and learned scores, the discretisation bias decreases approximately as $O(1/K)$ (see Fig.~\ref{fig:weak_order_integration_analytical},~\ref{fig:bias_learned_scores_rarity} in App.~\ref{app:taleytubero_expansion}). For a fixed $K$ we find the discretisation bias increases with rarity,
highlighting the sensitivity of rare-event probability estimation to small errors. Both SDE integrators track a lower bias compared to DDIM
which we postulate is due to the lower overall error in the generation process, see Theorem 3 of   \citet{xu2023restart}. Moreover, we characterise two sources of errors due to the MADM MCMC kernel \cite{lam2026metropolis}. It uses the $\nabla_x \log p_{t=t_k}$ to estimate Metropolis Hastings (MH) acceptance ratio, hence targets the continuous-time marginals rather than the law $\gamma_k$ induced by the integrator. This introduces an error decaying as $O(1/K)$ to leading order (Fig.~\ref{fig:madm_error_bias}). Additionally, MADM assumes access to a conservative score to compute the ratio
which degrades for a learned score network as $t \rightarrow 0$ \cite{vuong2025we}, see App.~\ref{app:madm_curl_exposition} for details. Heuristically, we gate rejuvenation to $t^* > 0$ to account for this assumption (Table~\ref{tab:cap_thr}). Finally, we show that the method translates without modification to interval and multi-constraint events (Table~\ref{tab:cap_box}). 

\subsection{Generating heatwaves with pre-trained climate model emulators}
\label{sec:exp_climate}
Next, we validate on the problem of estimating how the probability of extreme events evolves under a changing climate.
The high cost of running physics-based climate models \citep{eyring2016overview} has motivated the practice of training  cheap emulators on climate model output \citep{watson2022climatebench}. However, as we will show, even with these cost-savings, it can be difficult to generate targeted extreme events without explicit steering.

\textbf{Model.} We leverage the pre-trained climate emulator of \citet{https://doi.org/10.1029/2025MS005558} which is trained on data generated from the MPI-ESM1-2-LR model \citep{mauritsen2019developments} and released as part of the Coupled Model Intercomparison Project Phase 6 (CMIP6) \citep{eyring2016overview}. The emulator generates global monthly-mean fields of near-surface temperature ($\texttt{tas}$), precipitation ($\texttt{pr}$), relative humidity ($\texttt{hurs}$) and wind speed ($\texttt{sfcWind}$), conditioned on the global mean surface temperature anomaly, $\Delta T$, against a pre-industrial reference and the month of the year, $m$, to capture annual variability\footnote{As is common in the climate science research literature, the emulator outputs the change in output variables compared to a pre-industrial reference climate rather than absolute predictions.};  we refer to \citet{https://doi.org/10.1029/2025MS005558} for an in-depth motivation of the emulator setup.

\begin{figure}[t]
  \centering
  \includegraphics[width=\textwidth]{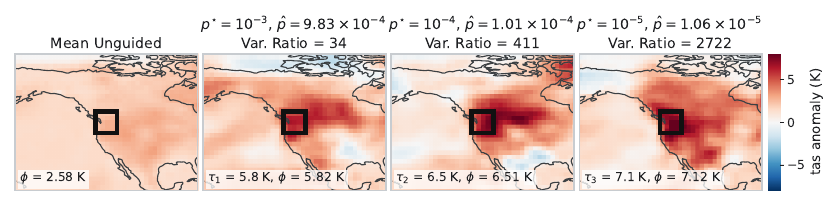}
  \vspace{-20pt}
  \caption{
    Surface-temperature anomaly (K) for July at $\Delta T = +1.5$~K over
    the Pacific Northwest (PNW) target region. \emph{Left:} average $\texttt{tas}$ field over $10^7$ samples. \emph{Remaining panels:} the highest-weight crossing
    particle produced by the guided sampler at different exceedance thresholds
    $\tau$ and regional average temperature $\phi(\cdot)$ of the sample. $p^\star$ are MC estimates; $\hat{p}$ is obtained from the guided estimator; full results in Table~\ref{tab:marginal_hot_dry}. The variance ratio (Var. Ratio) at matched compute is defined as $\Var(\hat{p}_{\text{unguided}}) / \Var(\hat{p}_{\text{guided}})$.}
  \label{fig:tas_ladder_uncond}
  \vspace{-10pt}
\end{figure}

\textbf{Setup.} We consider extremes in the temperature and humidity variables averaged within the region $R$ defined by the bounding box $42^\circ$-$52^\circ$N latitude and $235^\circ$-$245^\circ$E longitude, roughly corresponding to the region affected by the 2021 Pacific Northwest (PNW) heatwave \cite{pacificheatwave2021}. 
We use $N = 1024$ particles with $K = 200$ discretisation steps using the exponential integrator; error/variance metrics are computed using estimates from 16 seeds. 
We generate $\sim 10^7$ Monte Carlo samples to estimate ground-truth event probabilities. Further model details and hyperparameters are presented in App.~\ref{app:planned_climemu}. We select the hyperparameters once from a pilot sweep on the marginal hot and dry $\tau_1$ experiments (detailed below) and use them for all of our other experiments; proper per-problem tuning would lead to a larger variance reduction.

\textbf{Extreme temperatures.} We first verify how accurately DireSMC can estimate the probability of extreme temperatures relative to crude Monte Carlo. 
Assuming $\Delta T = 1.5$~K, for three different temperature thresholds $\tau_{1}, \tau_{2}, \tau_{3}$ 
(exact values in App.~\ref{app:planned_climemu}) we calculate the probability that the average surface temperature in the PNW exceeds a temperature threshold in the month of July.
The results in Fig.~\ref{fig:tas_ladder_uncond}, (full numbers in Table~\ref{tab:marginal_hot_dry}) show that our method provides significant improvements compared to standard MC estimates with at least an order of magnitude reduction in variance. 
For the most extreme event $\tau_{3}$, only 2/16 runs provided a single sample above the threshold; hence MC provides almost no samples with which to characterise the event. 
Figure~\ref{fig:tas_ladder_uncond} demonstrates that the generated guided samples provide physically plausible temperature maps: the maps show more warming over land masses compared to the ocean and spatially coherent warming patterns that are not restricted solely to the target region.
Additional results and ablations 
in App.~\ref{app:planned_climemu}.

\textbf{Compound hot and dry extremes.} 
Next we probe whether our method can capture extreme events impacting multiple variables by  
considering events that are both extremely hot and dry because of their outsized societal impacts \cite{https://doi.org/10.1029/2023EF003976,zscheischler2017bivariate}.
For these joint events it is important to not just consider the separate marginal probabilities of each variable but also to accurately capture the correlations between variables.
To do so, we compute the dependence ratio of the event, defined as $d = \pev{x > \tau_{\texttt{tas}}, x <\tau_{\texttt{hurs}}} \, \big/ \, \pev{x > \tau_{\texttt{tas}}}\,\pev{x < \tau_{\texttt{hurs}}}$. If the events are independent the dependence ratio would reduce to $d=1$, while $d > 1$ implies that the two events co-occur. We use the non-factorised form of the reward, $\Phi(\cdot) = \Phi_2(\texttt{hurs},\texttt{tas})$, a function of both variables; at the rarest level it reduces the CV fourfold relative to a product of marginal rewards (Table~\ref{tab:climate_exp2_full}). 
\begin{wraptable}{r}{0.65\linewidth}
\centering
\small
\setlength{\tabcolsep}{4pt}
\vspace*{-1.5em}
\caption{
Probability estimates of  three compound hot and dry extreme events over the PNW region. 
The dependence ratio $d$ is defined in the main text, $\hat{d}$ denotes that estimated by our method.
Net speed-up is the crude-MC cost of matching the CV of the guided run divided by the measured per-sample cost of guidance (Table~\ref{tab:compute_cost}); thresholds $\tau$ are given in
App.~\ref{app:model_details_climemu}. Mean $\pm$ standard error over 16 seeds.}
\label{tab:joint_model_exp}
\vspace*{0.3em}
\begin{tabular}{c l@{\;$\pm$\;}l l l l@{\;$\pm$\;}l r@{}}
\toprule
$p^\star$ & \multicolumn{2}{c}{$\hat p/p^\star$} & CV & $d^\star$
 & \multicolumn{2}{c}{$\hat d/d^\star$} & Speed Up $\uparrow$ \\
\midrule
  $3.14 \times 10^{-3}$ & $0.97$ & $0.02$ & $0.07$ & $31$ & $1.04$ & $0.02$ & $9\times$ \\
  $2.12 \times 10^{-4}$ & $0.99$ & $0.03$ & $0.11$ & $212$ & $1.02$ & $0.03$ & $47\times$ \\
  $1.39 \times 10^{-5}$ & $1.04$ & $0.02$ & $0.08$ & $1384$ & $1.06$ & $0.06$ & $1413\times$ \\
\bottomrule
\end{tabular}
\vspace*{-1.5em}
\end{wraptable}
Table~\ref{tab:joint_model_exp} shows we recover the correct joint probability and the dependence ratio against the ground-truth estimate with a positive net speed-up compared to crude MC. 
Notably, the dependence structure strengthens as the event becomes rarer from $d = 31.4 \rightarrow 1.4 \times10^{3} $. 
This suggests that a hot-and-dry July over $R$ is three orders of magnitude more likely than the product of its marginals would suggest. 

\begin{wrapfigure}{r}{0.5\linewidth}
\vspace{-2em}
  \centering
\includegraphics[width=0.95\linewidth]{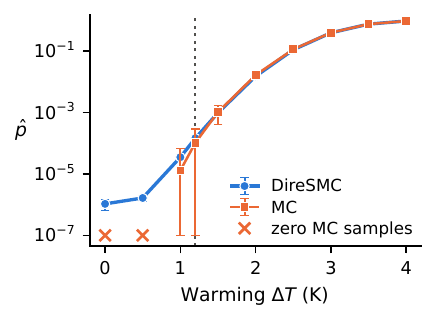} 
  \vspace{-1.5em}
  \caption{Estimated $\hat{p}_0[E]$ versus global temperature anomaly $\Delta T$ for our method versus unguided MC at equal compute. Points are the mean over 16 runs, bars are $\pm 1$ std. Bars reaching the axis floor indicate levels where some MC runs record no event; for $\Delta T \le 1.0$ each MC run records at most one event. See Tables~\ref{tab:climate_exp3} and \ref{tab:climate_exp3_sweep}.}
  \label{fig:extreme_exceedance}
  \vspace{-1em}
\end{wrapfigure}
\vspace{5pt}
\textbf{A warming climate.} As the model can generate samples conditioned on the global temperature anomaly $\Delta T$, it lets us ask \emph{how do extreme events change under climate change?} We fix the temperature threshold at $\tau_\texttt{tas} \ge 5.81$~K and sweep $\Delta T \in [0, 4]$~K, rerunning the same recipe as above. This temperature approximately corresponds to $\pE \approx 10^{-4}$ at $1.2$~K (2021) of warming, mirroring the PNW heatwave \cite{pacificheatwave2021}. The event moves from a threshold exceedance probability of $\sim\!10^{-6}$ at $\Delta T = 0$ to a near-certainty at $\Delta T = 4$~K. Our estimator tracks it across all six orders of magnitude without any modification. Fig.~\ref{fig:extreme_exceedance} reports this probability estimate. See App.~\ref{app:climemu_warming} for  details, including Table~\ref{tab:climate_exp3} and Fig.~\ref{fig:climate_change_rare_event_dT} for corresponding global temperature maps. This experiment shows that the method can be potentially used for climate attribution; see \citet{van2021pathways} for an overview.

\section{Conclusion}
\label{sec:conclusion}
We have shown that the diffusion sampler builds the same annealing bridge that traditional rare-event samplers build by hand. Using this insight, we present DireSMC and steer the generation process towards a user-defined rare event using a smooth relaxation of the event set, accumulating importance weights and estimating the probability along the way. Empirically, on synthetic examples and a score-based climate emulator, we achieve a net speedup of $9$ to $1413\times$ over crude Monte Carlo, improving with rarity. We believe that this is of value to the scientific community where there is a growing interest in investigating how well a generative model can capture rare events \cite{doi:10.1073/pnas.2420914122}.

\textbf{Limitations \& Future Work.} The variance of our estimator can be improved through better approximations of Doob's \emph{h}-transform e.g. \cite{potaptchik2026metaflowmapsenable, denker2024deft, uehara2025inference} as well as better reward schedule design \cite{syed2026optimized}; these can be readily plugged in without modification to the algorithm. Our estimates target the discretised, approximate law we sample from, and so carry a reducible discretisation bias and any errors of an approximate score. The rejuvenation kernel used \cite{lam2026metropolis} assumes a conservative score, which a trained network need not provide; distilling the network into an energy model \cite{thornton2025composition} could be a viable solution. Additionally the MADM kernel also does not target the law $\gamma_k$ induced by the numerical integrator, with this mismatch decreasing with the weak-order convergence rate of the integrator. Algorithm~\ref{alg:rne_ais} is agnostic to the choice of MCMC moves, and any alternatives to MADM can be used.

\begin{ack}
The authors would like to thank Abbas Mammadov, Kevin Lam, Alison Peard, Sam Howard, Charlotte Merchant and Lukas Seier for helpful discussions. AS is part of the Intelligent Earth CDT supported by funding from the UK Research and Innovation Council (UKRI) grant number EP/Y030907/1. TR acknowledges funding from the EU's Horizon Europe program under grant agreement number 10113184, acknowledge funding from UK Research and Innovation (UKRI), and also received funding from ARIA and DSIT and Pillar VC under the Encode: AI for Science Fellowship. CW acknowledges support from the Defence Science and Technology (DST) Group Australia. CW and YWT are supported in part by the Engineering and Physical Sciences Research
Council (EPSRC) through the AI Hub in Generative Models [grant number EP/Y028805/1]. PS acknowledges funding from European Union’s Horizon Europe projects CleanCloud and Embed2Scale with Grant Agreements 101137639 and 101131841 and their UKRI underwrite. We acknowledge the use of resources provided by JASMIN, the UK’s collaborative data analysis environment( https://www.jasmin.ac.uk) and the Isambard-AI National AI Research Resource
(AIRR) \cite{mcintosh2024isambard}. Isambard-AI is operated by the University of Bristol and is
funded by the UK Government’s Department for Science, Innovation and Technology (DSIT) via UK
Research and Innovation; and the Science and Technology Facilities Council [ST/AIRR/I-A-I/1023]
\end{ack}

\bibliographystyle{iclr2026_conference}
\bibliography{references}
\newpage
\appendix

\addtocontents{toc}{\protect\setcounter{tocdepth}{2}}
\setcounter{tocdepth}{2}
\begin{center}
  {\Huge\bfseries Appendix\par}
\end{center}
\vspace{2\baselineskip}
\begingroup
  \renewcommand{\contentsname}{}
  \vspace{-3.5\baselineskip}
  \tableofcontents
\endgroup
\newpage

\clearpage
\section{Proofs}
\label{app:proofs}

\subsection{Notation and two identities}
\label{app:proof_setup}

Throughout, $(X_0,\dots,X_K)$ is the base chain of Eq.~\eqref{eq:base_chain}, $\Pmk$ its law, $p_k$ its marginals, and $\Qmk$ the law of the guided chain with kernels $\Qk$; both start from $\pref$. All densities are with respect to Lebesgue measure. For a reward $\rz$ with $\gamma_0[1]<\infty$, write $Z_0 := \gamma_0[1]$ and
\begin{equation}
    \frac{d\Pms}{d\Pmk}(x_{K:0}) := \frac{e^{\rz(x_0)}}{Z_0},
    \qquad
    \frac{d\PmE}{d\Pmk}(x_{K:0}) := \frac{\inds{\Ev}(x_0)}{\pE},
\end{equation}
for the base chain tilted by $e^{\rz}$ and conditioned on $X_0\in\Ev$ respectively; the second is the first with $e^{\rz}=\inds{\Ev}$. The Doob rewards are
\begin{equation}
\label{eq:app_doob_reward}
    \rs{k}(x_k) := \log \int \pbwdk{0}{k}(\dd x_0\mid x_k)\, e^{\rz(x_0)}, \qquad k=0,\dots,K,
\end{equation}
so that $\rs{0}=\rz$ and $Z_0 = \int \pref(x_K)\,e^{\rs{K}(x_K)}\,\dd x_K$.

Two identities are used repeatedly. First, the \emph{tower identity}: by the Markov property,
\begin{equation}
\label{eq:app_tower}
    \int \Pk(\dd x_{k-1}\mid x_k)\, e^{\rs{k-1}(x_{k-1})} \;=\; e^{\rs{k}(x_k)}, \qquad k=1,\dots,K.
\end{equation}
Second, the \emph{telescoping identity}: the reward ratios in the incremental weight of Eq.~\eqref{eq:kernel_ratio_weight} collapse along a trajectory, since $r_K=0$:
\begin{align}
\label{eq:app_telescope}
    \prod_{k=1}^K \gk(x_{k-1},x_k)
    &= e^{\rz(x_0)} \prod_{k=1}^K \frac{\Pk(x_{k-1}\mid x_k)}{\Qk(x_{k-1}\mid x_k)} \notag\\
    &= e^{\rz(x_0)}\,\frac{d\Pmk}{d\Qmk}(x_{K:0}),
\end{align}
the last equality because both chains start from $\pref$.

\subsection{Structure of the tilted chain}
\label{app:doob_structure}

\begin{lemma}
\label{lem:doob}
$\Pms$ is a Markov chain. Its marginal at step $k$ is $\pis{k} := p_k\,e^{\rs{k}}/Z_0$, so its initial law is $\pis{K} = \pref\,e^{\rs{K}}/Z_0$, its transitions are the twisted kernels
\begin{equation}
\label{eq:app_twisted_kernel}
    \Pkr(\dd x_{k-1}\mid x_k) := \Pk(\dd x_{k-1}\mid x_k)\, e^{\rs{k-1}(x_{k-1}) - \rs{k}(x_k)},
\end{equation}
Consequently
\begin{equation}
\label{eq:app_dQstar_dQ}
    \frac{d\Pms}{d\Qmk}(x_{K:0})
    \;=\; \frac{e^{\rs{K}(x_K)}}{Z_0}\;\prod_{k=1}^K \frac{\Pkr(x_{k-1}\mid x_k)}{\Qk(x_{k-1}\mid x_k)} .
\end{equation}
\end{lemma}

\begin{proof}
By Eq.~\eqref{eq:app_tower}, each $\Pkr$ integrates to one and is a Markov kernel. Multiplying the kernels and telescoping the exponents,
\begin{equation}
    \prod_{k=1}^K \Pkr(x_{k-1}\mid x_k)
    \;=\; \frac{e^{\rz(x_0)}}{e^{\rs{K}(x_K)}} \prod_{k=1}^K \Pk(x_{k-1}\mid x_k),
\end{equation}
using $\rs{0}=\rz$. Hence
\begin{align}
    \frac{\pref(x_K)\,e^{\rs{K}(x_K)}}{Z_0}\prod_{k=1}^K \Pkr(x_{k-1}\mid x_k)
    &= \frac{e^{\rz(x_0)}}{Z_0}\,\pref(x_K)\prod_{k=1}^K \Pk(x_{k-1}\mid x_k) \notag\\
    &= \frac{d\Pms}{d x_{K:0}}(x_{K:0}),
\end{align}
which exhibits $\Pms$ as the chain started from $\pis{K}$ with kernels $\Pkr$. The marginal at step $k$ follows from the tower identity applied to the tilt: integrating out all coordinates but $x_k$ gives $p_k(x_k)\int\pbwdk{0}{k}(\dd x_0\mid x_k)e^{\rz(x_0)}/Z_0 = p_k(x_k)e^{\rs{k}(x_k)}/Z_0$. Dividing the density of $\Pms$ by that of $\Qmk$, $\pref(x_K)\prod_k\Qk(x_{k-1}\mid x_k)$, gives Eq.~\eqref{eq:app_dQstar_dQ}.
\end{proof}

\subsection{Proof of Theorem~\ref{thm:main}}
\label{app:proof_thm_main}

\begin{proof}
The auxiliary kernel $\Fk(x_k\mid x_{k-1}) = p_k(x_k)\Pk(x_{k-1}\mid x_k)/p_{k-1}(x_{k-1})$ integrates to one in $x_k$ by the marginal recursion of Section~\ref{sec:background}, so it is a Markov kernel. With targets $\gamma_k = p_k e^{r_k}$ and proposal $\Bk=\Qk$, the generic SMC weight of Eq.~\eqref{eq:generic_weight} is
\begin{align}
    \frac{\gamma_{k-1}(x_{k-1})\,\Fk(x_k\mid x_{k-1})}{\gamma_k(x_k)\,\Qk(x_{k-1}\mid x_k)}
    &= \frac{p_{k-1}(x_{k-1})\,e^{r_{k-1}(x_{k-1})}}{p_k(x_k)\,e^{r_k(x_k)}}
       \cdot\frac{\Fk(x_k\mid x_{k-1})}{\Qk(x_{k-1}\mid x_k)} \notag\\
    &= \frac{e^{r_{k-1}(x_{k-1})}}{e^{r_k(x_k)}}
       \cdot\frac{p_{k-1}(x_{k-1})}{p_k(x_k)}\cdot\frac{p_k(x_k)\,\Pk(x_{k-1}\mid x_k)}{p_{k-1}(x_{k-1})\,\Qk(x_{k-1}\mid x_k)} \notag\\
    &= \frac{e^{r_{k-1}(x_{k-1})}}{e^{r_k(x_k)}}\cdot\frac{\Pk(x_{k-1}\mid x_k)}{\Qk(x_{k-1}\mid x_k)},
\end{align}
which is Eq.~\eqref{eq:kernel_ratio_weight}; it is finite $\Qk$-a.s.\ because $\Pk\ll\Qk$. The boundary conditions give $\gamma_K=\pref=\pi_K$, which is normalised, and $\gamma_0=\pz e^{\rz}$ with $Z_0<\infty$. Algorithm~\ref{alg:rne_ais} is therefore the SMC sampler of Section~\ref{sec:sequential} with adaptive resampling and $\pi_k$-invariant rejuvenation, and its output is $\hat\gamma_0$ of Eq.~\eqref{eq:smc_estimators}, evaluated at $f=e^{-\rz}\inds{\Ev}$. By the unbiasedness of $\hat\gamma_0[f]$ under adaptive resampling and invariant moves \cite{del2012adaptive},
\begin{equation}
    \E[\hpE] = \gamma_0\big[e^{-\rz}\inds{\Ev}\big] = \int \pz(x)\,e^{\rz(x)}e^{-\rz(x)}\inds{\Ev}(x)\,\dd x = \pE,
\end{equation}
which is also $\E[\hpMC]$ since $X^n_0\sim\pz$ in Eq.~\eqref{eq:mc_estimator}. This holds for every $K$, $N$, reward sequence and resampling schedule. For the sample, $\hat p_0[\cdot\mid\Ev]$ is $\hat\pi_0$ reweighted by $e^{-\rz}\inds{\Ev}$ and normalised, i.e.\ the self-normalised estimator of Eq.~\eqref{eq:smc_estimators} for the target $\gamma_0 e^{-\rz}\inds{\Ev}\propto\pz[\cdot\mid\Ev]$, so it converges almost surely to $\pz[\cdot\mid\Ev]$ as $N\to\infty$ by the same result. For the limit, $\pz\Rightarrow\pdata$ as $K\to\infty$ under the standing assumptions of Section~\ref{sec:background}, and $\pdata[\partial\Ev]=0$, so $\pE\to\pEdata$ by the portmanteau theorem.
\end{proof}

\subsection{Proof of Theorem~\ref{thm:variance}}
\label{app:proof_thm_variance}

\begin{proof}
\emph{Step 1: the estimator is an importance-sampling average.}
Without resampling, $\hat Z=1$ and $\hpE=\frac1N\sum_{n=1}^N W^n$ with $W^n := W(X^n_{K:0})$ i.i.d.\ under $\Qmk$, where by the telescoping identity Eq.~\eqref{eq:app_telescope}
\begin{align}
    W(x_{K:0}) &:= \prod_{k=1}^K \gk(x_{k-1},x_k)\; \frac{\inds{\Ev}(x_0)}{e^{\rz(x_0)}} \notag\\
    &= \inds{\Ev}(x_0)\,\frac{d\Pmk}{d\Qmk}(x_{K:0}) \notag\\
    &= \pE\,\frac{d\PmE}{d\Qmk}(x_{K:0}).
\end{align}
Hence $\E_{\Qmk}[W]=\pE$ and
\begin{equation}
\label{eq:app_chi2_identity}
    \frac{\Var[\hpE]}{\pE^2} = \frac{1}{N}\Big(\E_{\Qmk}\Big[\Big(\frac{d\PmE}{d\Qmk}\Big)^2\Big]-1\Big) = \frac{\chi^2(\PmE\,\|\,\Qmk)}{N}.
\end{equation}

\emph{Step 2: change of measure through the tilted chain.}
Both $\PmE$ and $\Pms$ are tilts of $\Pmk$ by functions of $x_0$, so
\begin{equation}
\label{eq:app_PE_over_Pr}
    \frac{d\PmE}{d\Pms}(x_{K:0}) = \frac{Z_0\,e^{-\rz(x_0)}\inds{\Ev}(x_0)}{\pE} \;\le\; \frac{Z_0}{c\,\pE},
\end{equation}
using $e^{\rz}\ge c$ on $\Ev$. Therefore
\begin{align}
    1+\chi^2(\PmE\|\Qmk)
    &= \E_{\Qmk}\Big[\Big(\frac{d\PmE}{d\Pms}\Big)^2\Big(\frac{d\Pms}{d\Qmk}\Big)^2\Big] \notag\\
    &\le \Big(\frac{Z_0}{c\,\pE}\Big)^2\big(1+\chi^2(\Pms\|\Qmk)\big),
\end{align}
and substituting into Eq.~\eqref{eq:app_chi2_identity} gives Eq.~\eqref{eq:variance_bound}.

\emph{Step 3: the equality case.}
If $e^{\rz}=c\,\inds{\Ev}$, then $Z_0=c\,\pE$ and Eq.~\eqref{eq:app_PE_over_Pr} is identically one, so $\Pms=\PmE$ and Eq.~\eqref{eq:variance_bound} holds with equality.

\emph{Step 4: the floor.}
Under $\Qmk$, $X_K\sim\pref$; under $\Pms$, $X_K\sim\pis{K}$ by Lemma~\ref{lem:doob}. The data-processing inequality for the marginal at step $K$ gives
\begin{equation}
\label{eq:var_floor}
    \chi^2(\Pms\|\Qmk) \;\ge\; \chi^2(\pis{K}\,\|\,\pref).
\end{equation}
For the upper bound, $\Pms$ and $\Pmk$ share their transitions and differ only in the law of $X_0$, and $X_K\sim\pref$ under $\Pmk$, so data processing for the marginal at $K$ and then the fact that $d\Pms/d\Pmk$ depends on $x_0$ alone give
\begin{equation}
    \chi^2(\pis{K}\,\|\,\pref) \;\le\; \chi^2(\Pms\,\|\,\Pmk)
    = \E_{\Pmk}\Big[\Big(\frac{e^{\rz(X_0)}}{Z_0}\Big)^2\Big]-1 = \chi^2(\pi_0\,\|\,\pz).
\end{equation}
If $\Qk=\Pkr$ for all $k$, the product in Eq.~\eqref{eq:app_dQstar_dQ} is one, $d\Pms/d\Qmk$ is a function of $x_K$ alone, equal to $\pis{K}(x_K)/\pref(x_K)$, and the first inequality is an equality, which is Eq.~\eqref{eq:var_floor} with its equality case. If $X_K$ is independent of $X_0$ under $\Pmk$ then $\rs{K}$ is constant, so $\pis{K}=\pref$ and the minimum is zero.

Finally, for the CDF reward of Eq.~\eqref{eq:soft-reward}, $\phi\ge\tau$ on $\Ev$ gives $e^{\rsoft}\ge\Phi(0)=\tfrac12$ there, so $c=\tfrac12$, and $Z_0=\pz[e^{\rsoft}]\to\pE$ as $\delta\downarrow0$ by dominated convergence.
\end{proof}

\begin{remark}
Initialising from $\pis{K}$ instead of $\pref$, with kernels $\Pkr$ and weights $w^n_K=1$, would make every incremental weight equal to one and the estimator deterministic, the classical zero-variance statement for Doob's $h$-transform. It is circular: the normaliser of $\pis{K}$ is $Z_0$, which for the indicator reward is the quantity being estimated. The boundary condition $r_K=0$ is what allows Algorithm~\ref{alg:rne_ais} to start from $\pref$ and Eq.~\eqref{eq:var_floor} is the price this incurs.
\end{remark}

\subsection{The optimal reward recursion}
\label{app:reward_recursion}

We recall, for completeness, why the Doob rewards are the optimal intermediate targets \cite{pitt1999filtering, johansen2008note, guarniero2017iterated}. Fix $k$ and $r_{k-1}$, let $X_{k-1}\mid X_k\sim\Qk$, and define
\begin{equation}
    r^{\rm opt}_k(x_k) := \log\int\Pk(\dd x_{k-1}\mid x_k)\,e^{r_{k-1}(x_{k-1})}.
\end{equation}
From Eq.~\eqref{eq:kernel_ratio_weight}, for every proposal with $\Pk\ll\Qk$,
\begin{align}
    \E\big[\gk\mid X_k=x_k\big]
    &= e^{-r_k(x_k)}\int\Qk(\dd x_{k-1}\mid x_k)\,\frac{\Pk(x_{k-1}\mid x_k)}{\Qk(x_{k-1}\mid x_k)}\,e^{r_{k-1}(x_{k-1})} \notag\\
    &= \exp\big(r^{\rm opt}_k(x_k)-r_k(x_k)\big).
\end{align}
By the law of total variance,
\begin{equation}
    \Var[\gk] = \Var\big[\E[\gk\mid X_k]\big] + \E\big[\Var[\gk\mid X_k]\big].
\end{equation}
The reward $r_k$ enters $\gk$ only through the factor $e^{-r_k(x_k)}$, a function of $x_k$ alone, so the second term does not depend on $r_k$, while the first vanishes if and only if $r_k=r^{\rm opt}_k+\text{const}$. Thus $r^{\rm opt}_k$ minimises $\Var[\gk]$ over $r_k$ for every proposal. The second term vanishes if and only if $x_{k-1}\mapsto e^{r_{k-1}(x_{k-1})}\Pk(x_{k-1}\mid x_k)/\Qk(x_{k-1}\mid x_k)$ is $\Qk$-a.s.\ constant, i.e.\ $\Qk(\dd x_{k-1}\mid x_k)\propto\Pk(\dd x_{k-1}\mid x_k)e^{r_{k-1}(x_{k-1})}$, which with $r_k=r^{\rm opt}_k$ is the twisted kernel $\Pkr$ and gives $\gk\equiv1$. Finally, if $r_{k-1}=\rs{k-1}$ then $r^{\rm opt}_k=\rs{k}$ by the tower identity Eq.~\eqref{eq:app_tower}, so iterating from $r_0=\rs{0}$ recovers the Doob rewards at every step.

\subsection{Proof of Lemma~\ref{prop:omega_sufficent}}
\label{app:proof_prop_weight}

\begin{proof}
Write $\mu := m_k(x_k)$, $\Delta := c_k\nabla_x r_k(x_k)$ and $v:=v_k$. Under Eq.~\eqref{eq:gaussian_kernels} both kernels are Gaussian in $x_{k-1}$ with covariance $vI$, so
\begin{equation}
    \log\frac{\Pk(x_{k-1}\mid x_k)}{\Qk(x_{k-1}\mid x_k)}
    = -\frac{\|x_{k-1}-\mu\|^2}{2v} + \frac{\|x_{k-1}-\mu-\Delta\|^2}{2v}
    = \frac{\|\Delta\|^2}{2v} - \frac{\langle\Delta,\,x_{k-1}-\mu\rangle}{v}.
\end{equation}
Substituting $x_{k-1}=\mu+\Delta+\sqrt v\,\epsilon_k$, the noise drawn when sampling $\Qk$,
\begin{align}
    \log\frac{\Pk(x_{k-1}\mid x_k)}{\Qk(x_{k-1}\mid x_k)}
    &= \frac{\|\Delta\|^2}{2v} - \frac{\|\Delta\|^2}{v} - \frac{\langle\Delta,\epsilon_k\rangle}{\sqrt v} \notag\\
    &= -\frac{\om_k^2}{2}\big\|\nabla_x r_k(x_k)\big\|^2 - \om_k\big\langle\nabla_x r_k(x_k),\epsilon_k\big\rangle,
\end{align}
with $\om_k=c_k/\sqrt{v_k}$. Adding the reward tilt $r_{k-1}(x_{k-1})-r_k(x_k)$ of Eq.~\eqref{eq:kernel_ratio_weight} gives Eq.~\eqref{eq:dlogw}. No approximation is made; the integrator enters only through $(c_k,v_k)$. For Euler-Maruyama on a step of length $\Delta t_k$, $m_k(x_k)=x_k-[\mu_{t_k}(x_k)-\sigma^2_{t_k}\nabla_x\log p_{t_k}(x_k)]\Delta t_k$, the guidance shifts the mean by $c_k\nabla_x r_k(x_k)$ with $c_k=\sigma^2_{t_k}\Delta t_k$, and $v_k=\sigma^2_{t_k}\Delta t_k$, so $\om_k=\sigma_{t_k}\sqrt{\Delta t_k}$. The exponential-integrator constants are derived in App.~\ref{app:rne_weights}.
\end{proof}

\paragraph{Scope.} Three cases fall outside Lemma~\ref{prop:omega_sufficent}. Multi stage schemes such as the stochastic Heun sampler of \cite{karras2022elucidatingdesignspacediffusionbased} or higher order stochastic solvers of \cite{gonzalez2023seeds} evaluate the network at an additional noise-dependent interval point, so the kernel mean is not affine in a single score evaluation. Deterministic solvers have $v_k = 0$, giving degenerate kernels and $\om_k \to \infty$. Anisotropic or per dimension noise makes $v_k$ a matrix, and the scalar $\om_k$ is now a matrix.
\section{Diffusion model details}
\label{app:diffusion_details}

\paragraph{Training objective.} The score network $\sth(x_t,t)$ approximates the
marginal score $\nabla_{x_t}\log \pt$ and is trained by denoising score matching
\cite{song2021scorebasedgenerativemodelingstochastic},
\begin{equation}
  \mathcal{L}(\theta) = \E_{t}\,\E_{x_0\sim \pz}\,
  \E_{x_t\sim p(x_t\mid x_0)}
  \Big[\lambda_t\,\big\|\sth(x_t,t)-\nabla_{x_t}\log p(x_t\mid x_0)\big\|^2\Big],
\end{equation}
with weighting $\lambda_t>0$ and $t\sim\mathcal U[0,1]$. For a Gaussian forward kernel
$p(x_t\mid x_0)=\N(\alpha_t x_0,\sbar_t^2 I)$ the target is closed form,
$\nabla_{x_t}\log p(x_t\mid x_0)=-\epsilon/\sbar_t$, giving the equivalent
noise-prediction objective. Score-, noise-, and velocity-prediction parametrisations
are related by simple bijections \cite{gao2025diffusionmeetsflow}.

\paragraph{VP-SDE schedule.} We adopt the Variance-Preserving (VP) SDE
\cite{song2021scorebasedgenerativemodelingstochastic}, with
$f_t(x)=-\tfrac12\beta(t)x$ and $\sigma_t=\sqrt{\beta(t)}$ for a prescribed $\beta(t)$.
The forward kernel then admits the closed form, \[X_t=\alpha_t X_0+\sbar_t\epsilon, \qquad
\epsilon\sim\N(0,I)\] with,
\begin{equation}
   \alpha_t=\exp(-\tfrac12\!\int_0^t\!\beta(s)\,ds), \qquad
\sbar_t^2=1-\alpha_t^2.
\end{equation}
In experiments we use the linear schedule
$\beta(t)=\beta_{\min}+t(\beta_{\max}-\beta_{\min})$, $\beta_{\min}=0.1$,
$\beta_{\max}=20$ or a cosine schedule. 

\paragraph{Posterior mean (Tweedie).} The posterior mean of the clean state given a noisy $x_t$ is available in closed form from the score \cite{Efron01122011},
\begin{equation}
\label{eq:tweedie_appendix}
  \xh(x_t) = \frac{x_t + \sbar_t^2\, \sth(x_t, t)}{\alpha_t}.
\end{equation}
We use $\xh$ as the mean of the Gaussian denoising-posterior approximation in Section~\ref{sec:reward_selection}.

\paragraph{Posterior-variance proxy ($\Pi$GDM).} For the covariance of the Gaussian denoising posterior we use the $\Pi$GDM proxy, which is exact when the data distribution is a single Gaussian and is borrowed as a tractable surrogate otherwise \cite{song2023pseudoinverseguided}. Per dimension,
\begin{equation}
\label{eq:pigdm_variance}
  s_t^2 \;=\; \frac{\sbar_t^2\, V_{\mathrm{data}}}{\alpha_t^2\, V_{\mathrm{data}} + \sbar_t^2},
\end{equation}
with $\alpha_t^2$ and $\sbar_t^2=1-\alpha_t^2$ the VP-SDE marginal noise variance above (note $\sbar_t$ is the marginal noise std, distinct from the diffusion coefficient $\sigma_t=\sqrt{\beta(t)}$), and $V_{\mathrm{data}}=\Var(x_0)$ the marginal data variance (for the GMM benchmark, $\sum_k w_k(\Sigma_{k,11}+\mu_{k,1}^2)-(\sum_k w_k\mu_{k,1})^2$ along the read-out dimension). At early times ($t\to T$) $\alpha_t\to0$, $\sbar_t\to1$ and $s_t^2\to V_{\mathrm{data}}$ (full prior uncertainty); at late times ($t\to0$) $\sbar_t^2\to0$ and $s_t^2\to0$ (the posterior collapses to the Tweedie point estimate). Applied per dimension this yields the diagonal posterior covariance,
\begin{equation}
  \Sig = \operatorname{diag}\big(s_{t,1}^2,\dots,s_{t,d}^2\big),
\end{equation}
with $V_{\mathrm{data}}$ evaluated along each coordinate ($V_i=1$ for standardised data, otherwise estimated empirically).

\section{Sequential methods}
\label{app:annealing_splitting}
For completeness we provide an overview of classical sequential methods that are used in rare-event simulation, such as annealed importance sampling, splitting and sequential Monte Carlo (SMC). These are methods used to bridge the gap from a tractable reference $\pi_K$ to the tilted target $\piz$ of Eq.~\eqref{eq:target_sequence} through a sequence of overlapping intermediate distributions \cite{neal1998annealedimportancesampling, Kahn1951SplittingParticleTransmission, del2006sequential}. We briefly expand on this and highlight their similarities, but refer the readers to \cite{moral2004feynman} for a full exposition under the \emph{Feynman-Kac} formulation. Additionally, see \cite{naesseth2019elements, del2006sequential} for details on the SMC machinery. Formally the sequence of distributions defines a \emph{path-measure} over the space of all trajectories particles can take, but we omit this rigour for brevity.

Annealed Importance Sampling (AIS) and Sequential Monte Carlo (SMC) build this path by tempering the reward \cite{del2006sequential, neal1998annealedimportancesampling}. One such sequence of distributions tempers the reward,
\cite{del2006sequential, neal1998annealedimportancesampling},
\begin{equation}
\label{eq:ais_tempering}
    \pi_k(x) = \frac{1}{Z_k} \pz(x)\, e^{\beta_k r(x)},
    \qquad 1 = \beta_0 > \beta_1 > \cdots > \beta_K = 0.
\end{equation}
Rather than targeting the constrained target $\piz$ directly, tempering gradually sharpens the reward landscape until the chain converges to the target distribution. Additionally, SMC utilises resampling along the trajectory to prune low-weight particles and multiply those with higher weights \cite{doucet2000sequential, del2006sequential}. 

Multilevel splitting takes another approach, where the sequence is constructed using a sequence of nested rare event sets \cite{Kahn1951SplittingParticleTransmission, 0a3f6295-587d-3cd6-8835-1ba3c37b93ed}. Taking the canonical example of threshold exceedance, for a strictly decreasing ladder $\tau_0=\tau>\tau_1>\dots>\tau_K$, the event set is defined as $E_k = \{x : \phi(x) \ge \tau_k\}$. The sequence is defined to be,
\begin{equation}
\label{eq:splitting_sequence_tau}
    E_0 = \Ev \subset E_1 \subset \cdots \subset E_K,
    \qquad \pi_k(x) = \frac{1}{Z_{k, \tau}} \pz(x) \ind{ \phi(x) \ge \tau_k},
\end{equation} where $E_K$ is a region readily accessible from the state space. The aim of splitting is to estimate the rare event probability at the final $\tau$, i.e.\ $Z_{0, \tau} = \int \pz(x)\, \ind{\phi(x) \ge \tau_0}\, dx = \pE$. We show that all of these methods compute this in the same way. All methods work by transporting particles along the sequence, SMC and splitting additionally resample particles to focus computation on high-weight trajectories (i.e. trajectories that particles undertake that are more likely under the possible sequence of trajectories they could have taken)\cite{del2006sequential, 0a3f6295-587d-3cd6-8835-1ba3c37b93ed, Kahn1951SplittingParticleTransmission}. Intuitively, multilevel splitting is a form of SMC/AIS with a hard indicator function replacing a smoothed distribution and the schemes only differ in their construction of their sequences \cite{moral2004feynman}.

\subsection{Estimating the normalisation constant}
\label{sec:smc_weight_primer}
Having constructed the sequence, the quantity of interest is the terminal normalisation constant $\Zz$, which for the rare-event target is $Z_{0,\tau}$. We show explicitly, that splitting estimates the probability estimate in the same way as standard SMC/AIS. Please see \cite{del2006sequential, neal1998annealedimportancesampling, Kahn1951SplittingParticleTransmission, 0a3f6295-587d-3cd6-8835-1ba3c37b93ed} for further details.

For a sequence of distributions, $\pi_K\to\cdots\to\pi_0$ of Section~\ref{sec:sequential}, let $\pi_k = \frac{1}{Z_k}\gamma_k$ be the normalised density with the normalisation constant $Z_k = \int \gamma_k$. Each step along the trajectory incurs an incremental importance weight that is, in general, a ratio of forward and backward transition kernels between the two distributions \cite{del2006sequential}. A particle is propagated from $\pi_k$ to $\pi_{k-1}$ by the proposal
kernel $\Bk(x_{k-1}\mid x_k)$, and the auxiliary kernel $\Fk(x_k\mid x_{k-1})$ running in the opposite direction
defines the incremental weight of Eq.~\eqref{eq:generic_weight} \cite{del2006sequential},
\begin{equation}
\label{eq:smc_incremental_weight}
  \gk(x_{k-1}, x_k) \;=\; \frac{\gamma_{k-1}(x_{k-1})\,\Fk(x_k\mid x_{k-1})}{\gamma_k(x_k)\,\Bk(x_{k-1}\mid x_k)},
\end{equation}

Choosing the proposal to be a $\pi_{k-1}$-invariant MCMC kernel, as in the rejuvenation move of Section~\ref{sec:sequential}, with the auxiliary kernel $\Fk$ taken as its time reversal, the kernels cancel and the weight reduces to the target ratio
$\gk = \gamma_{k-1}(x_k)/\gamma_k(x_k)$
\cite{del2006sequential, neal1998annealedimportancesampling}, which we use throughout this subsection. The ratio of two normalisation constants $Z_{k-1}/Z_k$ is estimated via the Monte Carlo average of the incremental weights, i.e.,
\begin{equation}
\label{eq:Z_ratio}
  \E_{\pi_k}[\gk]
  = \int \frac{\gamma_{k-1}(x)}{\gamma_k(x)}\,\frac{\gamma_k(x)}{Z_k}\,dx
  = \frac{Z_{k-1}}{Z_k},
\end{equation}
with the standard empirical estimator,
\begin{equation}
  \widehat{\frac{Z_{k-1}}{Z_k}} = \frac{1}{N}\sum_{i=1}^{N} \gk^{(i)}, \qquad x^{(i)}\sim\pi_k.
\end{equation}
Hence the normaliser $\Zz$ follows from the telescoping product with the corresponding estimator given by,
\begin{equation}
    \label{eq:Z_telescoping}
    \hat \Zz = Z_K\prod_{k=1}^{K}\Big(\frac{1}{N}\sum_{i=1}^{N} \gk^{(i)}\Big),
\end{equation}
where $Z_K$ is known, e.g. for score based models $Z_K = 1$ as $\gamma_K = p_1$ is the normalised Gaussian prior \cite{song2021scorebasedgenerativemodelingstochastic, NEURIPS2020_4c5bcfec}. 

\paragraph{Tempering (AIS/SMC)} Annealed importance sampling and SMC interpolate by tempering the reward, $\gamma_k = \pz\,e^{\beta_k r}$, with $1=\beta_0>\cdots>\beta_K=0$, so the incremental weight is given by
\begin{equation}
\label{eq:Z_ais}
  \gk = e^{(\beta_{k-1}-\beta_k)\,r(x_k)},
  \qquad
  \frac{Z_{k-1}}{Z_k} = \E_{\pi_k}\!\big[e^{(\beta_{k-1}-\beta_k)\,r}\big].
\end{equation}
Each normalisation ratio is the average of this factor under the current target. When consecutive temperatures are close, i.e. $\beta_{k-1} - \beta_k$ small, the weights stay close to one for all particles and the ratio is estimated reliably. A single large step, by contrast can produce weights spanning many orders of magnitude - collapsing the effective sample size where the final estimate will be dominated by only a few particles that have the largest weights \cite{neal1998annealedimportancesampling}. 

\paragraph{Splitting} Multilevel splitting interpolates instead with hard indicators, i.e. $\gamma_k = \pz\,\ind{x\in E_k}$, for a shrinking sequence $E_K \supset \cdots \supset E_0$. The incremental weight collapses to,
\begin{equation}
\label{eq:Z_split_weight}
  \gk = \frac{\pz\ind{x\in E_{k-1}}}{\pz\ind{x\in E_k}} = \ind{x\in E_{k-1}},
  \end{equation}
for which the ratio of normalisation constant is given by,
\begin{equation}
 \label{eq:splitting_norm_estimator}   
\ \frac{Z_{k-1}}{Z_k} = \E_{\pi_k}\!\big[\ind{x\in E_{k-1}}\big]
  = \Pr[X\in E_{k-1}\mid X\in E_k] =: \alpha_k,
\end{equation}
so each ratio is the conditional probability of reaching the next level, estimated by the surviving fraction,  $\hat\alpha_k = N_{X\in E_{k-1}}/N_{X\in E_k}$. Thus the telescoping product of  Eq.~\eqref{eq:Z_telescoping} for the threshold exceedance given in Eq.~\eqref{eq:splitting_sequence_tau} becomes,
\begin{equation}
\label{eq:Z_split_product}
 \Zz = Z_K\prod_{k=1}^{K}\alpha_k
  \;=\; \Pr[X>\tau_K]\prod_{k=1}^{K}\Pr[X>\tau_{k-1}\mid X>\tau_k],
\end{equation}
This recovers the rare event probability estimate of the multilevel splitting algorithm \cite{rubino2009rare, morio2014survey}.

Both methods estimate the same object via a product of incremental ratios along the bridge where consecutive levels overlap, instead of the single rare-event ratio estimated in one step.  This is precisely the failure that accumulating importance weights over the trajectory sidesteps. The diffusion process supplies the sequence of transition kernels and the incremental weights \cite{doucet2022scorebased, vargas2023denoising} carry their own normalisation avoiding a separate estimate of $\Zz$, see Appendix~\ref{sec:why_not_snis}.
\subsection{Generic SMC algorithm}
\label{app:generic_smc_algorithm}

Given an unnormalised density $\gamma_0$ with a normalised form $\pi_0 := \gamma_0/\gamma_0[1]$, we seek to estimate integrals $\gamma_0[f]$ and to sample from $\pi_0$. SMC samplers \cite{del2006sequential} transport a population of $N$ weighted particles from a normalised reference $\gamma_K$ down to $\gamma_0$ along a sequence $\gamma_K,\dots,\gamma_0$, with $\pi_k := \gamma_k/\gamma_k[1]$. 

\textbf{Algorithm.} We begin by initialising $N$ samples, $X_K^n\sim \pref$ with weights $w_K^n=1$. At each step $k$, move each particle with a backward proposal kernel $\Bk$ and update its weight with the incremental weight $\gk$,
\begin{equation}
\label{eq:smc_incremental_weight_app}
  X_{k-1}^n\sim \Bk(\dd x_{k-1}\mid X_k^n),\qquad w_{k-1}^n = w_k^n\,\gk\big(X_{k-1}^n, X_k^n\big),
\end{equation}
\begin{equation}
\label{eq:generic_weight_app}
  \gk(x_{k-1}, x_k) = \frac{\gamma_{k-1}(x_{k-1})\,\Fk(x_k \mid x_{k-1})}{\gamma_k(x_k)\,\Bk(x_{k-1}\mid x_k)},
\end{equation}
where $\Fk$ is a forward auxiliary kernel running from $k-1$ to $k$. Any kernels can be selected, provided we can sample from $\Bk$ and evaluate $\gk$ pointwise. When the effective sample size $\ESS(w_{k-1}) := \|w_{k-1}\|_1^2 / \|w_{k-1}\|_2^2$ drops below $\eta N$, we resample, replacing each particle by an ancestor drawn in proportion to its weight and resetting the weights to one \cite{chopin2022resampling}, and rejuvenate using a $\pi_{k-1}$-invariant MCMC kernel.

\textbf{Estimators.} Let $\hat Z$ be the product, over the resampling steps so far, of the average weight $\frac1N\sum_n w^n_k$ just before the reset ($\hat Z=1$ if none). The particle system at step $k$ defines the weighted approximations
\begin{equation}
\label{eq:smc_estimators}
    \hat\gamma_0[f] := \hat Z\,\frac{1}{N}\sum_{n=1}^N w^n_0\,f(X^n_0),
    \qquad
    \hat\pi_0(\dd x) := \frac{\sum_{n=1}^N w^n_0\,\delta_{X^n_0}(\dd x)}{\sum_{n=1}^N w^n_0}.
\end{equation}
$\hat{\gamma}_0[f]$ is \emph{unbiased}, $\E[\hat\gamma_k[f]] = \gamma_k[f]$ for every $f$. The second is biased for finite $N$ but converges almost surely to $\pi_k$ as $N \rightarrow \infty$ \cite{del2006sequential}.
\subsection{Why annealing becomes necessary}
\label{sec:why_not_snis}
We now motivate why annealing is required and where importance sampling alone is insufficient. Even with \emph{exact} samples from the tilted target $\piz$ of Eq.~\eqref{eq:target_sequence}, taken with $\rz = \rsoft$ or the mollified reward. The naive endpoint (self-normalised) estimator of $\pE$ is badly behaved for estimating the normalisation constant, motivating the use of annealing, either via splitting or accumulated importance weights.

Given a smooth reward $\rsoft(x)$, suppose we can draw exact samples $x \sim \piz$ from the tilted distribution, $\piz = \gamma_0/\Zz$ with $\gamma_0 = \pz e^{\rsoft}$. The rare-event probability under the base distribution $\pz$ can then be written as,
\begin{equation}
\label{eq:rare_event}
\begin{aligned}
    \pE = \E_{\piz}\!\left[\ind{\phi(x)\ge \tau}\,\frac{\pz(x)}{\piz(x)}\right]
    &= \E_{\piz}\!\left[\ind{\phi(x)\ge \tau}\,\frac{\pz(x)}{\pz(x) e^{\rsoft(x)}/\Zz}\right] \\
    &= \E_{\piz}\!\left[\ind{\phi(x)\ge \tau}\,\Zz e^{-\rsoft(x)}\right],
\end{aligned}
\end{equation}
where $w(x) := \pz(x)/\piz(x) = \Zz e^{-\rsoft(x)}$ is the importance weight, and the normaliser $\Zz$ is estimated from the same samples via,
\begin{equation}
    \frac{1}{\Zz} = \frac{1}{\Zz}\int{\pz} dx = \frac{1}{\Zz} \int \pz(x) e^{\rsoft(x)} e^{-\rsoft(x)} =  \int \piz e^{-\rsoft(x)}= \E_{\piz}[e^{-\rsoft(x)}].
\label{eq:SNIS}
\end{equation}
Substituting Eq.~\eqref{eq:SNIS} into Eq.~\eqref{eq:rare_event} yields the
self-normalised importance-sampling (SNIS) estimator \cite{mcbook},
\begin{equation}
  \hpE
  \;=\;
  \frac{\sum_{i=1}^N \ind{\phi(x_i)\ge \tau}\,e^{-\rsoft(x_i)}}
       {\sum_{i=1}^N e^{-\rsoft(x_i)}},
       \qquad x_i \sim \piz(x)
  \label{eq:snis_pe_estimate}
\end{equation}

The numerator and denominator of \eqref{eq:snis_pe_estimate} behave very differently. The numerator is only affected by the region that $\piz$ approximates well, since the indicator $\ind{\phi(x)\ge\tau}$ provides a ``shielding'' effect. The denominator, by contrast is an expectation of the normalisation constant $\Zz$ which is affected by the entire domain of $\pz$. As the tilt is designed to concentrate on the rare event set, we are effectively using the rare-event tilted distribution $\piz$ to estimate a function of $\pz$. This naturally has a low ESS and yields a high-variance estimate of the normaliser $\Zz$. Even given perfect samples from $\piz$ we essentially have another ``rare-event'' expectation in the denominator where no particles of $\piz$ fall into the bulk. The denominator integrand $w(x) = \frac{\pz(x)}{\piz(x)} = \Zz e^{-\rsoft(x)}$ is largest in the bulk, where $\pz$ holds most of its mass but the tilt $\piz$ has almost none, \emph{exactly the region $\piz$ is designed to avoid}. The numerator $\E_q[\inds{\text{rare}}e^{-\rsoft}]$ stays bounded as it lives where $q$ samples effectively, whereas the denominator $\E_q[e^{-\rsoft}]$ needs exactly the bulk samples that $q$ avoids. Hence, the ESS for estimating the denominator becomes worse by design and this contributes a high variance estimate of $\Zz$!  We could in practice break this down into a series of transition importance distributions, which leads us back to splitting/annealing of Section \ref{app:annealing_splitting} \cite{morio2014survey, del2006sequential, neal1998annealedimportancesampling}.

\subsubsection{Empirical evidence on the 2D GMM}

If we draw exact samples from $\piz$, via Rejection Sampling or analytically, we illustrate that we still \emph{cannot} estimate $\Zz$. For the 2D GMM case with a log-CDF tilting, we have access to the analytical normalisation constant, the convolution of a Gaussian with a Gaussian CDF simply inflates the variance. The normaliser of the tilted distribution as given in Eq.~\eqref{eq:corr_gmm_Z} repeated here again is given by, 
\begin{equation}
\label{eq:compute_norm}
    \Zsoft = \int \Phi(f) \N(f \mid \mu, \sigma^2) \, df = \Phi\left( \frac{\mu}{\sqrt{1 + \sigma^2}} \right),
\end{equation}
with $f \mapsto \frac{x - \tau}{\delta}$, $\mu \mapsto \frac{\mu_i - \tau}{\delta}$ and $\sigma^2 \mapsto \frac{v_i}{\delta^2}$ (see Section \ref{app:corr_gmm}). This lets us disentangle whether the failure is in the numerator or the denominator, knowing the analytical denominator and the true probability, the true numerator is recoverable. We draw a large number of samples from $\pz$ and use rejection sampling to generate samples from $\piz$. We find that a small number of particles drawn directly from $\piz$ suffices to estimate the rare-event probability \emph{provided the denominator is known}. Conversely, if the denominator is unknown we cannot estimate the probability at all - the difficulty lies entirely in the denominator as seen in Table \ref{tab:snis_catastrophic}, which is exactly what motivates annealing (Section~\ref{sec:sequential} \& ~\ref{app:annealing_splitting}).

For $\tau = 7.0$ and $\delta = 0.25$ the true $\Zsoft = 7.331 \times 10^{-3}$ from Eq.~\eqref{eq:compute_norm}.
\begin{table}[htpb]
    \centering
    \caption{Self-normalised importance sampling (SNIS) estimates vs.\ sample size ($N$).}
    \label{tab:snis_catastrophic}
    \begin{tabular}{r c c c}
        \toprule
        $N$ & $\hat{p}_{\text{SNIS}}$ & Relative Bias & $\hat{Z}_{0,\text{SNIS}}$ \\
        \midrule
        100    & $4.6035 \times 10^{-1}$ & $+6,593\%$ & $4.873 \times 10^{-1}$ \\
        1,000  & $3.5854 \times 10^{-1}$ & $+5,113\%$ & $3.770 \times 10^{-1}$ \\
        10,000 & $3.1680 \times 10^{-1}$ & $+4,506\%$ & $3.367 \times 10^{-1}$ \\
        25,000 & $1.9804 \times 10^{-1}$ & $+2,779\%$ & $2.102 \times 10^{-1}$ \\
        \midrule
        100 & $6.9250 \times 10^{-3}$ &  $+0.68\%$  & (analytical $\Zsoft$) \\
        \bottomrule
    \end{tabular}
\end{table}

Note that $\hat{p}_{\text{SNIS}}$ differs from $\hat{Z}_{0,\text{SNIS}}$ because some particles do not pass through the numerator's indicator by design of the smoothing in Section~\ref{sec:method}. The goal of annealing/splitting (Section~\ref{sec:sequential}) is precisely to make the numerator easier to estimate.

\subsection{Splitting estimator}
\label{sec:splitting_estimation}

In the case where the path weights are not readily accessible, either because they are not computed or if a large number of discretisation is infeasible, e.g. for a second order numerical integrator. We can fall back to a guided diffusion proposal to build a splitting/annealed probability estimator. In this case, the annealing bridge becomes the nested sequence of event sets defined in Eq.~\eqref{eq:splitting_sequence_tau}. For the threshold exceedance example given in Sec \ref{app:annealing_splitting}, the conditional ratio between adjacent levels is given by,
\begin{equation}
\label{eq:alpha}
  \alpha _k\;=\; \pz[X > \tauh \mid X > \taul] = \frac{\pz[ X > \tauh \cap  X > \taul]}{\pz[X > \taul]} = \frac{\pz[X > \tauh]}{ \pz[X > \taul]}.
\end{equation}
In the denoising-diffusion setting each conditional is sampled by transporting particles from the Gaussian prior to the corresponding target, this means that the probability, $\Pr[X > \tauh]$ is estimated by generating samples from the tilted distribution, $p(X \mid  X > \tauh)$, as highlighted in Section~\ref{sec:method}. Within this scheme $\alpha_k$ admits two estimators, depending on whether a single tilt or one tilt per level is used.

\paragraph{Formulation I: single-tilt, two-indicator estimator.} Use a \emph{single} proposal $q_\ell$ targeting the lower level $\tau_\ell$ and estimate both exceedance probabilities from the same particles. With $\tau_\ell<\tau_h$,
\begin{equation}
\label{eq:alpha_single_proposal}
  \alpha_k = \frac{\E_{q_\ell}\!\big[\ind{\phi(X)>\tau_h}\,e^{-r_\ell(X)}\big]}{\E_{q_\ell}\!\big[\ind{\phi(X)>\tau_\ell}\,e^{-r_\ell(X)}\big]},
\end{equation}
where the shared normaliser $Z_\ell$ cancels between numerator and denominator, so it need not be estimated.  The denominator is estimated where $q_\ell$ has mass (it targets $\tau_\ell$), so it is well-conditioned; the variance is dominated by the numerator, which lives where $q_\ell$ is sparse (${X>\tau_h}$). 

\paragraph{Formulation II: two-tilt transitioning.} Use one proposal per level, $q_\ell\propto \pz\,e^{r_\ell}$ and $q_h\propto \pz\,e^{r_h}$, and estimate the numerator under $q_h$ and the denominator under $q_\ell$,
\begin{equation}
\label{eq:alpha_two_tilt}
  \alpha_k = \frac{Z_h}{Z_\ell}\,
  \frac{\E_{q_h}\!\big[\ind{\phi(X)>\tau_h}\,e^{-r_h(X)}\big]}{\E_{q_\ell}\!\big[\ind{\phi(X)>\tau_\ell}\,e^{-r_\ell(X)}\big]}.
\end{equation}
Each expectation is now evaluated under a tilt matched to its own threshold, so neither term suffers the ESS collapse of Formulation~I; the price is that the normalisers no longer cancel and the ratio $Z_h/Z_\ell$ must be estimated. Annealing on the output space, i.e. between $p_0e^{r_\ell} \rightarrow p_0e^{r_h}$, we can compute this ratio with the tempering estimator given in Eq.~\eqref{eq:Z_ais},
\begin{equation}
  \frac{Z_h}{Z_\ell} = \E_{q_\ell}\!\big[e^{\,r_h(X)-r_\ell(X)}\big]
  \approx \frac{1}{N}\sum_i e^{\,r_h(x_i)-r_\ell(x_i)}, \qquad x_i\sim q_\ell,
\end{equation}
which is reliable when adjacent tilts overlap. Both formulations require sufficient overlap of the thresholds, the only difference is Formulation~I requires overlap between two hard indicators and through the smooth reward difference, $r_h - r_\ell$, in Formulation~II.

\paragraph{Notable Caveats:}  Evolving directly from $p(X \mid X > \tau_l) \rightarrow p(X \mid X > \tau_h)$ is non-trivial without additional training. Unlike traditional splitting, we do not evolve/reuse particles across different thresholds \cite{morio2014survey}, hence the above formulations are not a form of sequential methods. Instead they are a sequence of independent estimators sharing a compute budget.  Additionally, any importance weights used assume that within the threshold set by the indicators, samples are well distributed along the tilted distribution. As such both estimators are sensitive to the quality of the guidance. Training based methods could be introduced such as Flow Matching \cite{lipman2022flow, liu2022flow, albergo2022building} or Schrödinger's Bridges \cite{de2021diffusion, shi2023diffusion} but these remain beyond the scope of this work.

\section{Weight computation}
\label{sec:gaussian_ratio_sigma_notation}
Let $\Pm$ and $\Qm$ be the path-measures on $[0, 1]$ of the base reverse-time SDE of Eq.~\eqref{eq:reverse_sde_sigma} and the guided SDE of Eq.~\eqref{eq:guided_sde} respectively. Here we do not assume that $\Qm$ is the ideal Doob's path measure $\Pm^*$ and the gap between the two is varaince in Theorem~\ref{thm:variance}. Both measures start from $X_1 \sim \pref$ and differ only the drift obtained by the reward gradient term, $-\sigma^2_t \nabla_x r_t$ so Girsanov's theorem gives us an exact form of the Radon-Nikodym Derivative \cite{oksendal2003stochastic},
\begin{equation}
\label{eq:app_girsanov}
\dPQ
= \exp\!\Big(\!-\!\int_0^1 \sigma_t\, \nabla_x \rt(X_t)\cdot\dd\Wb_t
\;-\;\tfrac12\int_0^1 \sigma_t^2\,\|\nabla_x \rt(X_t)\|^2\,\dd t\Big),
\end{equation}
where $\Wb_t$ is the reverse-time Wiener process driving the
guided sampler. This expression is valid provided the change of measure exists
at all. A sufficient condition for existence is given by Novikov's condition,
\begin{equation}
\label{eq:app_novikov}
\E_{\Qm}\!\Big[\exp\Big(\tfrac12\int_0^1 \sigma_t^2\,\|
\nabla_x \rt(X_t)\|^2\,\dd t\Big)\Big] < \infty.
\end{equation}
The discrete estimator of Algorithm~\ref{alg:rne_ais} does not require this condition to be fulfilled as Theorem~\ref{thm:main} requires only $\Pk\ll\Qk$ (App.~\ref{app:proofs}).

\paragraph{Discrete transition kernels}
In practice we sample on the grid $0 = t_0 < \cdots < t_K = 1$ of Section~\ref{sec:background}, with $t=0$ the data distribution, writing $x_k$ for the state at $t_k$ so that step $k$ moves $x_k \to x_{k-1}$. The path measures factorise into the backward transition kernels of Section~\ref{sec:background} and the kernel we actually simulate \cite{wu2023practical},
\begin{equation}
\label{eq:app_path_factorise}
\begin{aligned}
  \Pmk(x_{0:K}) &= \pref(x_K)\!\prod_{k=1}^{K}\! \Pk(x_{k-1}\mid x_k),\\
  \Qmk(x_{0:K}) &= \pref(x_K)\!\prod_{k=1}^{K}\! \Qk(x_{k-1}\mid x_k),
\end{aligned}
\end{equation}
where $\Pk$ integrates the base reverse SDE of Eq.~\eqref{eq:reverse_sde_sigma} and $\Qk$ the guided SDE of Eq.~\eqref{eq:guided_sde} from $t_k$ down to $t_{k-1}$. Both are initialised from the same Gaussian prior $\pref$ which cancels,
\begin{equation}
\label{eq:app_dpdq}
  \frac{\dd\Pmk}{\dd\Qmk} = \prod_{k=1}^{K} \frac{\Pk(x_{k-1}\mid x_k)}{\Qk(x_{k-1}\mid x_k)}.
\end{equation}
Incorporating a reward tilt $e^{\rz(x_0)}$ as a telescope product of individual reward potentials, 
\begin{equation}
\label{eq:app_reward_telescope}
  R(x_{0:K}) = \prod_{k=1}^{K} G_k, \qquad G_k = \frac{e^{r_{k-1}(x_{k-1})}}{e^{r_k(x_k)}},
\end{equation}

which collapses to $R = \exp\!\big(r_{0}(x_0) - r_{K}(x_K)\big) = e^{\rz(x_0)}$ since $r_K =0$. Multiply Eqs.~\eqref{eq:app_dpdq} and \eqref{eq:app_reward_telescope} yields,
\begin{equation}
\label{eq:app_weight_product}
  e^{\rz(x_0)}\,\frac{\dd\Pmk}{\dd\Qmk}(x_{K:0}) = \prod_{k=1}^{K} G_k\,\frac{\Pk(x_{k-1}\mid x_k)}{\Qk(x_{k-1}\mid x_k)} = \prod_{k=1}^{K}\gk(x_{k-1},x_k),
\end{equation}
which recovers the incremental weight of Eq.~\eqref{eq:kernel_ratio_weight} accumulated by Algorithm~\ref{alg:rne_ais}.

\paragraph{Weight Computation}
\label{app:rne_weights}
Recall from Lemma~\ref{prop:omega_sufficent} that, for Gaussian transition kernels of Eq.~\eqref{eq:gaussian_kernels}, the incremental log weight is given by,

\begin{equation}
\label{eq:dlogw_explicit}
  \log \gk(x_{k-1}, x_k)
  \;=\;
  \underbrace{r_{k-1}(x_{k-1}) - r_k(x_k)}_{\text{reward telescope}}
  \;-\;
  \underbrace{\frac{\om_k^2}{2} \| \nabla_x r_k(x_k) \|^2 }_{\text{control cost}}
  \;-\;
  \underbrace{\om_k \langle \nabla_x r_k(x_k), \epsilon_k \rangle }_{\text{martingale correction}}.
\end{equation}

 As all path-measures we have introduced differ only in the drift, the change of measure between them is given by Girsanov's theorem which yields an exact Radon-Nikodym Derivative \cite{oksendal2003stochastic} and the above computation represents its discretised form. The control-cost is the price paid by the guided SDE for perturbing the base drift towards the rare-event set and a martingale correction which accounts for any ``lucky'' noise $\epsilon_k$ sampled along the guided step that carries the particle towards the set.

\textbf{Remark.} Nothing in this derivation constrains the $r_{k}$ to be constant across time; and any family of schedules defined by $(\lambda(t), \tau(t), \delta(t))$ do not break the validity and the weight computation remains exact. The schedules enter the reward telescope of Eq.~\eqref{eq:app_reward_telescope} only through the endpoints themselves $R = \exp\!\big(r_{t_0}(x_0) - r_{t_K}(x_K)\big)$, and the kernel ratio of Eq.~\eqref{eq:app_dpdq} only through the guidance $\nabla_{x_t} r_t$.

\subsection{$\om_k$ for common integrators}
\label{app:gamma_derieved}
The first-order exponential integrator \cite{milstein2004stochastic} of \cite{gonzalez2023seeds} integrates the linear part of the reverse drift and the noise exactly, freezing only the network output across the step. Reparameterising the time increment in half-log Signal to Noise Ratio, $\ell_t := \log (\alpha_t / \sbar_t)$; write $\Delta \ell_k = \ell_{t_{k-1}} - \ell_{t_k} \ge 0$ for its increment across a step. 

\paragraph{Exponential Integrator (VP \& $\hat{\epsilon}_\theta $ frozen)} Restating their Eq.~14 \cite{gonzalez2023seeds} in our notation and abbreviating the symbols, 
\begin{equation}
A_k := \frac{\alpha_{t_{k-1}}}{\alpha_{t_k}},
  \qquad
  D_k := 2\,\sbar_{t_{k-1}}\bigl(e^{\Delta\ell_k}-1\bigr),
  \qquad
  v_k := \sbar^2_{t_{k-1}}\bigl(e^{2\Delta\ell_k}-1\bigr).
\end{equation}
The incremental reverse step is given by,
\begin{equation}
\label{c4:step}
  x_{k-1} = A_k\,x_k
  \;-\; D_k\hat\epsilon_\theta\bigl(x_k,t_k\bigr)
  \;+\; \sqrt{v_k} \epsilon_k,
  \qquad
  \epsilon_k \sim \N(0, I),
\end{equation}
where $\hat{\epsilon}_\theta$ is the noise predicted by the network. This is related by a change of variable to the score network, $\sth = - \frac{\epsilon_\theta}{\sbar}$ \cite{gao2025diffusionmeetsflow}. Hence, the guidance update in score maps to the noise prediction as,
\begin{equation}
\label{eq:change_noise_score}
\begin{aligned}
    \sth  &\rightarrow \sth  + \nabla_x r_t, \\
    \epsilon_\theta &\rightarrow \epsilon_\theta - \sbar_{t_k} \nabla_x r_{t_k}(x_k).
\end{aligned}
\end{equation}
Therefore, the change of means between the two guided and unguided process is obtained by substituting Eq.\eqref{eq:change_noise_score} into Eq.\eqref{c4:step} to get,
\begin{equation}
    \mu^a_k-\mu^\nu_k = D_k \sbar_{t_k} \nabla_x r_{t_k}(x_k) 
    = c_{k}^{\text{VP}} \nabla_x r_{t_k}(x_k),
\end{equation}
where we have further equated this to the mean shift of Eq.~\eqref{eq:gaussian_kernels}.
Hence, under the exponential first-order integrator and a VP noise schedule, 
\begin{equation}
    c_{k}^{\text{VP}} = D_k \sbar_{t_k} = 2 \sbar_{t_{k-1}}\sbar_{t_k}(e^{\Delta\ell_k}-1).
\end{equation}
To obtain $\om_k$, we make further substitutions,
\begin{equation}
\label{eq:gamma_exp}
\begin{split}
    \big(\om^{\mathrm{VP}}_k\big)^2 = \frac{(c_k^{\mathrm{VP}})^2}{v_k}
    &= \frac{4\,\sbar^2_{t_{k-1}}\sbar^2_{t_k}\bigl(e^{\Delta\ell_k}-1\bigr)^2}
            {\sbar^2_{t_{k-1}}\bigl(e^{2\Delta\ell_k}-1\bigr)} \\
    &= \frac{4\,\sbar^2_{t_k}\bigl(e^{\Delta\ell_k}-1\bigr)^2}
            {\bigl(e^{\Delta\ell_k}-1\bigr)\bigl(e^{\Delta\ell_k}+1\bigr)} \\
    &= \frac{4\,\sbar^2_{t_k}\bigl(e^{\Delta\ell_k}-1\bigr)}
            {e^{\Delta\ell_k}+1} \\
    &= 4\,\sbar^2_{t_k} \tanh({\Delta \ell_k /2}),
\end{split}
\end{equation}
so that $\om^{\mathrm{VP}}_k = 2\,\sbar_{t_k}\sqrt{\tanh(\Delta\ell_k/2)}$.

\paragraph{Exponential Integrator (VE)} Under the VE schedule, $f_t = 0$ and $\alpha_t \equiv 1$. Following \cite{gonzalez2023seeds} and rewriting the score network as a data-prediction $\xh = x_t + \sbar_t^2 \sth$ using Tweedie's formula \cite{Efron01122011, gao2025diffusionmeetsflow} and freezing $\xh$ along the increment, the reverse drift is given by,
\begin{equation}
    \label{eq:reverse_drift_ve}
    x_{k-1} = \kappa_k x_k + (1 - \kappa_k) \xh + \sqrt{v_k} \epsilon_k, \qquad v_k = \sbar^2_{t_{k-1}}(1 - \kappa_{k}),
\end{equation}
where $\kappa_k := \frac{\sbar^2_{t_{k-1}}}{\sbar^2_{t_k}}$. Under a guided score the data prediction is perturbed to,
\begin{equation}
\begin{aligned}
    \sth  &\rightarrow \sth  + \nabla_x r_t, \\
    \xh &\rightarrow \xh + \sbar^2_{t_k}\nabla_x r_{t_k}(x_k),
\end{aligned}
\end{equation}
hence, $c_k = (1 - \kappa_k)\sbar^2_{t_k}$ and it follows,
\begin{equation}
    \om^{\text{VE}, \hat{x}}_k = \sbar_{t_k}\sqrt{e^{2 \Delta \ell_k } - 1},
\end{equation}
with $\ell_t = - \log \sbar_t$ when $\alpha_t \equiv 1$. 

Freezing the noise-prediction $\hat{\epsilon}_k$ would lead to an exact quadrature scheme of the noise term, $v_k$.
Under a VE schedule $f_t = 0$, so the reverse drift
carries no part linear in $x$ and, using $\sigma_t^2/\sbar_t = 2\dot{\sbar}_t$, the frozen
integrand may be integrated in closed form \cite{gonzalez2023seeds},
\begin{equation}
\label{eq:reverse_drift_ve_eps}
  x_{k-1} = x_k - D^{\text{VE}}_k \hat\epsilon_\theta(x_k, t_k) + \sqrt{v_k}\,\epsilon_k,
  \qquad
  D^{\text{VE}}_k := 2\bigl(\sbar_{t_k} - \sbar_{t_{k-1}}\bigr),
  \qquad
  v_k := \sbar^2_{t_k} - \sbar^2_{t_{k-1}},
\end{equation}
Substituting the guided noise-prediction of Eq.~\eqref{eq:change_noise_score} gives
$c^{\text{VE}}_k = D^{\text{VE}}_k \sbar_{t_k} = 2\sbar_{t_k}(\sbar_{t_k} - \sbar_{t_{k-1}})$, and
since $\sbar_{t_{k-1}} = \sbar_{t_k} e^{-\Delta\ell_k}$ when $\alpha_t \equiv 1$, both
coefficients collapse onto the VP forms,
\begin{equation}
  c^{\text{VE}}_k = 2\,\sbar_{t_{k-1}}\sbar_{t_k}\bigl(e^{\Delta\ell_k}-1\bigr),
  \qquad
  v_k = \sbar^2_{t_{k-1}}\bigl(e^{2\Delta\ell_k}-1\bigr),
\end{equation}
so that $\om^{\text{VE, $\hat{\epsilon}$}}_k = 2\,\sbar_{t_k}\sqrt{\tanh(\Delta\ell_k/2)}$, identical to
$\om^{\mathrm{VP}}_k$ of Eq.~\eqref{eq:gamma_exp}.

\paragraph{Euler--Maruyama} Euler--Maruyama steps \cite{kloeden1992numerical} discretise the reverse SDE directly, so the base and guided kernels share a common covariance, $v_k = \sigma^2_{t_k} \Delta t_k$ and $u_t = - \sigma^2_t \nabla_x r_t$. Hence, $c^{EM}_k = \sigma^2_{t_k} \Delta t_k$, and,
\begin{equation}
    \om^{EM}_k = \sigma_{t_k} \sqrt{\Delta t_k}.
\end{equation}

Table~\ref{tab:gamma_summary} collects the three cases. Notably, as we increase the number of discretisations or $\Delta\ell_k \rightarrow 0$, all variants given in Table~\ref{tab:gamma_summary} recover the EM $\om^2_k$.

\begin{table}[h]
\centering
\caption{$\om_k^2 = c_k^2/v_k$ for different integrators and noise schedules.}
\label{tab:gamma_summary}
\setlength{\tabcolsep}{5pt}%
\begin{tabular}{lccc}
\toprule
 & $c_k$ & $v_k$ & $\om^2_k = c^2_k/v_k$ \\
\midrule
Euler--Maruyama (VP/VE) & $\sigma^2_{t_k}\Delta t_k$ & $\sigma^2_{t_k}\Delta t_k$ & $\sigma^2_{t_k}\Delta t_k$ \\
exponential, VP ($\hat{\epsilon}$ frozen) & $2\sbar_{t_{k-1}}\sbar_{t_k}(e^{\Delta\ell_k}-1)$ & $\sbar^2_{t_{k-1}}(e^{2\Delta\ell_k}-1)$ & $4\,\sbar^2_{t_k}\tanh(\Delta\ell_k/2)$ \\
exponential, VE ($\hat{x}$ frozen) & $(1-\kappa_k)\,\sbar^2_{t_k}$ & $\sbar^2_{t_{k-1}}(1-\kappa_k)$ & $\sbar^2_{t_k}(e^{2\Delta\ell_k}-1)$ \\
exponential, VE ($\hat{\epsilon}$ frozen) & $2\,\sbar_{t_{k-1}}\sbar_{t_k}\bigl(e^{\Delta\ell_k}-1\bigr)$ & $\sbar^2_{t_{k-1}}\bigl(e^{2\Delta\ell_k}-1\bigr)$ & $4\,\sbar^2_{t_k}\tanh(\Delta\ell_k/2)$ \\
\bottomrule
\end{tabular}
\end{table}

\section{The guidance gradient via local linearisation}
\label{app:guidance_gradient}
\begin{figure}[!htb]
    \centering
\includegraphics[width=0.85\linewidth]{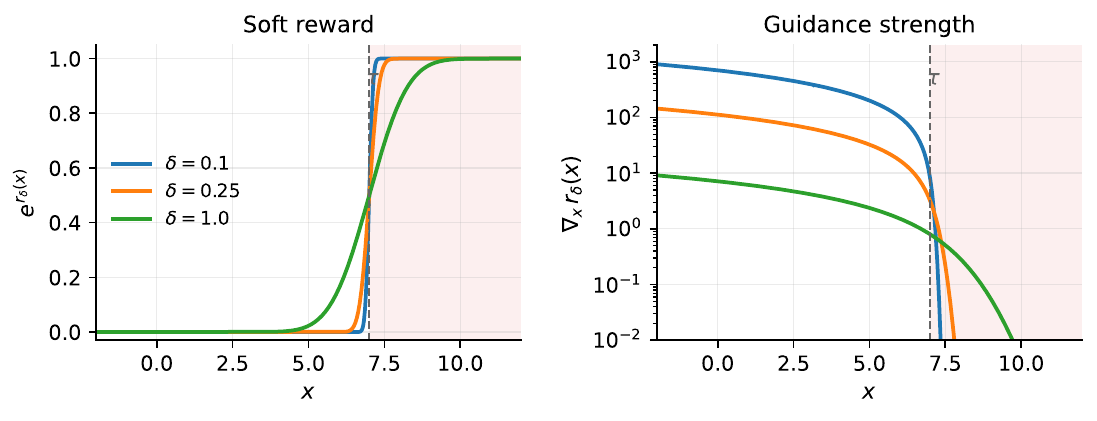}
    \caption{Soft reward and its guidance for a threshold event $E_7=\{x : \phi(x)\ge7\}$ shown for three smoothing levels $\delta$. \emph{Left.} The soft indicator relaxes the hard indicator. \emph{Right.} The guidance $\nabla_x \rsoft(x)$ points toward the event set (shaded) and decays inside it, a smaller $\delta$ gives a stronger guidance.}
    \label{fig:reward_guidance_term}
\end{figure}

The variance-aware guidance, illustrated in Fig.~\ref{fig:reward_guidance_term} used in Section~\ref{sec:method} is,

\begin{equation*}
\nabla_{x_t}\log\E_{p(x_0\mid x_t)}[e^{\rsoft(x_0)}]
\end{equation*} 

This appendix derives the closed-form approximation used in our implementation. We consider a general vector-valued constraint $\phi:\R^d\to\R^m$ using a Gaussian approximation of the denoising posterior and a local linearisation of $\phi$. The scalar affine guidance in Eq.~\eqref{eq:reward_expectation}, used for the one-sided threshold-exceedance experiments, is recovered as a special case $m=1$ with a linear-$\phi$. We then discuss how interval and compound events are represented by multiple one-sided constraints and the approximation induced when their guidance terms are treated independently.

For a multivariate constraint, the smooth reward is a multivariate Gaussian CDF
$\Phi_m(\,\cdot\mid\Sigma_c)$ with constraint covariance $\Sigma_c$ (the
smoothing bandwidth; in the scalar case $\Sigma_c=\delta^2$). The integral is given by,
\begin{equation}
  I(x_t) \;=\; \E_{p(x_0\mid x_t)}\!\left[\Phi_m\big(\phi(x_0)-\tau \,\big|\, \Sigma_c\big)\right].
\end{equation}

\paragraph{Gaussian posterior.} We use the same Gaussian denoising-posterior
approximation as in the main text, $p(x_0\mid x_t)\approx
\N(x_0\mid \xh,\Sig)$ with $\xh$ the Tweedie estimate and $\Sig$ the diagonal $\Pi$GDM posterior-covariance proxy (Appendix~\ref{app:diffusion_details}), hence,
\begin{equation}
  I(x_t) \;\approx\; \int_{\R^d}\N(x_0\mid \xh,\Sig)\,
  \Phi_m\big(\phi(x_0)-\tau \,\big|\, \Sigma_c\big)\,\mathrm dx_0 .
\end{equation}

\paragraph{Linearisation.} Linearising the constraint about the Tweedie mean,
\begin{equation}
  \phi(x_0) \;\approx\; \phi(\xh) + J_\phi(\xh)\,(x_0-\xh),
  \qquad J_\phi(\xh)=\nabla_{x_0}\phi\big|_{\xh}\in\R^{m\times d},
\end{equation}
makes the integrand a Gaussian convolved with a Gaussian CDF, which integrates
in closed form to a single CDF with the two covariances added \cite{10.7551/mitpress/3206.001.0001},
\begin{equation}
\label{eq:lookahead_multivariate}
  I(x_t) \;\approx\; \Phi_m\!\left(\phi(\xh)-\tau \,\Big|\, \Sigma_c + J_\phi(\xh)\,\Sig\,J_\phi(\xh)^\top\right).
\end{equation}
The posterior covariance $\Sig$ is thus pushed through the constraint
Jacobian and inflates the effective observation variance.

\paragraph{Guidance gradient.} Let $\pit \propto \pt\,e^{\rt}$ be the continuous-time version of the tilted target $\pi_k$ of Eq.~\eqref{eq:target_sequence}. The guided score is then the unconditional score plus
the gradient of $\log I$,
\begin{equation}
  \nabla_{x_t}\log \pit(x_t)
  \;=\; \underbrace{\nabla_{x_t}\log \pt(x_t)}_{\text{unconditional score}}
  \;+\; \underbrace{\nabla_{x_t}\log \Phi_m\!\left(\phi(\xh)-\tau \,\Big|\, \Sigma_c + J_\phi(\xh)\,\Sig\,J_\phi(\xh)^\top\right)}_{\text{variance-aware guidance gradient}}.
\end{equation}

\paragraph{Linear constraint.} For an affine read-out $\phi(x_0)=Ax_0+b$ the
Jacobian is constant, $J_\phi=A$, and the guidance simplifies to
\begin{equation}
  \nabla_{x_t}\log \pit(x_t)
  \;=\; \nabla_{x_t}\log \pt(x_t)
  \;+\; \nabla_{x_t}\log \Phi_m\!\left(A\xh+b-\tau \,\Big|\, \Sigma_c + A\,\Sig\,A^\top\right).
\end{equation}
Specialising further to $m=1$ with $\phi(x)=a^\top x$ and $\Sigma_c=\delta^2$ gives the scalar guidance used in the main text,
\begin{equation}
\label{eq:reward_expectation}
\begin{aligned}
     \nabla_{x_t} \log \E_{p(x_0 \mid x_t)} [\exp(\rsoft(x_0))]
     &= \nabla_{x_t} \log\int p(x_0 \mid x_t)\,\Phi\!\left(\frac{\phi(x_0)-\tau}{\delta}\right) dx_0 \\
    &\;\approx\;\nabla_{x_t}\log\Phi\!\left(\frac{\phi(\xh)-\tau}{\sqrt{\delta^2 + a^\top\Sig a}}\right),
\end{aligned}
\end{equation}
with the inflated width $\sqrt{\delta^2 + a^\top\Sig a}$.

\subsection{Factorised guidance for multiple constraints}
\label{app:multiple_guidances}

Consider an event defined by the intersection of $m$ one-sided constraints,
\begin{equation}
\label{eq:app_multiple_events}
    \Ev=\bigcap_{j=1}^{m}
    \left\{
        \phi_j(x) \geq \tau_j
    \right\},
\end{equation}
such as the boxed intervals of App.~\ref{app:box_events}. Section~\ref{app:guidance_gradient} gives the corresponding look-ahead likelihood as a multivariate Gaussian CDF $\Phi_m(\,\cdot \mid \Sigma_{\mathrm{eff}})$ with
\begin{equation}
    \Sigma_{\mathrm{eff}} = \Sigma_c + J_\phi(\xh)\,\Sig\,J_\phi(\xh)^\top,
\end{equation}
where $\Sigma_c$ is the user-specified smoothing covariance and the second term is the
denoising-posterior uncertainty projected through the read-out (Eq.~\eqref{eq:lookahead_multivariate}).

To obtain the guidance gradient, we require $\nabla_{x_t}\log\Phi_m$, which in turn requires evaluating $\Phi_m$ itself. For $m=1$, this is available in closed form and for $m=2, 3$ it reduces to a low-dimensional quadrature \cite{owen1956tables, drezner1994computation, genz2004numerical}. For $m \ge 4$ with an unstructured or full covariance matrix $\Sigma_\mathrm{eff}$ no such reduction exists and the function has to be approximated by (randomised) quasi-Monte Carlo techniques \cite{83355a14-2760-31d3-afd2-16d5bbbc1183}. The cost of that approximation is incurred once per particle per diffusion step, and unless the evaluation points are held fixed across particles and steps, the reward and its gradient become stochastic rather than deterministic functions of $x_t$. We therefore also provide the factorised surrogate below.

Optionally, $\Sigma_\mathrm{eff}$ may be approximated by its diagonal, under which $\Phi_m$
factorises into a product of univariate Gaussian CDFs,
\begin{equation}
\label{eq:factorised_lookahead}
    I_{\mathrm{fact}}(x_t)
    =
    \prod_{j=1}^{m}
    \Phi\!\left(
        \frac{\phi_j(\xh)-\tau_j}{\delta_{\text{eff}, j}(x_t)}
    \right),
    \qquad
    \delta_{\mathrm{eff},j}^2(x_t) = \delta_j^2 + a_j^\top \Sig\, a_j ,
    \quad
    a_j = \nabla_{x_0}\phi_j\big|_{\xh},
\end{equation}
with a factorised reward $r^{\mathrm{fact}}_t := \log I_{\mathrm{fact}}$ and guidance, 
\begin{equation}
    \nabla_{x_t} r^{\mathrm{fact}}_t = \sum_{j=1}^{m} \nabla_{x_t}\log
\Phi\left(\frac{\phi_j(\xh)-\tau_j}{\delta_{\text{eff}, j}(x_t)}\right).
\end{equation}
This is simply $m$ independent copies of the scalar guidance and while the factorisation retains the marginal posterior uncertainty of each
constraint, it discards the off-diagonal entries of $\Sigma_{\text{eff}}$. This alters the guidance drift, and hence the implemented proposal $\Qm$ together with the sequence of marginals $\pit$. Our rejuvenation kernels are made invariant with respect to the same marginals, using the factorised surrogate $I_{\mathrm{fact}}$. However, the target itself remains unchanged $\pz e^{\rsoft(x)}$ as it depends on the data-reward at $t=0$ alone and remains unaffected by $\pit$ or $\nabla_x \rt$. Hence, this approximation only contributes to the variance of our estimator.

Note finally that the factorisation describes the tilted marginals and not the sampler. The guided drift is, \[\nabla_{x_t}\log \pt(x_t) + \nabla_{x_t} r^{\mathrm{fact}}_t(x_t),\]  and the unconditional score retains the emulator's full joint dependence structure. Additionally, each term is also mapped through the dense
$\partial\xh/\partial x_t$, so a constraint supported on a single region and channel still induces a global drift across all variables. The diagonality of $\Sigma_{\mathrm{eff}}$ suppresses cross-constraint terms in the look-ahead \emph{variance}; it does not decouple the resulting drift.

\paragraph{Interval Guidance} A useful special case is an event defined by the interval, 
$\Ev = \{\tau_a \leq \phi(x) \leq \tau_b\}$, which is the intersection
\eqref{eq:app_multiple_events} with $m=2$, taking $\phi_1 = \phi$, $\tau_1 = \tau_a$ and
$\phi_2 = -\phi$, $\tau_2 = -\tau_b$. We write the reward as a difference of two individual rewards given by,
\begin{equation}
\label{eq:interval_reward}
    e^{\rsoft(x)}
    =
    \Phi\!\left(\frac{\phi(x)-\tau_a}{\delta}\right)
    -\Phi\!\left(\frac{\phi(x) - \tau_b}{\delta}\right),
\end{equation}
with the look-ahead and guidance following as before, replacing $\delta$ by
$\delta_{\mathrm{eff}}^2 = \delta^2 + a^\top\Sig a$ which is shared across both boundaries.

\section{The correlated 2D GMM benchmark}
\label{app:corr_gmm}
\begin{figure}[!ht]
\centering
\includegraphics[width=0.7\linewidth]{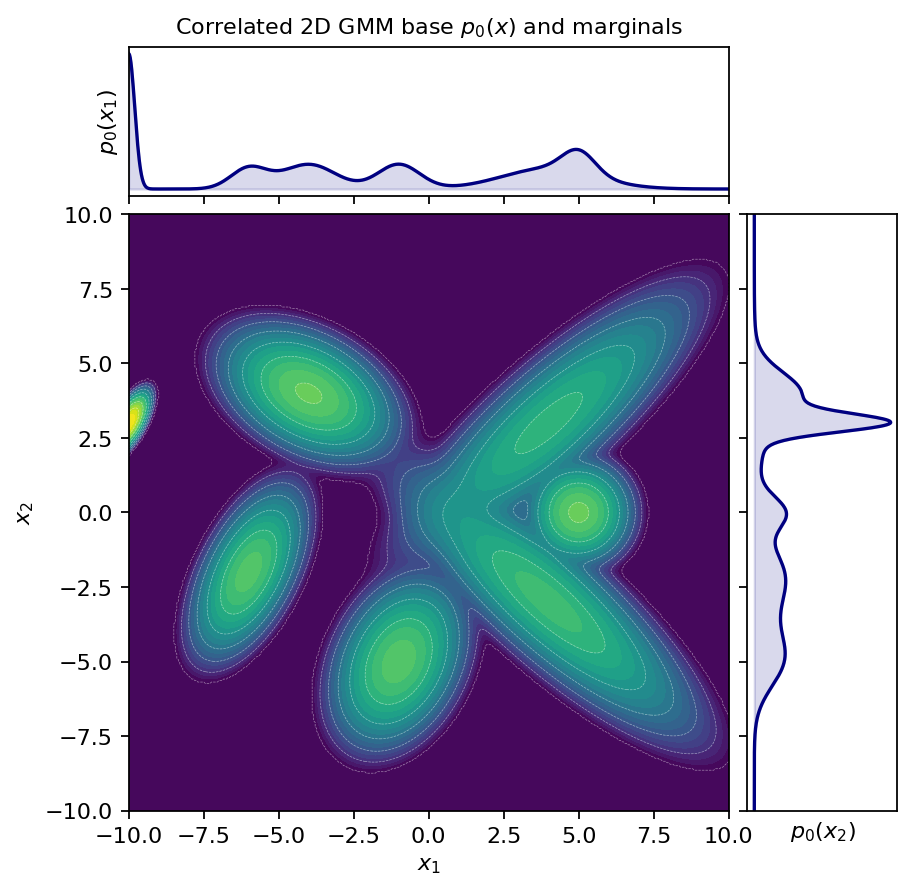}
\caption{Base distribution $\pz(x)$ of the correlated 2D GMM (log-density, so
all modes are visible) with the analytic $x_1$ and $x_2$ marginals. The two
right-hand modes ($i=5,6$) form the tilted ``X'' of opposite-sign correlation
that gives the rare set a bi-modal conditional.}
\label{fig:corr_gmm_base}
\end{figure}

This appendix specifies the synthetic benchmark used in our 2D experiments and
collects the closed-form quantities it admits: the base density, the
normalisation constant $\Zsoft$ of the soft tilt, the exact exceedance
probability $\pE$, and the exact Doob guidance $\nabla_x \rstar=\nabla_x\log G_t$ \cite{dai1991stochastic, denker2024deft}. All of these are available in closed form precisely because the GMM
prior is conjugate to both the Gaussian forward kernel and the Gaussian-CDF
reward, which lets us isolate the behaviour of the estimator from score- and
reward-approximation error \cite{10.7551/mitpress/3206.001.0001}.

\paragraph{Base distribution.}
The data distribution is a $M=7$-component Gaussian mixture in $\R^2$,
\begin{equation}
  \pz(x)\;=\;\sum_{i=1}^{M}w_i\,\N(x;\,\mu_i,\Sigma_i),
  \qquad \sum_i w_i = 1,
\end{equation}
where $w_i$ ($i=1,\dots,M$) denote the mixture weights.
The parameters are given in
Table~\ref{tab:corr_gmm}. The four ``bulk'' modes carry
most of the mass near the origin; the three right-hand modes are pushed out
towards large $x_1$ and given strong off-diagonal covariance ($\rho\approx\pm0.85$)
of opposite sign. This is by design: it makes the tail conditional
$\pz(x_2\mid x_1>\tau)$ \emph{bi-modal} and markedly non-Gaussian, so that hitting
the rare set is not equivalent to matching a single Gaussian tail and the proposal must reproduce a genuine multi-modal conditional.

\begin{table}[ht]
\centering
\caption{Parameters of the correlated 2D GMM. The right-hand modes ($i=5,6,7$)
carry the rare-event tail; modes $5$ and $6$ have correlation $\rho\approx+0.85$
and $-0.85$ respectively.}
\label{tab:corr_gmm}
\small
\begin{tabular}{c c c c}
\toprule
$i$ & $w_i$ & $\mu_i$ & $\Sigma_i$ \\
\midrule
1 & $0.10$ & $(-6,-2)$ & $\big[\begin{smallmatrix}0.36 & 0.30\\ 0.30 & 0.81\end{smallmatrix}\big]$ \\[2pt]
2 & $0.18$ & $(-4,\ 4)$ & $\big[\begin{smallmatrix}0.81 & -0.25\\ -0.25 & 0.49\end{smallmatrix}\big]$ \\[2pt]
3 & $0.14$ & $(-1,-5)$ & $\big[\begin{smallmatrix}0.49 & 0.20\\ 0.20 & 0.81\end{smallmatrix}\big]$ \\[2pt]
4 & $0.22$ & $(-10,\ 3)$ & $\big[\begin{smallmatrix}0.04 & 0.04\\ 0.04 & 0.09\end{smallmatrix}\big]$ \\[2pt]
5 & $0.10$ & $(\ 4,\ 3)$ & $\big[\begin{smallmatrix}2.40 & 1.85\\ 1.85 & 2.00\end{smallmatrix}\big]$ \\[2pt]
6 & $0.16$ & $(\ 4,-3)$ & $\big[\begin{smallmatrix}2.40 & -1.85\\ -1.85 & 2.00\end{smallmatrix}\big]$ \\[2pt]
7 & $0.10$ & $(\ 5,\ 0)$ & $\big[\begin{smallmatrix}0.30 & 0.00\\ 0.00 & 0.30\end{smallmatrix}\big]$ \\
\bottomrule
\end{tabular}
\end{table}

The mixture weights $w_i$ are given in Table~\ref{tab:corr_gmm} and sum to one. The observable is the linear read-out
$\phi(x)=a^\top x$ with $a=(1,0)^\top$, i.e.\ the event of interest is the threshold exceedance $\{x: x_1\ge\tau\}$.

\paragraph{Normalisation constant $\Zsoft$.}
The soft tilt $\pisoft(x)\propto \pz(x)\,e^{\rsoft(x)}$ with the Gaussian-CDF reward
$e^{\rsoft(x)}=\Phi\!\big((a^\top x-\tau)/\delta\big)$ has normaliser
$\Zsoft=\E_{\pz}[e^{\rsoft(X)}]$. Under each Gaussian component the read-out
$y=a^\top x\sim\N(m_i,\,v_i)$ with $m_i=a^\top\mu_i$ and
$v_i=a^\top\Sigma_i a$, and a standard Gaussian--Gaussian-CDF convolution (see Section~3.9 of \cite{10.7551/mitpress/3206.001.0001}) gives the closed form,
\begin{equation}
  \Zsoft
  \;=\;\sum_{i=1}^{M}w_i\;\E_{\N(m_i,v_i)}\!\Big[\Phi\big(\tfrac{y-\tau}{\delta}\big)\Big]
  \;=\;\sum_{i=1}^{M}w_i\,\Phi\!\left(\frac{m_i-\tau}{\sqrt{\delta^2+v_i}}\right).
  \label{eq:corr_gmm_Z}
\end{equation}
As $\delta\downarrow 0$ the reward sharpens to the hard indicator and
$\Zsoft\to \pE$.

\paragraph{Exact exceedance probability.}
Because the read-out $a^\top X$ is itself a 1D GMM with component means $m_i$ and
variances $v_i$, the exceedance probability is available in closed form,
\begin{equation}
  \pE\;=\;\pev{a^\top X\ge\tau}
  \;=\;\sum_{i=1}^{M}w_i\Big[1-\Phi\!\big(\tfrac{\tau-m_i}{\sqrt{v_i}}\big)\Big]
  \;=\;\sum_{i=1}^{M}w_i\,\Phi\!\left(\frac{m_i-\tau}{\sqrt{v_i}}\right),
  \label{eq:corr_gmm_pE}
\end{equation}
which we use as ground truth.
\paragraph{Tilted read-out marginal.}
While the full tilted density $\pisoft(x)\propto \pz(x)\,\Phi\big((a^\top x-\tau)/\delta\big)$
is not itself a Gaussian mixture, the marginal of the read-out $y=a^\top x$ is available
in closed form,
\begin{equation}
  \pisoft^{\mathrm{read}}(y)\;=\;\frac{1}{\Zsoft}\sum_{i=1}^{M}w_i\,\N(y;m_i,v_i)\,
  \Phi\!\Big(\frac{y-\tau}{\delta}\Big),\qquad y=a^\top x,
  \label{eq:corr_gmm_qtilt}
\end{equation}
with $\Zsoft$ from Eq.~\eqref{eq:corr_gmm_Z} and $(m_i,v_i)$ as above.
\paragraph{Forward process.}
We use the VP-SDE schedule of Appendix~\ref{app:diffusion_details} with
$\beta_{\min}=0.1$, $\beta_{\max}=20$, so that
$B(t)=\int_0^t\beta(s)\,ds=t\beta_{\min}+\tfrac12 t^2(\beta_{\max}-\beta_{\min})$,
$\alpha_t=e^{-B(t)/2}$ and $\sbar_t^2=1-e^{-B(t)}$. Because the forward kernel is
Gaussian, the noised marginal $\pt$ is again a GMM whose $i$-th component has
mean $\alpha_t\mu_i$ and covariance \cite{10.7551/mitpress/3206.001.0001},
\begin{equation}
S_{t,i}\;=\;\alpha_t^2\,\Sigma_i + \sbar_t^2\,I .
\end{equation}

\paragraph{Exact score.}
The exact score of the noised marginal is the weighted sum of the
per-component scores. For any mixture, 
\begin{equation}
     \nabla_x\log \pt(x)
     \;=\; \frac{\sum w_i \nabla_x \N_i}{\sum w_j \N_j}
= \frac{\sum w_i \N_i \big[-S_{t,i}^{-1}(x-\alpha_t\mu_i)\big]}{\sum w_j \N_j}\,.
\end{equation}
We use this exact score wherever analytical score is reported.
\paragraph{Exact look-ahead guidance.}
The exact guidance is the gradient of the Doob reward of Eq.~\eqref{eq:doob_objects} \cite{dai1991stochastic, denker2024deft}, which for the soft reward reads $\rstar(x) = \log G_t(x)$ with
$G_t(x)= \E_{p(x_0 \mid x_t)} [\exp(\rsoft(x_0))] $, i.e.\ $\nabla_x \rstar=\nabla_x\log G_t(x)$.
Conjugacy makes this closed-form.  We can use the fact that $p(x_0 | x_t, i) \propto p(x_t | x_0)p(x_0|i)$, hence the posterior under each
component $i$ is a product of two Gaussians - another Gaussian, $\N(m_{t,i},C_{t,i})$ with,
\begin{equation}
C_{t,i}=\Big(\Sigma_i^{-1}+\tfrac{\alpha_t^2}{\sbar_t^2}I\Big)^{-1},
  \qquad
m_{t,i}=C_{t,i}\Big(\Sigma_i^{-1}\mu_i+\tfrac{\alpha_t}{\sbar_t^2}x\Big).
\end{equation}
To evaluate the expected reward, we consider a linear observable $\phi(x_0) = a^\top x_0$. Under the $i$-th component posterior, this projection is a 1D Gaussian, $a^\top X_0\sim\N(a^\top m_{t,i},\,s_{t,i}^2)$
with $s_{t,i}^2=a^\top C_{t,i}\,a$, so the per-mode look-ahead reward integrates to another Gaussian CDF with,
\begin{equation}
  \E_{p(x_0 | x_t)} [\Phi(.)] =\Phi(\eta_{t,i}),
  \qquad
  \eta_{t,i}=\frac{a^\top m_{t,i}-\tau}{\sqrt{\delta^2+s_{t,i}^2}}.
  \label{eq:exact_per_component_guidance}
\end{equation}

We can then write the global look-ahead potential $G_t(x)$ as the expected reward marginalised over all $M$. Using Eq.\eqref{eq:exact_per_component_guidance}, we can write $ G_t(x) = \sum_{i=1}^{M}w_{t,i}\Phi(\eta_{t,i})$, giving the exact guidance as follows.
\begin{align}
    \nabla_x \rstar(x) &= \nabla_x \log G_t(x) \nonumber \\
    &= \frac{1}{G_t(x)} \nabla_x \left( \sum_{i=1}^M w_{t,i}(x) \Phi(\eta_{t,i}) \right) \nonumber \\
    &= \frac{1}{G_t(x)} \sum_{i=1}^M \left[ \Phi(\eta_{t,i}) \nabla_x w_{t,i}(x) + w_{t,i}(x) \nabla_x \Phi(\eta_{t,i}) \right]
\end{align}
We can rewrite the above equation as, 

\begin{equation}
    \nabla_x \rstar(x) = \sum_{i=1}^M \widetilde{w}_i(x) \left[ \frac{\nabla_x \Phi(\eta_{t,i})}{\Phi(\eta_{t,i})} + \frac{\nabla_x w_{t,i}(x)}{w_{t,i}(x)} \right] \qquad \widetilde{w}_i(x) = \frac{w_{t,i}(x) \Phi(\eta_{t,i})}{\sum_{j=1}^M w_{t,j}(x) \Phi(\eta_{t,j})}
\label{eq:corr_gmm_hstar}
\end{equation}
Here $\widetilde{w}_i$ represents the reward weighted (posterior) responsibility of the GMM - i.e. it is the responsibility re-weighted by each mode's chance of reaching the rare event set \cite{bishop2006pattern}. A mode $i$ is up-weighted when $x_t$ is consistent with it (i.e. $w_{t,i}$ large) and when it is well placed to cross the threshold ($\Phi(\eta_{t, i})$ large). The first term inside the bracket is the per-mode push towards the rare set;
second is the routing correction that re-weights mass between modes as the reward favours different components.  Evaluating the two terms in $\nabla_x \rstar$, the first term is the reward push, 
\begin{equation}
    \frac{\nabla_x \Phi(\eta_{t,i})}{\Phi(\eta_{t,i})} = \frac{\varphi(\eta_{t,i})}{\Phi(\eta_{t,i})} g_{t,i} = \psi(\eta_{t,i}) g_{t,i},
\end{equation}
where $g_{t,i} = \nabla_x \eta_{t,i} = \frac{\alpha_t}{\sbar_t^2} \frac{C_{t,i} a}{\sqrt{\delta^2 + s_{t,i}^2}}$ and $\psi(\cdot)=\varphi(\cdot)/\Phi(\cdot)$ is the Inverse-Mills ratio. 

The second term of Eq.~\eqref{eq:corr_gmm_hstar} represents the standard GMM responsibility at a noised time $t$, given by, 
\[w_{t,i}(x) = \frac{w_i \N(x; \alpha_t \mu_i, S_{t,i})}{\pt(x)}.\] Evaluating this yields,
\begin{align}
    \nabla_x \log w_{t,i}(x) &= \nabla_x \log \N(x; \alpha_t \mu_i, S_{t,i}) - \nabla_x \log \pt(x).
\end{align}
Putting this together, we get the analytical guidance for the GMM,
\begin{equation}
    \nabla_x \rstar(x) = \sum_{i=1}^M \widetilde{w}_i(x) \Big[ \psi(\eta_{t,i}) g_{t,i} + \nabla_x \log \N(x; \alpha_t \mu_i, S_{t,i}) - \nabla_x \log \pt(x) \Big].
\end{equation}

We refer readers to \cite{3692070.3694254} for a rigorous exposition of the effect of guidance on GMMs.

\section{2D GMM experiments}
\label{app:more_2dgmm}

\subsection{Experimental details}
\label{app:model_details_2dgmm}
\label{app:tempering_ladder}

\paragraph{The learned score.}
The learned score for the correlated 7-component GMM of Table~\ref{tab:corr_gmm}, is a velocity prediction MLP trained under the linear VP schedule \cite{song2021scorebasedgenerativemodelingstochastic}. The network takes
$(x,\mathrm{emb}(t))\in\R^{2+128}$, where $t$ is embedded using a sinusoidal time  embedding through 6 layers (hidden width 512, SiLU \cite{elfwing2017sigmoidweightedlinearunitsneural};
$1.1$M parameters) to a two-dimensional output. Targets are
$v = \alpha_t\,\varepsilon - \sbar_t\,x_0$ with
$x_t=\alpha_t x_0+\sbar_t\varepsilon$; the score used by the sampler is
recovered as $\sth(x,t) = -\hat\varepsilon/\sbar_t$ with
$\hat\varepsilon=\sbar_t x + \alpha_t v_\theta(x,t)$. The loss is the
per-sample $v$-MSE under min-SNR-$\gamma$ weighting with $\gamma = 5$ and the weights given by $w(t)=\min(\mathrm{SNR}(t),\gamma)/(\mathrm{SNR}(t){+}1)$ of \cite{hang2023minsnr},
with $t\sim\mathcal{U}[0.01,0.99]$. Training draws fresh exact GMM samples
every step: 40{,}000 steps at batch 4096 with AdamW \cite{kingma2014adam, loshchilov2017decoupled}
(learning rate $2\times10^{-4}$, weight decay $10^{-4}$), 1{,}000-step linear warm-up followed by cosine decay, gradient clipping at $1.0$, and an
exponential moving average of the weights (decay $0.999$); the EMA weights are used for evaluation.

\textbf{The tempering ladder.} Our guidance ladder used through the experiment is
$\lambda(t) = 1 - (t/t_{\rm on})^2$ for $t < t_{\rm on} = 0.9$ in the case of rejuvenation and $\lambda_t \equiv 1$ without it, see App.~\ref{app:gmm_schedule_ablations}. We chose this simple form of tempering, inspired by polynomial tempered likelihood \cite{friel2008marginal}.

\textbf{The benchmark.}
The mixture parameters and every closed form we compare against: the base density, the exact score, the soft-tilt normaliser $\Zsoft$ of Eq.~\eqref{eq:corr_gmm_Z}, the exceedance probability $\pE$ of Eq.~\eqref{eq:corr_gmm_pE}, the tilted target $\pisoft$ of Eq.~\eqref{eq:corr_gmm_qtilt} and the exact Doob guidance $\nabla_x \rstar$ of Eq.~\eqref{eq:corr_gmm_hstar} are defined in Appendix~\ref{app:corr_gmm}. The right-hand modes carry strong off-diagonal covariance of opposite sign, so the tail conditional $\pz(x_2\mid x_1{>}\tau)$ is genuinely bi-modal: hitting the rare set requires reproducing a multi-modal conditional. The learned-score arm uses the $v$-prediction MLP described above.

\paragraph{Sampling parameters} Unless stated otherwise every run below uses the deployed configuration of Section~\ref{sec:exp_2dgmm}: $N=1000$ particles, $K=1000$ SEEDS-1 steps on a uniform half-log SNR noise schedule with $\delta=0.01$, the quadratic tempering ladder $\lambda(t)$ above for the rejuvenation arm and $\lambda\equiv1$ otherwise, adaptive resampling at ESS$/N<0.5$ with three MADM rejuvenation sweeps per resampling (Simpson's (1/3) rule with Eq.~\eqref{eq:madm_simpson}), and $24$ seeds. 
The MH step is $\rho = c(t)\,\sbar_t^2$ with $c(t)$ fixed once from a pilot sweep targeting an acceptance ratio of $\approx0.72$ at each noise level, as prescribed in \cite{lam2026metropolis}.

All experiments were done on a single Nvidia A100 GPU with 40 GB of VRAM.

\subsection{Sources of Bias}
We begin by discussing the two main contributing factors that bias our estimate against the true $\pE$, followed by additional results. Neither bias is an error of the estimator, and the path-weighted estimator remains unbiased exactly for the law we sample from. The two primary sources are the learning process where the score is approximated poorly (Appendix~\ref{app:tail_bias}) and when the discretised law we sample from differs from the continuous law (Appendix~\ref{app:em_bias}). We separate and isolate each to study their contributions to the final estimate.
\subsubsection{Tail bias of the learned score}
\label{app:tail_bias}

We estimate the tail bias of the learned score by comparing the model's output at a specific time and location $(x_t, t)$ against the analytical score. Write $\Delta s(x,t) = \sth(x,t) - s^\star(x,t)$ for the pointwise gap
between the learned and exact scores. The induced log-density bias
$b(x_0) := \log p_{\rm model}(x_0) - \log p^\star(x_0)$ accumulates along the
probability-flow ODE trajectory $\{x_t\}_{t\in[0,1]}$ that carries $x_0$ back to
the prior under the \emph{learned} velocity field, as
\begin{equation}
\label{eq:bias_path}
  b(x_0) \;=\; -\!\int_0^1 \tfrac12\beta(t)\,
  \Big[\nabla\!\cdot\!\Delta s \;+\; s^\star\!\cdot\!\Delta s\Big](x_t)\;\dd t ,
\end{equation}
from which the tail ratio follows by reweighting the exact conditional,
\begin{equation}
\label{eq:bias_tail}
  \frac{P_{\rm model}(\phi(X)\ge\tau)}{P^\star(\phi(X)\ge\tau)}
  \;=\; \E_{p^\star(\,\cdot\,\mid\,\phi(X)\ge\tau)}\!\left[e^{\,b(x_0)}\right].
\end{equation}

\begin{table}[!ht]
\centering
\small
\caption{Tail bias of the learned 2D-GMM score. $p_{\rm model}$ (MC) is obtained from a large MC reference pool using unguided draws. Path integral measures the ratio obtained by integrating the pointwise score error along the probability-flow ODE (Eq.~\eqref{eq:bias_path}) and reweighing the exact conditional. The two induce different laws but agree to within $7\%$ across three decades of rarity. $\pdata[E]$ is the analytical probability.}
\label{tab:model_bias}
\begin{tabular}{rcccccc}
\toprule
$\tau$ & $\pdata[E]$ & $p_{\rm model}$ (MC) & $p_{\rm model}/\pdata[E]$ & rel.\ s.e. & path integral & ratio \\
\midrule
7  & $6.88\!\times\!10^{-3}$ & $7.13\!\times\!10^{-3}$ & 1.037 & 0.08\% & 1.075 & 1.037 \\
8  & $1.28\!\times\!10^{-3}$ & $1.38\!\times\!10^{-3}$ & 1.081 & 0.18\% & 1.156 & 1.069 \\
9  & $1.62\!\times\!10^{-4}$ & $1.87\!\times\!10^{-4}$ & 1.153 & 0.50\% & 1.236 & 1.072 \\
10 & $1.40\!\times\!10^{-5}$ & $1.72\!\times\!10^{-5}$ & 1.229 & 1.64\% & 1.267 & 1.031 \\
\bottomrule
\end{tabular}
\end{table}

In practice, we draw $2.16\times10^{8}$ samples from the model unconditionally and use it as our reference pool. We find that with both methods, the tail bias widens with rarity, and the score-estimate degrades monotonically, see Table~\ref{tab:model_bias}. All quoted comparisons for the trained model are against the model's own rare event probability $\pE$, obtained via Monte Carlo estimate.

\subsubsection{The discretised law: grid and integrator}
\label{app:em_bias}
Alongside the bias due to the score estimate, which in the general case cannot be estimated, our discretisation itself induces a distinct, reducible form of bias. Let $\Pm$ be the continuous path measure defined by the diffusion process, and ${\Pm}_k$ bethe discretised law that we simulate. Our estimator is unbiased for the discrete law that we compute weights for under any guidance \cite{del2006sequential}. Hence any gap that remains is a property of the base chain and the integrator alone. This section first measures this gap for two first order integrators and whether the choice of integrator reduces it at a fixed $K$.

\paragraph{Setup.} Exact Doob guidance (i.e. $\nabla_{x} r_t^*(x)$), with resampling turned off, so neither guidance error nor adaptive resampling contributes and the residual is resolved above seed noise. We compare the Euler-Maruyama (EM) integrator against the first order exponential SDE solver of \cite{gonzalez2023seeds} which we denote as SEEDS-1.  Both are explicit one-step schemes of first order. The only difference is that SEEDS-1 treats the integration of the reverse SDE as a semi-linear integral, separating the linear and non-linear parts of the score, and integrating the linear and the noise term exactly, whereas EM approximates it to first order, see \cite{gonzalez2023seeds, milstein2004stochastic, pfarr2026stochsticsamplinggenerativediffusion} for more details. Alongside the integrator, the placement of the grid itself is a design choice: we run the sweep on a grid uniform in $t$ and on a grid uniform in the half-log Signal-to-Noise ratio $\ell_t = \log(\alpha_t/\sbar_t)$ of App.~\ref{app:gamma_derieved}, the integrator's own time variable. Table~\ref{tab:em_vs_exp_bias} sweeps the step count $K$ under both placements and compares the signed bias in the probability estimate with the ground truth in $\tau = 10$.

\begin{table}[ht]
\centering
\caption{Signed discretisation bias versus step count $K$ for the two integrators, under
exact Doob guidance, at different grids, uniform in $t$ (top block) and uniform in the half-log SNR $\ell_t = \log(\alpha_t/\sbar_t)$ (bottom block). Positive values are overestimates. Paired columns are EM/SEEDS-1 throughout. $\hat{k}$ is the Pareto Smoothed Importance Sampling shape estimate \cite{vehtari2024pareto} on the weights, and $\ESS_E$ is the ESS of the final particles, measured after thresholding to $\tau \ge 10$.}
\label{tab:em_vs_exp_bias}
\small
\setlength{\tabcolsep}{3pt}%
\begin{tabular}{r r r c c c c}
\toprule
$K$ & EM bias & SEEDS-1 bias & $|\mathrm{EM}|/|\mathrm{SEEDS}|$ & CV $\downarrow$ & $\hat{k} \downarrow$ & $\ESS_E \uparrow$ \\
\midrule
\multicolumn{7}{l}{Uniform in $t$} \\
\midrule
10 & +64.02\% $\pm$ 11.31\% & -15.98\% $\pm$ 1.88\% & 4.01 & 0.338/0.109 & 0.85/0.46 & 0.032/0.102 \\
15 & +48.61\% $\pm$ 8.39\% & -16.23\% $\pm$ 1.28\% & 3.00 & 0.277/0.075 & 0.70/0.45 & 0.063/0.189 \\
20 & +54.90\% $\pm$ 16.50\% & -15.13\% $\pm$ 0.97\% & 3.63 & 0.522/0.056 & 0.68/0.43 & 0.099/0.248 \\
30 & +23.41\% $\pm$ 1.65\% & -13.62\% $\pm$ 0.75\% & 1.72 & 0.065/0.043 & 0.56/0.40 & 0.176/0.375 \\
50 & +16.75\% $\pm$ 1.38\% & -8.36\% $\pm$ 0.63\% & 2.00 & 0.058/0.034 & 0.55/0.45 & 0.315/0.465 \\
100 & +6.82\% $\pm$ 0.69\% & -4.99\% $\pm$ 0.57\% & 1.37 & 0.032/0.029 & 0.40/0.45 & 0.600/0.634 \\
200 & +3.72\% $\pm$ 0.30\% & -2.71\% $\pm$ 0.26\% & 1.38 & 0.014/0.013 & 0.40/0.40 & 0.771/0.802 \\
500 & +1.30\% $\pm$ 0.19\% & -1.27\% $\pm$ 0.16\% & 1.02 & 0.009/0.008 & 0.39/0.40 & 0.901/0.906 \\
1000 & +0.75\% $\pm$ 0.19\% & -0.51\% $\pm$ 0.17\% & 1.46 & 0.009/0.009 & 0.36/0.39 & 0.943/0.945 \\
\midrule
\multicolumn{7}{l}{Uniform in $\ell_t=\log(\alpha_t/\sbar_t)$}\\%
\midrule
10 & -34.55\% $\pm$ 2.77\% & +256.82\% $\pm$ 39.09\% & 0.13 & 0.207/0.537 & 0.75/0.80 & 0.031/0.021 \\
15 & -24.93\% $\pm$ 3.12\% & +28.68\% $\pm$ 2.42\% & 0.87 & 0.203/0.092 & 0.58/0.48 & 0.083/0.115 \\
20 & -20.68\% $\pm$ 1.76\% & +12.66\% $\pm$ 1.63\% & 1.63 & 0.109/0.071 & 0.46/0.40 & 0.169/0.201 \\
30 & -13.78\% $\pm$ 0.88\% & +2.13\% $\pm$ 1.01\% & 6.46 & 0.050/0.048 & 0.38/0.37 & 0.283/0.313 \\
50 & -8.98\% $\pm$ 0.59\% & -2.19\% $\pm$ 0.70\% & 4.09 & 0.032/0.035 & 0.34/0.34 & 0.474/0.506 \\
100 & -4.63\% $\pm$ 0.42\% & -2.59\% $\pm$ 0.42\% & 1.79 & 0.022/0.021 & 0.24/0.26 & 0.713/0.724 \\
200 & -2.54\% $\pm$ 0.36\% & -1.89\% $\pm$ 0.34\% & 1.35 & 0.018/0.017 & 0.17/0.18 & 0.841/0.849 \\
500 & -1.29\% $\pm$ 0.17\% & -1.02\% $\pm$ 0.17\% & 1.27 & 0.008/0.009 & 0.19/0.19 & 0.928/0.930 \\
1000 & -0.27\% $\pm$ 0.13\% & -0.16\% $\pm$ 0.13\% & 1.66 & 0.007/0.006 & 0.22/0.23 & 0.956/0.957 \\
\bottomrule
\end{tabular}
\end{table}

\paragraph{The $1/K$ law.} 
\label{app:taleytubero_expansion} 
Both integrators are one-step schemes of weak order one \cite{milstein2004stochastic}, for which the weak-error admits the expansion $c_1/K + O(1/K^2)$ with $c_1$ independent of the step size \cite{talay1990expansion, milstein2004stochastic}. While this holds for smooth test functions, it does not automatically hold for the indicator $\inds{E}$. We therefore fit the form using ordinary least-squares instead of invoking it directly as a theorem. Fitting the coefficient $c_1$, weighted by the per-K standard errors $s_k$,
\begin{equation}
\label{eq:c1_fit}
  \hat c_1 = \frac{\sum_k b_k K_k^{-1} s_k^{-2}}{\sum_k K_k^{-2} s_k^{-2}},
\end{equation}
restricted to $K\ge30$, since $K=10$-$20$ sits outside the asymptotic regime. On the uniform-$t$ grid the fits are $\hat c_1 = +7.27$ for Euler-Maruyama against $-4.41$ for SEEDS-1, with largest residuals of $1.6$ and $2.4$ standard errors respectively, so the $1/K$ law holds to within roughly two standard errors over the fitted range. The signs record that the two integrators approach the continuous limit from opposite sides on this grid: EM overestimates and SEEDS-1 underestimates at every $K$, and the $|\mathrm{EM}|/|\mathrm{SEEDS}|$ column of Table~\ref{tab:em_vs_exp_bias} is the per-$K$ magnitude comparison. The fitted coefficients track the measurements closely: $\hat c_1 = 7.27$ predicts $+3.6\%$ at $K=200$ against a measured $+3.72\%$, and $+14.6\%$ at $K=50$ against $+16.8\%$. On the uniform-$\ell$ grid the EM arm fits similarly ($\hat c_1 = -4.45$, largest residual $2.4$ standard errors), but the SEEDS-1 arm \emph{changes sign} between $K=30$ and $K=50$: its leading coefficient is small enough that higher-order terms are not negligible over the measured range. 

\paragraph{Grid placement.} At coarse grids the two placements order the integrators differently. In uniform in $t$, SEEDS-1 dominates EM in every $K$: the magnitude ratio drops from $4.0$ in $K=10$ towards parity by $K\ge500$. For the uniform in the $\ell$ grid, the order inverts in the coarsest grids $K = 10$ and $K = 15$, where we reproduce the result of \cite{gonzalez2023seeds} where the performance of the integrator is poor at very small discretisation. We hypothesise that this is because, a uniform $\Delta \ell$ starves the ``middle'' of the denoising process, which principled methods for noise-scheduling such as \cite{williams2024score} place a larger number of discretisation on, see Fig. 4 of \cite{williams2024score}, which adversely affects the discretised law that we sample. Given that the schedule intimately interplays with Gaussian transition kernel approximation, we leave studying the effects it has on the weight computation for future work.  

From $K\ge20$ SEEDS-1 yields a smaller discretisation bias, and the two integrators converge to each other by $K = 200$. We find that in our setting, SEEDS-1 reduces the discretisation bias at lower $K$ on either grid placement away from the coarsest $\ell$-grids, and both integrators are comparable at higher $K$'s. We deploy SEEDS-1 at $K = 1000$ on the uniform-$\ell$ grid, where both integrators are within $0.2\%$ of the continuous limit (Table~\ref{tab:em_vs_exp_bias}, last row). The exponential integrator is also used in the larger climate experiments.

\paragraph{Both channels improve together.} In this experiment, the integrator does not only control the bias but also affects the variance. Across Table~\ref{tab:em_vs_exp_bias} the variance diagnostics move with it: on the uniform-$t$ EM arm the CV falls from $0.338$ to $0.009$, $\ESS_E$ climbs from $0.032$ to $0.943$, with the PSIS $\hat k$ dropping below $0.5$ by $K=100$. The grid placement moves the same diagnostics: at matched $K$ the $\ell$-grid improves $\ESS_E$ and $\hat k$ for both integrators (e.g.\ at $K=100$, EM's $\ESS_E$ is $0.713$ uniform in $\ell$ against $0.600$ uniform in $t$, with $\hat k$ of $0.24$ against $0.40$), the coarsest $\ell$-grids under SEEDS-1 excepted. These are not two separate effects but one mechanism with two consequences. A finer grid puts the discrete chain closer to the continuous dynamics, which simultaneously shrinks the gap $\Pm_K - \Pm$ that the bias measures. As we simulate using the exact Doob's guidance, this moves the discretised proposal $\Qm^K$ close to the variance minimising Doob's proposal which helps all other metrics. In a general setting with approximate guidance the dominant source of error is not the discretisation bias but instead the approximation error of Doob's h-transform, e.g. using DPS.

\begin{figure}[!ht]
    \centering
    \includegraphics[width=1.0\linewidth]{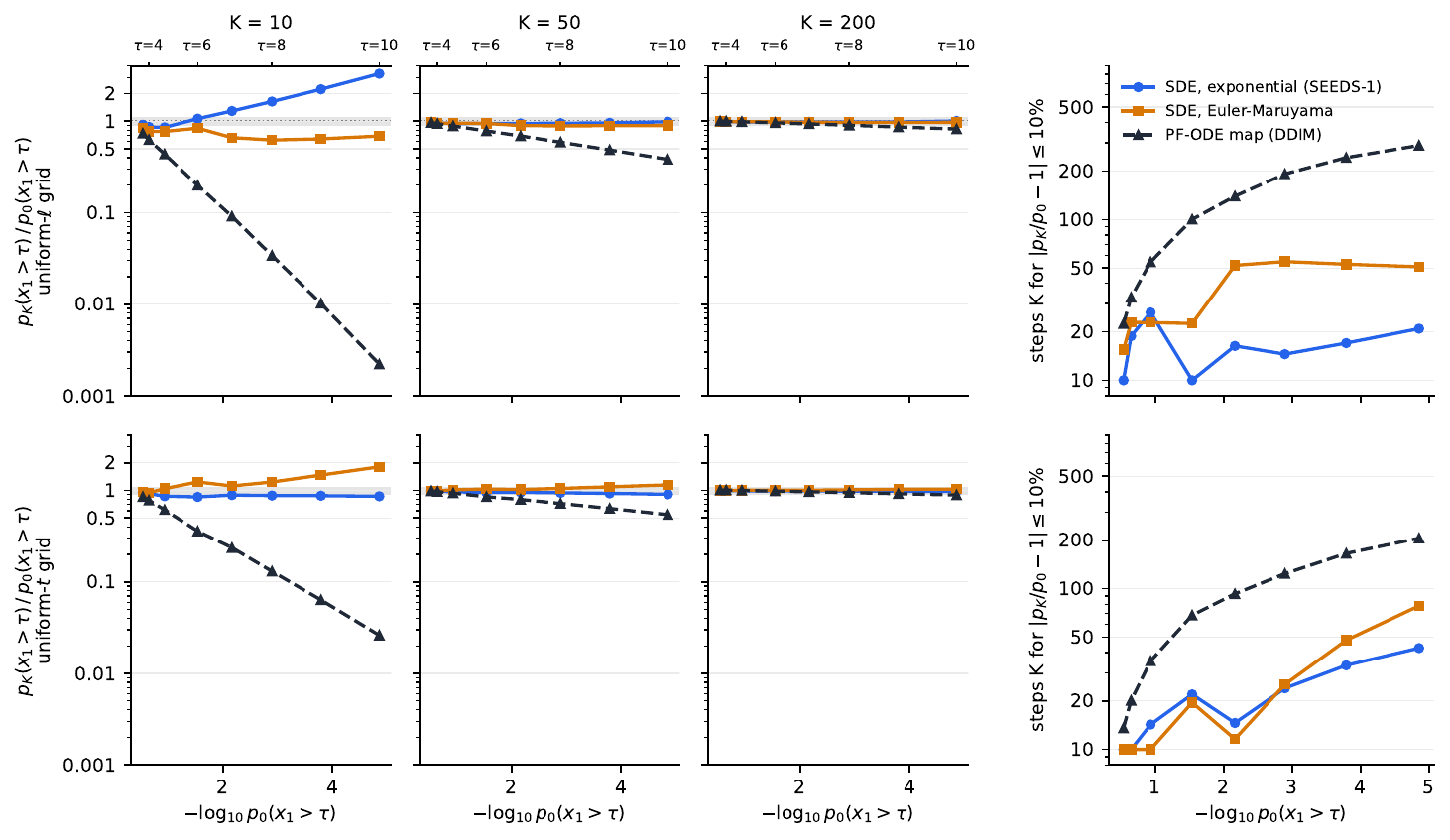}
   \caption{Discretisation bias for the three integrators, Euler-Maruyama, SEEDS-1 and DDIM. We unconditionally draw 1e8 samples from the base model at different discretisation $K$ and evaluate the bias at different rarities. The figure on the right shows the number of discretisation required to get a $<10\%$ bias for different rarities, showing that the SDE requires less steps than the ODE integrator. Top row uses the uniform log-SNR $\Delta \ell$ schedule and the bottom is uniform in time $\Delta t$. }
\label{fig:weak_order_integration_analytical}
\end{figure}

\begin{figure}[!ht]
    \centering
    \includegraphics[width=1.0\linewidth]{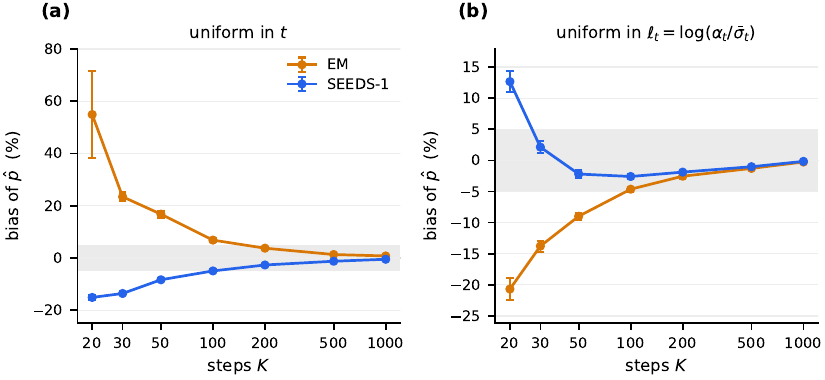}
   \caption{Measured bias versus $K$ with 1000 particles, analytical scores and Doob's transition kernels. See App~\ref{app:em_bias}. We show the bias of our estimate falls with a $O(1/K)$ leading order term. We compare both the EM integrator and the SEEDs-1 integrator in the noise prediction form. We compare two different noise schedules: uniform in time and uniform in log-SNR $\ell$. Error-bars represent standard error of the mean. Experiment is conducted the $\tau = 10$ threshold exceed problem. See Table~\ref{tab:em_vs_exp_bias}.}
\label{fig:doobs_bias.}
\end{figure}

\begin{figure}[!ht]
    \centering
    \includegraphics[width=0.75\linewidth]{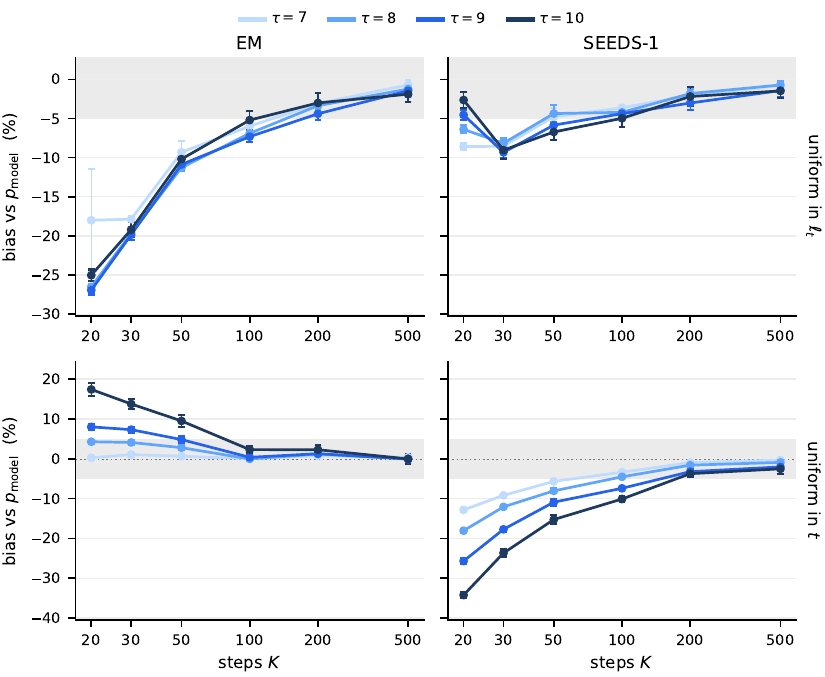}
   \caption{Measured bias versus $K$ with 10,000 particles and the deployed guidance and no rejuvenation for the learned score. We compare the discretisation bias at different rarities, we compare against a $10^8$ unconditional MC draw of the exponential integrator at $K$ as our reference. Hence, this measures how quickly the estimates converge to the reference with $K$.}
\label{fig:bias_learned_scores_rarity}
\end{figure}

\subsection{The threshold-exceedance experiment}
\label{app:threshold_extra}
Everything in this subsection belongs to the deployed threshold-exceedance
experiment of Section~\ref{sec:exp_2dgmm}: the conditional clouds behind the
sample-quality columns, the no-resampling arm, the whole-path ESS diagnostics,
and how the estimator scales with the particle budget $N$. We summarise our results in Table.~\ref{tab:step1_recovery}, where we find that our method recovers the correct probability estimates as well as the correct shape statistics. We analyse all individual aspects of these results in the section below. All experiments were done with the same guidance with $K = 1000$ and the SEEDS-1 integrator on the uniform-$\ell$ grid unless stated otherwise. 
\begin{table}[!ht]
\centering
\small
\caption{Guided generation on the GMM with learned and analytical scores. Each arm is scored against the law it samples from, $p^\star$; either $\pdata[E]$ (analytical) or an MC estimate of $p_0[E]$ (learned). Sample quality metrics are evaluated against a reference from the same law, presented as a ratio against a noise-floor. Estimates $\hat p/p^\star$ are biased to the MADM kernel's floor (Fig.~\ref{fig:madm_error_bias}).
We evaluate the two-sample Kolmogorov-Smirnov (KS) test statistic, the Earth Mover's Distance (EMD) and the coefficient of variation $\text{CV}^2 = \Var(\hat{p_E})/p_E^2$.
}
\label{tab:step1_recovery}
\begin{tabular}{llcccccc}
\toprule
$\tau$ & score & $p^\star$ & $\hat p/p^\star$ & CV $\downarrow$ & KS$_{x_1}$ & KS$_{x_2}$ & EMD \\
\midrule
7 & analytical & $6.88\!\times\!10^{-3}$ & $0.980 \pm 0.012$ & 0.060 & 0.88 & 1.33 & 1.52 \\
8 & analytical & $1.28\!\times\!10^{-3}$ & $0.966 \pm 0.012$ & 0.061 & 1.13 & 1.10 & 1.06 \\
9 & analytical & $1.62\!\times\!10^{-4}$ & $0.974 \pm 0.016$ & 0.083 & 1.12 & 1.00 & 0.97 \\
10 & analytical & $1.40\!\times\!10^{-5}$ & $0.954 \pm 0.013$ & 0.069 & 0.94 & 1.05 & 1.46 \\
\addlinespace
7 & learned & $7.13\!\times\!10^{-3}$ & $0.978 \pm 0.012$ & 0.060 & 1.11 & 1.20 & 1.14 \\
8 & learned & $1.38\!\times\!10^{-3}$ & $0.985 \pm 0.013$ & 0.064 & 0.97 & 1.00 & 0.73 \\
9 & learned & $1.87\!\times\!10^{-4}$ & $0.980 \pm 0.018$ & 0.089 & 0.77 & 0.83 & 0.98 \\
10 & learned & $1.72\!\times\!10^{-5}$ & $0.965 \pm 0.016$ & 0.079 & 1.23 & 1.07 & 0.98 \\
\bottomrule
\end{tabular}
\vspace{-10pt}
\end{table}

\subsubsection{Conditional-sample coverage: clouds for learned and analytical}
\label{app:clouds}
We plot the sample clouds ($N= 2048$)  using the deployed configuration of Section~\ref{sec:exp_2dgmm} for both Analytical and Learned Guidance with the same guidance/resampling/rejuvenation parameters. The red samples are drawn from the exact conditional $p(x\mid x_1\ge\tau)$, drawn larger at ($2\times10^{4}$) and the blue particles represent the guided samples. Together with the KS/EMD ratios of Table~\ref{tab:step1_recovery} the plots qualitatively highlight that samples are drawn from the conditional distribution itself. The two clouds remain visually indistinguishable as can be seen in Figure~\ref{fig:step1_clouds}.

\begin{figure}[!htb]
\centering
{\small (a) Analytical score}\\[2pt]
\includegraphics[width=0.73\linewidth]{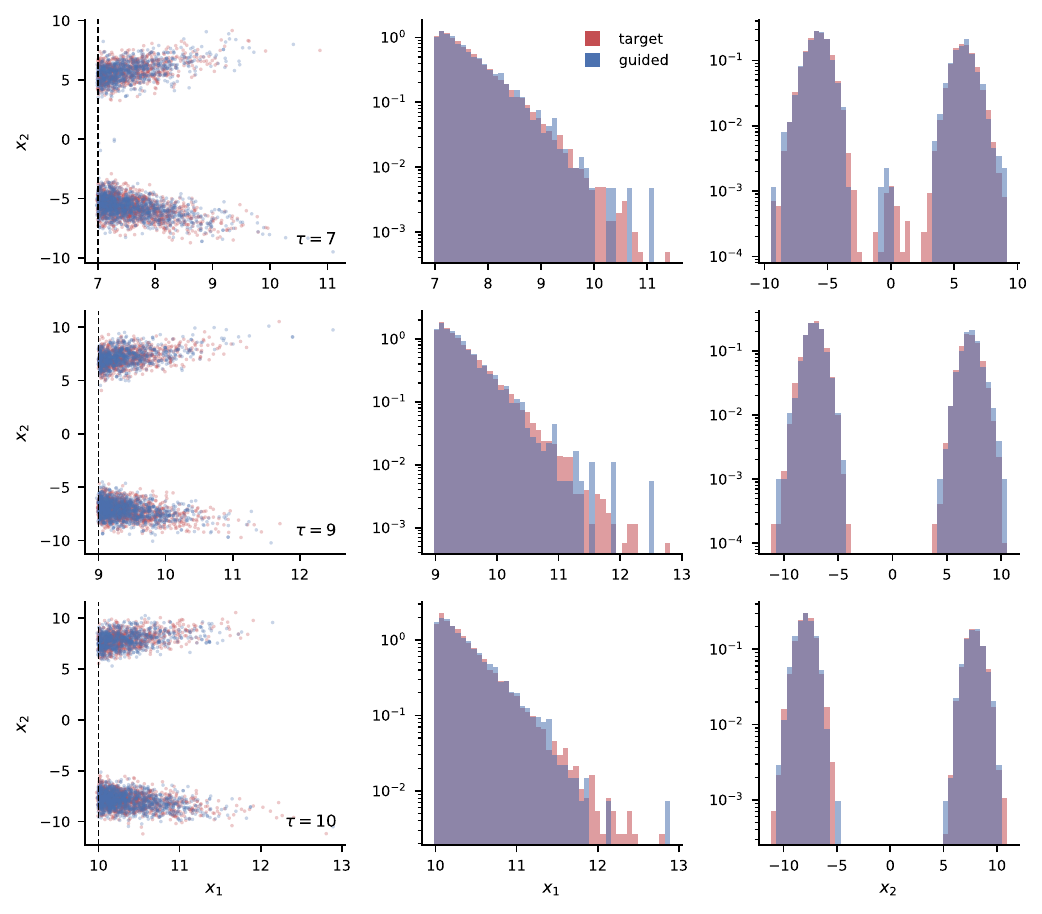}\\[8pt]
{\small (b) Learned score}\\[2pt]
\includegraphics[width=0.73\linewidth]{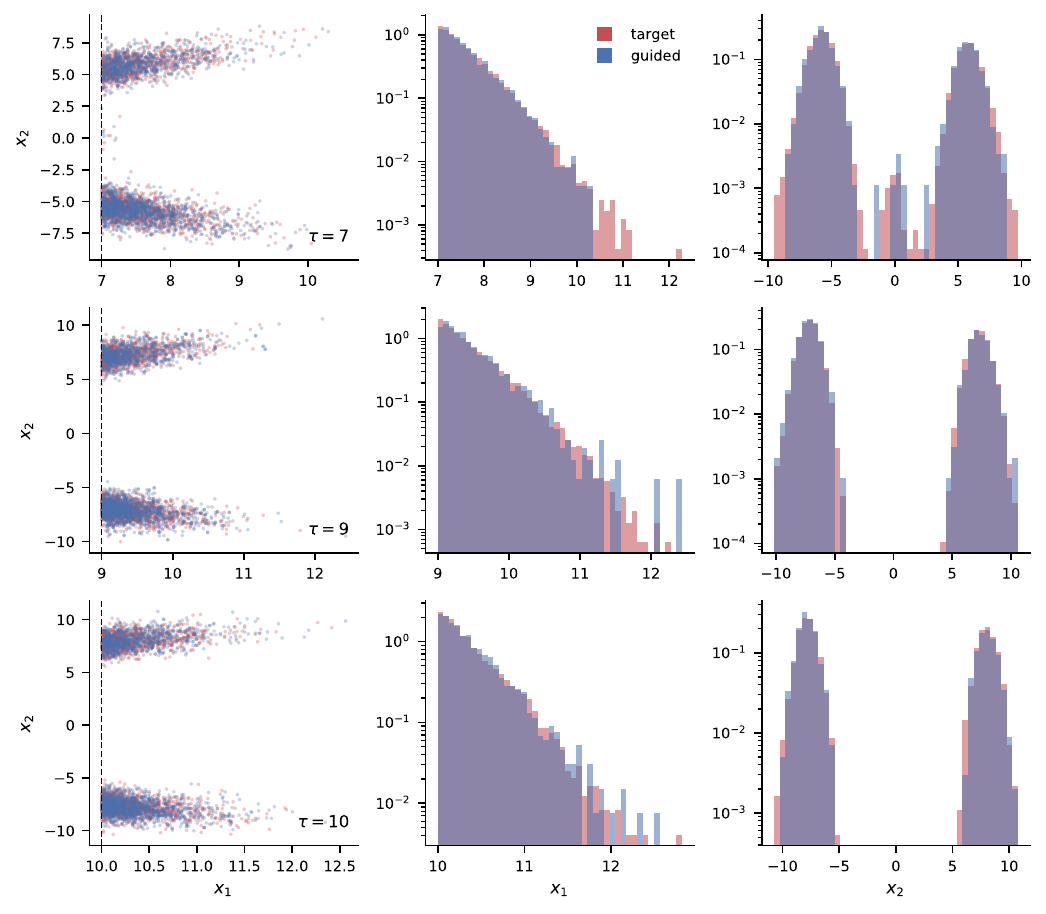}
\caption{Conditional-sample coverage on the 2D GMM for the analytical (a) and learned (b) scores. Blue is the path-weighted guided cloud ($2048$ samples; pooled over $8$ seeds for the learned score); red is a MC draw from the conditional $p(x\mid x_1\ge\tau)$. \emph{Left:} the
joint cloud, with the threshold dashed. \emph{Centre, right:} the $x_1$ and
$x_2$ marginals on a log density scale.
Corresponding quantitative KS/EMD metrics in Table~\ref{tab:step1_recovery}.}
\label{fig:step1_clouds}
\end{figure}
\subsubsection{ESS along the trajectory}
\label{app:ess_trajectory}

Without resampling the accumulated weight is never reset and the terminal $\ESS$ asymptotically converges to $1/(1+\chi^2(\Pms\|\Qmk))$ \cite{agapiou2017importance}. Hence for the sampled path, the $\ESS$ measures the gap between the implemented guidance and the exact Doob's guidance.
Figure~\ref{fig:ess_time} tracks it along the trajectory for both scores and five thresholds. Two features that are worth noting. The trace is flat at $\ESS/N = 1$ until
$t \approx 0.9$, which is the untilted stretch of the denoising time, where the guidance is turned off. After guidance is turned on, the ESS decays monotonically in both learned and analytical arms.
The two panels remain close to indistinguishable, highlighting that the profile is set by the guidance geometry rather than the score itself. 
\begin{figure}[!ht]
\centering
\includegraphics[width=\linewidth]{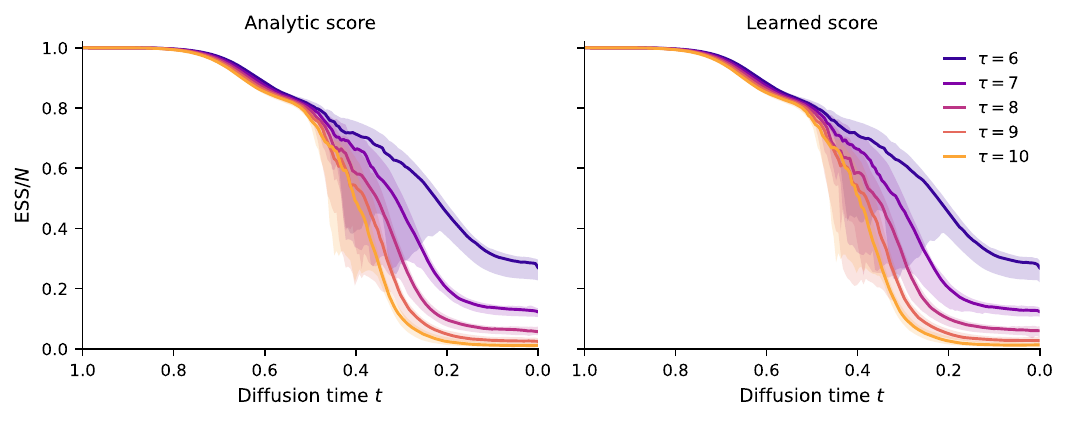}
\caption{Effective sample size of the path weights along diffusion time, without
resampling (analytic score, left; learned score, right), for
$\tau \in \{6,7,8,9,10\}$. Settings as Table~\ref{tab:step1_recovery}. Time runs right to left, from noise
($t=1$) to data ($t=0$).}
\label{fig:ess_time}
\end{figure}

\subsubsection{Selection and rejuvenation}
\label{app:noresample}
\label{app:madm_curl}

\begin{table}[!ht]
\centering\small\setlength{\tabcolsep}{4pt}
\caption{Selection and rejuvenation on the 2D GMM under the \emph{analytical}
score, $24$ seeds, all other settings as in Table~\ref{tab:step1_recovery}. Resampling and rejuvenation are ablated at different threshold exceedance rarities. Weighted shape-metrics are computed against a conditional reference. Rejuvenation helps reduce the CV across all rarities; without it resampling alone can adversely prune lineages.}
\label{tab:2x2}
\begin{tabular}{llcccc}
\toprule
$\tau$ & arm & CV $\downarrow$ & KS$_{x_1}$ & KS$_{x_2}$ & EMD \\
\midrule
\multirow{3}{*}{7} & no resampling & 0.067 & 0.88 & 1.52 & 1.50 \\
 & resampling & 0.062 & 1.05 & 0.99 & 0.83 \\
 & resampling $+$ rejuv.\ ($M=3$) & 0.060 & 0.88 & 1.33 & 1.52 \\
\addlinespace
\multirow{3}{*}{8} & no resampling & 0.083 & 1.14 & 1.35 & 1.17 \\
 & resampling & 0.098 & 1.12 & 0.95 & 0.77 \\
 & resampling $+$ rejuv.\ ($M=3$) & 0.061 & 1.13 & 1.10 & 1.06 \\
\addlinespace
\multirow{3}{*}{9} & no resampling & 0.113 & 0.91 & 1.11 & 1.03 \\
 & resampling & 0.150 & 1.15 & 1.47 & 1.14 \\
 & resampling $+$ rejuv.\ ($M=3$) & 0.083 & 1.12 & 1.00 & 0.97 \\
\addlinespace
\multirow{3}{*}{10} & no resampling & 0.147 & 0.81 & 1.46 & 2.45 \\
 & resampling & 0.230 & 0.91 & 1.46 & 1.98 \\
 & resampling $+$ rejuv.\ ($M=3$) & 0.069 & 0.94 & 1.05 & 1.46 \\
\bottomrule
\end{tabular}
\end{table}

We ablate whether resampling or rejuvenation is helpful for our problem. For
simplicity we provide two tables, separating the effect of each on the analytical and learned scores. Table~\ref{tab:2x2} ablates the no-resampling, resampling and rejuvenation for the analytical score. Table~\ref{tab:cap_thr} then runs the full threshold ladder on the
learned score, adding the arm with no resampling at all and three noise levels where the rejuvenation move is switched off. All experiments are done using the same configuration as mentioned above. Each arm is scored against the law it samples from, $p^\star$. Weighted sample quality statistics are evaluated against a conditional reference drawn from the same law, presented as ratios to their split-half noise floor. All reported uncertainties are computed over 24 seeds.

\begin{table}[!htb]
\centering\small\setlength{\tabcolsep}{4pt}
\caption{Resampling and rejuvenation ablation for the \emph{learned} score network. Distributional metrics are presented as ratios to the reference pool's split-half noise floor. Gate levels are quoted as the half-log SNR $\ell$ below which the rejuvenation move triggers; the adopted gate $\ell\lesssim-0.21$ is $t\ge0.30$. The no-resampling and resampling-only arms run at $\lambda\equiv1$. Rejuvenation tracks the lowest CV.}
\label{tab:cap_thr}
\begin{tabular}{llccccc}
\toprule
event & arm & $\hat p/p^\star$ & CV $\downarrow$ & KS$_0$ & KS$_1$ & EMD \\
\midrule
\multirow{6}{*}{$x_0>7$} & no resampling & 0.952 & 0.070 & 0.74 & 1.00 & 1.09 \\
 & resampling & 0.988 & 0.093 & 0.94 & 1.25 & 1.47 \\
 & resampling $+$ rejuv.\ ($M=3$) & 0.985 & 0.061 & 1.01 & 1.17 & 1.28 \\
 & \quad gated $\ell\lesssim0.05$ & 0.983 & 0.060 & 0.89 & 1.15 & 1.11 \\
 & \quad gated $\ell\lesssim-0.21$ & 0.978 & 0.060 & 1.11 & 1.20 & 1.14 \\
 & \quad gated $\ell\lesssim-0.46$ & 0.992 & 0.082 & 0.90 & 1.33 & 1.37 \\
\addlinespace
\multirow{6}{*}{$x_0>8$} & no resampling & 0.964 & 0.092 & 1.13 & 1.23 & 0.84 \\
 & resampling & 0.979 & 0.093 & 1.02 & 1.09 & 1.09 \\
 & resampling $+$ rejuv.\ ($M=3$) & 0.995 & 0.056 & 1.01 & 0.93 & 0.78 \\
 & \quad gated $\ell\lesssim0.05$ & 0.990 & 0.059 & 0.93 & 0.83 & 0.92 \\
 & \quad gated $\ell\lesssim-0.21$ & 0.985 & 0.064 & 0.97 & 1.00 & 0.73 \\
 & \quad gated $\ell\lesssim-0.46$ & 0.971 & 0.086 & 1.00 & 1.14 & 1.53 \\
\addlinespace
\multirow{6}{*}{$x_0>9$} & no resampling & 0.983 & 0.127 & 0.94 & 0.92 & 0.73 \\
 & resampling & 0.972 & 0.159 & 0.92 & 1.24 & 1.33 \\
 & resampling $+$ rejuv.\ ($M=3$) & 0.984 & 0.068 & 0.96 & 1.12 & 0.92 \\
 & \quad gated $\ell\lesssim0.05$ & 0.985 & 0.073 & 1.10 & 1.08 & 0.84 \\
 & \quad gated $\ell\lesssim-0.21$ & 0.980 & 0.089 & 0.77 & 0.83 & 0.98 \\
 & \quad gated $\ell\lesssim-0.46$ & 0.978 & 0.109 & 0.99 & 1.53 & 1.74 \\
\addlinespace
\multirow{6}{*}{$x_0>10$} & no resampling & 1.020 & 0.213 & 1.16 & 1.13 & 1.85 \\
 & resampling & 0.986 & 0.166 & 0.75 & 1.74 & 3.90 \\
 & resampling $+$ rejuv.\ ($M=3$) & 0.994 & 0.085 & 1.48 & 1.59 & 2.10 \\
 & \quad gated $\ell\lesssim0.05$ & 0.981 & 0.090 & 0.97 & 0.94 & 1.45 \\
 & \quad gated $\ell\lesssim-0.21$ & 0.965 & 0.079 & 1.23 & 1.07 & 0.98 \\
 & \quad gated $\ell\lesssim-0.46$ & 0.956 & 0.119 & 0.92 & 1.21 & 1.05 \\
\bottomrule
\end{tabular}
\end{table}

\paragraph{Benefits of SMC} Across, both Tables~\ref{tab:2x2} and~\ref{tab:cap_thr}, resampling prunes low-weight trajectories and rejuvenation decorrelates duplicated samples. We find that rejuvenation keeps the relative variance constant across rarity for both analytical and learned scores, showing a noticeable decrease in CV if the base estimator is poor \cite{del2006sequential, del2012adaptive}.
\begin{figure}[ht]
\centering
\includegraphics[width=0.9\linewidth]{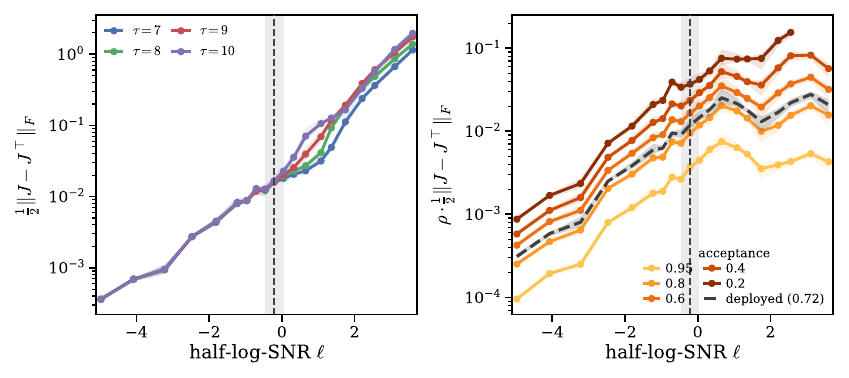}
\caption{(Left) Magnitude of the antisymmetric part of the Jacobian of the score network, measured as a Frobenius norm for 1024 particles for the different threshold exceedance problems of Table~\ref{tab:cap_thr}. (Right) Same value as the left multiplied by $\rho$ for $\tau > 10$ threshold exceedance problem,  for different calibrated step sizes $\rho$, denoted by their average acceptance rate. The shaded band marks the range of gates tested and the dashed line the adopted gate.}
\label{fig:curl_tau}
\end{figure}

\paragraph{Gating the rejuvenation move by noise level} \label{app:madm_curl_exposition} For the analytical score, there is no score approximation error, hence the line-integral of the Metropolis-Hastings acceptance ratio is defined exactly up to quadrature error \cite{lam2026metropolis}. However for a learned score it is approximated by $\sth \approx \nabla_{x_t} \log p_t$, and any errors in the learning process can lead to a non-conservative score approximation. The left of Fig.~\ref{fig:curl_tau}, evaluates this via the antisymmetric part of the Jacobian; here we reproduce the findings of \cite{pmlr-v202-lai23d, thornton2025composition}, where the score network becomes increasingly non-conservative as the denoising time, $t \rightarrow 0$.

A MALA transition kernel requires that $M_k$ fulfil the Kolmogorov criteria of reversibility.
As the Metropolis-Hastings acceptance ratio of MADM \cite{lam2026metropolis} is estimated by the line-integral of the score, i.e. Eq.~\eqref{eq:madm_line_integral}; this means the line integral, $I(\cdot)$ must fulfil the cyclic sum (or Stokes' Theorem),
\begin{equation}
I(b, a) + I(c, b) + I(a, c) = 0,
\end{equation}
for all triples $(a, b, c)$. This means that the acceptance ratio of the transition: $a \rightarrow b \rightarrow c \rightarrow a$ must be equal to 1 for MADM kernel to equate to MALA. A non-conservative score does not fulfil this criterion \cite{vuong2025we} and hence for a learned score, MADM need not be invariant to $\pi_k$.

To account for this, we turn off rejuvenation at a specific threshold in time or $\ell_t$ when the non-conservativity of the score increases.
Table~\ref{tab:cap_thr} ablates different thresholds. At $\tau = 10$ the ungated move reads KS $1.48/1.59$ and EMD $2.10$ against $1.23/1.07/0.98$ when gated at $t\ge0.30$, and the same ordering holds at $\tau=9$; we therefore adopt the gated rejuvenation for the rest of our examples, selecting the middle of the values we tested, $\ell_t \le -0.21$ ($t\ge0.30$). At $\tau = 7, 8$, where the corrector and resampling events were triggered less frequently, the gated and ungated rows coincide within noise, as the ungated version already sits close to the noise floor. Both gated and un-gated arms at $\tau=10$ read $0.965$ and $0.994$ within $1.3$ standard errors. 

\paragraph{Step size selection} As the step-size $\rho$ defines the region where a move is proposed, i.e. $\tilde{x} - x \sim \N(\rho g(x)/2, \rho I)$ of Eq.~\eqref{eq:mala_proposal}, a larger $\rho$ leads to, what we hypothesise, a higher exposure to the non-conservative part of the score network, which affects the estimated $M_k$ further. For a simple visualisation, we plot the proposed step size $\rho$ multiplied by the Frobenius norm of the antisymmetric part of the Jacobian, to the right of Fig~\ref{fig:curl_tau}. This for various $\rho$ on the $\tau = 10$ threshold exceedance problem of Section~\ref{sec:exp_2dgmm}; here the step size has been tuned to target a range of acceptance probabilities as is common \cite{lam2026metropolis}. We elect to take smaller, more conservative steps, tuning $\rho$ to target a 0.72 acceptance.

We hypothesise that, for a fixed number of MADM steps and $\rho$, there is a tradeoff between the total number of MADM steps that rejuvenate duplicated particles and the net error injected into the system by an approximate MH ratio. An adaptive schedule on $\rho(t)$ or $M(t)$ that accounts for these trade-offs could be an interesting direction for future work.

\paragraph{Ablation of Errors due to MADM}
In this setting we ablate what the discretisation error is at different $K$ using MADM. As MADM targets the continuous-time marginal $p_t$ instead of the discretised $\gamma_0$ law, it incurs a discretisation bias. We find the same issue while not plotted here, exists with MALA suggesting it is not a quadrature error itself. We verify this using analytical scores and the variance minimising Doob's guidance. See Fig.~\ref{fig:madm_error_bias}.

\begin{figure}[!ht]
    \centering
    \includegraphics[width=1.0\linewidth]{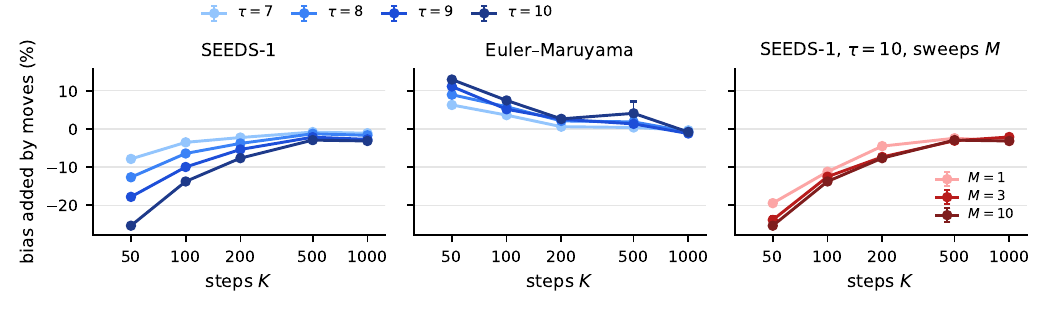}
   \caption{Bias introduced by MADM for the two integrators, using analytical scores and Doob's guidance with $M = 3$ for the first two panels. \emph{Left \& Middle} Injected bias affects rarer samples more, with all biases decreasing as $K \rightarrow \infty$. The two integrators have a different error profile. \emph{Right} Bias plotted for the $\tau = 10$ threshold exceedance problem. We find that the error grows with $M$. Error bars are due to seed variability.}
\label{fig:madm_error_bias}
\end{figure}

\subsection{Box and multiple box events}
\label{app:box_events}
We test the guidance on boxed geometries, each at two rarities, plus a bounded family. This is an analogue of an extreme defined as a set where variables exist in some range. Both families are boxes in the observable, i.e., $\Ev = \{x : a \le \phi(x) \le b\}$ with $a,b \in (\R\cup\{\pm\infty\})^m$ and the inequalities read component-wise. We take the interval guidance of Eq.~\eqref{eq:interval_reward}, see App.~\ref{app:guidance_gradient}. Any corner boxes that we use, takes $b = +\infty$, and it represents extremes that are joint threshold exceedances; bounded boxes have finite $a$ and $b$. 
\begin{figure}
\centering
\includegraphics[width=0.8\linewidth]{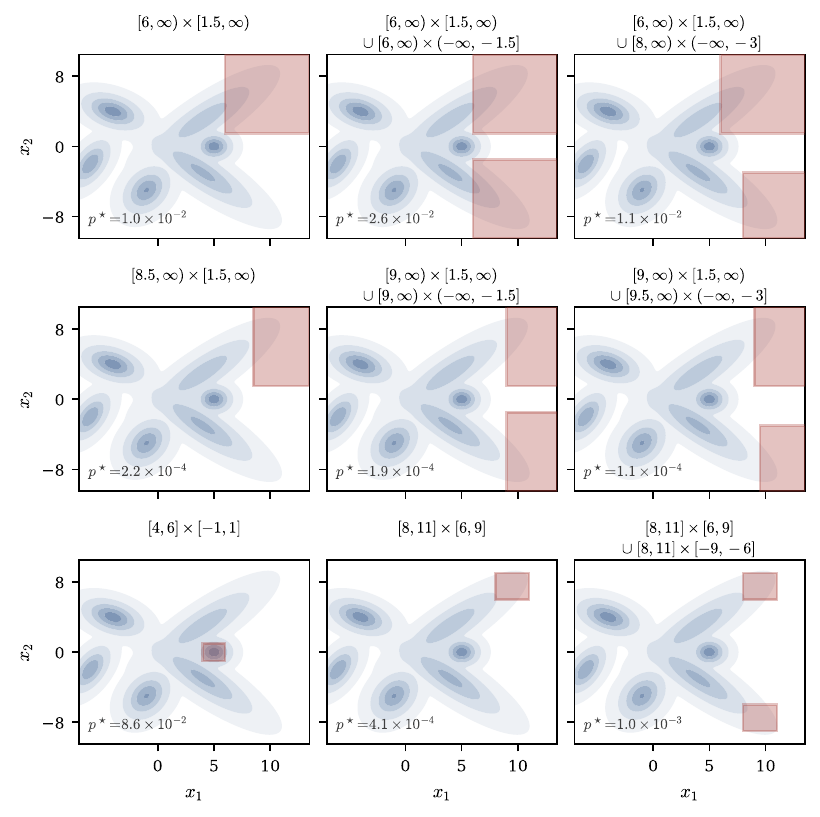}
\caption{Illustrative example of box boundaries. Plotted against a backdrop of the unconditional density of the 2D-GMM. Quantitative results of rare-event sampling are given in Table~\ref{tab:cap_box} below.}
\label{fig:boxes_illustrated}
\end{figure}

For simplicity we have illustrated the box boundaries we test in Fig~\ref{fig:boxes_illustrated} overlaid against the true unconditional density of the analytical 2D GMM. We report performance on the learned variant in Table~\ref{tab:cap_box}, where we ablate the effect of no resampling, resampling with gated/un-gated rejuvenation. We find that similar to the other experiments, rejuvenation lowers the CV on the rare boxes ($p^\star\le2\times10^{-4}$, by $1.5$-$2.5\times$) and is neutral or slightly worse on the easy ones ($p^\star\ge10^{-2}$), where resampling is triggered at most twice per run. 

\begin{table}[!htb]
\centering\footnotesize\setlength{\tabcolsep}{4pt}
\caption{Box events on the learned score. Every region is drawn on the model density in Fig.~\ref{fig:boxes_illustrated}. $p^\star$ is the $K=1000$ chain's own tail (Table~\ref{tab:model_bias} pool). We compare the estimate $\hat{p}$ and the shape metrics of the final particles, presented as a ratio against a noise floor.}
\begin{tabular}{llcccccc}
\toprule
event & variant & $p^\star$ & $\hat p/p^\star$ & CV $\downarrow$ & KS$_0$ & KS$_1$ & EMD \\
\midrule
\multirow{4}{*}{$[6,\infty)\times[1.5,\infty)$} & no resampling & \multirow{4}{*}{$9.97\!\times\!10^{-3}$} & 0.955 & 0.066 & 0.95 & 0.98 & 1.01 \\
 & resampling &  & 1.004 & 0.107 & 1.37 & 0.87 & 0.98 \\
 & resampling $+$ rejuv.\ ($M=3$) &  & 1.000 & 0.070 & 0.89 & 0.91 & 0.96 \\
 & \quad gated $\ell\lesssim-0.21$ &  & 1.000 & 0.070 & 0.94 & 0.89 & 1.00 \\
\addlinespace
\multirow{4}{*}{\shortstack[l]{$[6,\infty)\times[1.5,\infty)$\\$\cup\,[6,\infty)\times(-\infty,-1.5]$}} & no resampling & \multirow{4}{*}{$2.60\!\times\!10^{-2}$} & 0.966 & 0.057 & 1.15 & 1.15 & 1.09 \\
 & resampling &  & 0.979 & 0.055 & 0.99 & 1.08 & 0.95 \\
 & resampling $+$ rejuv.\ ($M=3$) &  & 0.988 & 0.042 & 0.87 & 1.10 & 1.15 \\
 & \quad gated $\ell\lesssim-0.21$ &  & 0.992 & 0.044 & 0.96 & 1.44 & 1.50 \\
\addlinespace
\multirow{4}{*}{\shortstack[l]{$[6,\infty)\times[1.5,\infty)$\\$\cup\,[8,\infty)\times(-\infty,-3]$}} & no resampling & \multirow{4}{*}{$1.08\!\times\!10^{-2}$} & 0.956 & 0.063 & 1.10 & 0.85 & 0.82 \\
 & resampling &  & 0.992 & 0.086 & 0.90 & 1.27 & 1.15 \\
 & resampling $+$ rejuv.\ ($M=3$) &  & 1.011 & 0.089 & 0.94 & 1.03 & 1.36 \\
 & \quad gated $\ell\lesssim-0.21$ &  & 1.011 & 0.089 & 1.02 & 1.10 & 1.07 \\
\addlinespace
\multirow{4}{*}{$[8.5,\infty)\times[1.5,\infty)$} & no resampling & \multirow{4}{*}{$2.16\!\times\!10^{-4}$} & 0.945 & 0.132 & 1.21 & 0.96 & 0.90 \\
 & resampling &  & 0.967 & 0.116 & 1.19 & 0.86 & 0.96 \\
 & resampling $+$ rejuv.\ ($M=3$) &  & 0.979 & 0.086 & 1.07 & 0.90 & 0.96 \\
 & \quad gated $\ell\lesssim-0.21$ &  & 0.982 & 0.105 & 1.17 & 0.91 & 1.02 \\
\addlinespace
\multirow{4}{*}{\shortstack[l]{$[9,\infty)\times[1.5,\infty)$\\$\cup\,[9,\infty)\times(-\infty,-1.5]$}} & no resampling & \multirow{4}{*}{$1.87\!\times\!10^{-4}$} & 0.986 & 0.123 & 1.15 & 1.32 & 0.85 \\
 & resampling &  & 0.990 & 0.133 & 1.10 & 1.17 & 1.28 \\
 & resampling $+$ rejuv.\ ($M=3$) &  & 1.008 & 0.079 & 1.22 & 1.34 & 1.12 \\
 & \quad gated $\ell\lesssim-0.21$ &  & 1.008 & 0.113 & 1.02 & 1.03 & 0.96 \\
\addlinespace
\multirow{4}{*}{\shortstack[l]{$[9,\infty)\times[1.5,\infty)$\\$\cup\,[9.5,\infty)\times(-\infty,-3]$}} & no resampling & \multirow{4}{*}{$1.12\!\times\!10^{-4}$} & 0.972 & 0.164 & 1.11 & 1.21 & 1.26 \\
 & resampling &  & 0.972 & 0.133 & 1.18 & 1.83 & 1.97 \\
 & resampling $+$ rejuv.\ ($M=3$) &  & 0.998 & 0.066 & 1.15 & 1.00 & 1.29 \\
 & \quad gated $\ell\lesssim-0.21$ &  & 1.010 & 0.088 & 1.06 & 1.00 & 0.82 \\
\addlinespace
\multirow{4}{*}{$[4,6]\times[-1,1]$} & no resampling & \multirow{4}{*}{$8.66\!\times\!10^{-2}$} & 0.968 & 0.032 & 1.04 & 1.16 & 1.13 \\
 & resampling &  & 0.987 & 0.050 & 1.18 & 1.01 & 1.08 \\
 & resampling $+$ rejuv.\ ($M=3$) &  & 0.986 & 0.035 & 1.19 & 0.96 & 1.09 \\
 & \quad gated $\ell\lesssim-0.21$ &  & 0.992 & 0.035 & 1.04 & 0.89 & 0.98 \\
\addlinespace
\multirow{4}{*}{$[8,11]\times[6,9]$} & no resampling & \multirow{4}{*}{$4.11\!\times\!10^{-4}$} & 0.951 & 0.067 & 1.06 & 1.11 & 1.07 \\
 & resampling &  & 0.996 & 0.092 & 1.02 & 1.06 & 0.96 \\
 & resampling $+$ rejuv.\ ($M=3$) &  & 0.984 & 0.052 & 0.96 & 1.15 & 1.04 \\
 & \quad gated $\ell\lesssim-0.21$ &  & 0.983 & 0.051 & 0.97 & 1.10 & 0.98 \\
\addlinespace
\multirow{4}{*}{\shortstack[l]{$[8,11]\times[6,9]$\\$\cup\,[8,11]\times[-9,-6]$}} & no resampling & \multirow{4}{*}{$1.01\!\times\!10^{-3}$} & 0.964 & 0.057 & 0.81 & 1.41 & 1.35 \\
 & resampling &  & 0.984 & 0.056 & 0.81 & 0.84 & 1.09 \\
 & resampling $+$ rejuv.\ ($M=3$) &  & 0.986 & 0.050 & 0.86 & 1.12 & 1.57 \\
 & \quad gated $\ell\lesssim-0.21$ &  & 0.984 & 0.053 & 0.89 & 0.90 & 1.26 \\
\bottomrule
\end{tabular}

\label{tab:cap_box}
\end{table}
\subsection{Ablation of Parameters}
\label{app:gmm_schedule_ablations}
We ablate our parameters once: one the $\tau = 10$ threshold exceedance problem and we reuse it throughout our experiments. There are two parameters we ablate over the starting point of our guidance $t_{\text{on}}$ and the exponential $p$ in $\lambda(t) = 1 - (t/t_{\rm on})^p$. Below are the results of the ablation for our method, separated by the base, resampling and resampling with rejuvenation; we also ablate for the lowest $\mathrm{CV}$ of \citet{manshausen2026accurateextremeeventlikelihoods} which we use for the experiments in Fig.~\ref{fig:table_1_replacement}. All experiments were done with 24 seeds, $N = 1000$, $K = 1000$ using the SEEDS-1 integrator \cite{gonzalez2023seeds}, with $M = 3$ steps of MADM \cite{lam2026metropolis} after resampling.

\begin{table}[ht]
\centering
\caption{Schedule ablation at $\tau=10$, $K=1000$, $N=1000$: CV of $\hat p$ for the tempering
schedule $\lambda(t)=1-(t/t_{\rm on})^p$. Bold marks the minimum of each column. We find that the no-rejuvenation variants converged on the same tempering schedule $\lambda_t \equiv 1$ compared to the MADM variant. All schedules with rejuvenation yield a lower CV compared to without.}
\label{tab:schedule_ablation}
\small
\begin{tabular}{lcccc}
\toprule
schedule & base $\downarrow$ & + resampling$\downarrow$ & + MADM $\downarrow$& \cite{manshausen2026accurateextremeeventlikelihoods} $\downarrow$\\
\midrule
\multicolumn{5}{l}{\emph{analytic score}} \\
$\lambda\equiv1$            & \textbf{0.15} & \textbf{0.23} & 0.09          & \textbf{0.29} \\
$t_{\rm on}=1.0,\ p=2$      & 0.28          & 0.32          & 0.08          & 0.68 \\
$t_{\rm on}=0.9,\ p=2$ (deployed) & 0.35    & 0.33          & \textbf{0.07} & 0.83 \\
$t_{\rm on}=0.7,\ p=2$      & 0.62          & 0.40          & 0.10          & 1.26 \\
$t_{\rm on}=0.5,\ p=2$      & 1.73          & 0.91          & 0.13          & 1.94 \\
$t_{\rm on}=0.9,\ p=1$      & 0.71          & 0.48          & 0.11          & 1.45 \\
$t_{\rm on}=0.9,\ p=4$      & 0.20          & 0.23          & 0.10          & 0.56 \\
\midrule
\multicolumn{5}{l}{\emph{learned score}} \\
$\lambda\equiv1$            & \textbf{0.21} & \textbf{0.17} & 0.11          & \textbf{0.17} \\
$t_{\rm on}=1.0,\ p=2$      & 0.32          & 0.26          & 0.13          & 0.51 \\
$t_{\rm on}=0.9,\ p=2$ (deployed) & 0.40    & 0.29          & \textbf{0.08} & 0.68 \\
$t_{\rm on}=0.7,\ p=2$      & 0.70          & 0.36          & 0.15          & 1.18 \\
$t_{\rm on}=0.5,\ p=2$      & 2.00          & 0.75          & 0.25          & 1.90 \\
$t_{\rm on}=0.9,\ p=1$      & 0.78          & 0.37          & 0.18          & 1.28 \\
$t_{\rm on}=0.9,\ p=4$      & 0.23          & 0.18          & 0.16          & 0.45 \\
\bottomrule
\end{tabular}
\end{table}

\FloatBarrier

\newpage

\section{Climate experiments}
\label{app:planned_climemu}

\subsection{Setup and model details}
\label{app:climemu_setup}
\label{app:model_details_climemu}

We use a pretrained score-based diffusion climate emulator that acts as a surrogate for physics-based Earth System Models \cite{https://doi.org/10.1029/2025MS005558} (licensed under MIT Licence), trained on large-ensemble runs from the Coupled Model Intercomparison Project (CMIP6) on different climate change (forcing) scenarios \cite{o2016scenario}. Our results are based on the MPI‐ESM1‐2‐LR emulator. The model outputs monthly-averaged fields conditioned on the global mean surface temperature anomaly (GMST), which captures global warming against preindustrial warming, and on the month of the year, which captures annual variability.

The emulator generates four surface fields jointly across the globe on every
draw, listed in Table~\ref{tab:climemu_variables} by their CMIP6 short names. All fields are monthly means and are handled throughout as \emph{anomalies}
against the pre-industrial (piControl) climatology, matching the convention of
Table~\ref{tab:tail-thresholds}.

\begin{table}[h]
    \centering
    \small
    \caption{Emulated variables with their corresponding CMIP6 short names,
    descriptions and units.}
    \label{tab:climemu_variables}
    \begin{tabular}{lll}
    \toprule
    Short name & Description & Units \\
    \midrule
    \texttt{tas}     & $2$~m air temperature    & K \\
    \texttt{hurs}    & $2$~m relative humidity  & \% \\
    \texttt{pr}      & Precipitation rate       & mm\,day$^{-1}$ \\
    \texttt{sfcWind} & $10$~m wind speed        & m\,s$^{-1}$ \\
    \bottomrule
    \end{tabular}
    
\end{table}

We compute a large Monte Carlo average of the rare event probability with $N = 9{,}011{,}200$ unguided draws at a GMST of $\Delta T = 1.5$~K for the month of July to validate our results against. For the region $R$, with $\phi(\cdot)$ as a per-variable regional average, the associated mean, standard deviation and tail values at different rarities of interest are given in Table~\ref{tab:tail-thresholds}. 
\begin{table}[h]
  \centering
  \setlength{\tabcolsep}{4.5pt}
  \small
  \caption{Tail thresholds for the four emulator variables, from
  $N = 9{,}011{,}200$ unguided Monte Carlo draws over $R$ at $\Delta T = +1.5$~K in July. Upper-tail entries satisfy
  $P(\phi \ge \tau) = p$, lower-tail entries $P(\phi \le \tau) = p$. Binomial
  relative standard errors are $1.1\%$, $3.3\%$ and $10.5\%$ at
  $p = 10^{-3}, 10^{-4}, 10^{-5}$ respectively. The guided runs use the same thresholds. }
  \label{tab:tail-thresholds}
  \begin{tabular}{llrrrrrrrr}
    \toprule
    & & \multicolumn{2}{c}{Marginal}
      & \multicolumn{3}{c}{Upper tail}
      & \multicolumn{3}{c}{Lower tail} \\
    \cmidrule(lr){3-4}\cmidrule(lr){5-7}\cmidrule(lr){8-10}
    Variable & Unit & Mean & SD
      & $10^{-3}$ & $10^{-4}$ & $10^{-5}$
      & $10^{-3}$ & $10^{-4}$ & $10^{-5}$ \\
    \midrule
    \texttt{tas}     & K      & $+2.577$ & $1.023$
      & $5.813$  & $6.498$  & $7.096$
      & $-0.539$ & $-1.174$ & $-1.727$ \\
    \texttt{pr}      & mm/day & $-0.250$ & $0.450$
      & $1.698$  & $2.311$  & $2.842$
      & $-1.334$ & $-1.517$ & $-1.650$ \\
    \texttt{hurs}    & \%     & $-3.275$ & $4.220$
      & $10.519$ & $13.485$ & $16.054$
      & $-15.838$ & $-18.387$ & $-20.554$ \\
    \texttt{sfcWind} & m/s    & $+0.193$ & $0.186$
      & $0.773$  & $0.895$  & $1.004$
      & $-0.387$ & $-0.508$ & $-0.613$ \\
    \bottomrule
  \end{tabular}
  
\end{table}
\subsubsection{Implementation Details}
\label{app:climemu_impl}
For all experiments throughout we use $K = 200$ discretisation of the exponential integrator in the noise-prediction mode, see App.~\ref{app:rne_weights} and Fig.~\ref{fig:climate_kerror}. We schedule $\lambda(t)$ and $\delta(t)$ for all experiments below. For a VE-schedule, $\alpha_t \equiv 1$, so the half-log SNR is $\ell_t = - \log \sbar_t$ and we specify both schedule in $\sbar_t$: piecewise linear in $\lambda$ and $\delta_t$ is interpolated log-linearly between two endpoints. Similar to GMM example, we turn guidance, i.e. $\lambda(t) \equiv 0$ off for the first ~15\% ($\sbar > 40)$ of the denoising process to account for any approximation errors in $\xh$, converging to $\lambda(\sbar) \rightarrow 1$ after adding a lambda overshoot at a maximum of 1.25. $\delta(t)$ is held at 0.5 until $\sbar = 0.3$ SD and sharpens to 0.01 by $\sbar = 0.03$SD, where SD is the standard deviation of each observable approximated by the same method as the $\Pi$GDM mentioned in App~\ref{app:diffusion_details}. We decrease $\delta$ towards the end of the denoising process to reduce the variance of our estimator due to mismatch in the generated $\pi_0$. For simplicity, we do not anneal $\tau(\sbar)$.

Correctness of our method is independent of these choices, so they are selected for variance alone. We set them empirically on the $p^* \approx 10^{-3}$ rarity, PNW marginal hot and dry threshold exceedance problems using a pilot sweep at the same setup as described above. We provide an ablation of the values we tested below (Table~\ref{tab:climemu-recipe-schedule}), but adopt the \emph{same} hyper-parameters for all experiments, with $N = 1024$, $K = 200$ and $16$ seeds throughout, see Table~\ref{tab:climemu-hparams}.

\begin{figure}[!h]
    \centering
\includegraphics[width=0.75\linewidth]{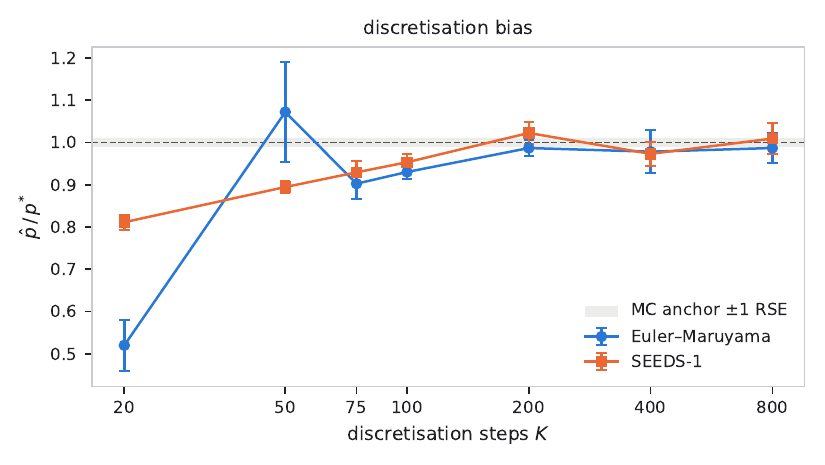}
   \caption{Discretisation bias for the integrator for the $\tau_1$ PNW heatwave event using the SEEDS-1 integrator, with 1024 samples at each point and without any rejuvenation. Dotted line represents the reference pool. We find that for SEEDS-1 and Euler-Maruyama integrators using $K = 200$ saturates the discretisation bias. We could not compare against a large $K$ unconditional reference due to prohibitive costs. Bars represent $\pm 1$ standard error of the mean over 8 shared seeds across both integrators.}
\label{fig:climate_kerror}
\end{figure}

All experiments were done on a single Nvidia H100 GPU with 96 GB of VRAM. We used approximately 370 GPUh to generate the 9 million reference samples. A single batch of 1024 unconditional draws at $K=200$ steps takes approximately 155 seconds, with the method and its variants multiplying this cost; we summarise this in Table~\ref{tab:compute_cost}.
\begin{table}[t]
  \centering
  \caption{Deployed schedule and sampler configuration for the climate
  emulator experiments. $\bar\sigma$ is the marginal noise standard
  deviation. Selected once on the marginal $10^{-3}$ PNW hot and dry events and applied
  unchanged to other experiments. $^{\ddagger}$$\delta$ is quoted in units of the observable's unconditional
  standard deviation.}
  \label{tab:climemu-hparams}
  \begin{tabular}{lll}
    \toprule
    & Symbol & Value \\
    \midrule
    \multicolumn{3}{l}{\emph{Schedules}}\\
    Tempering onset            & $\bar\sigma_{\mathrm{on}}$   & $40$ \\
    Tempering peak             & $\lambda_{\max}$             & $1.25$ at $\bar\sigma = 2$ \\
    Tempering relaxation       & $\bar\sigma_{\mathrm{end}}$  & $0.3$ \\
    Smoothing width, transport & $\delta_{\mathrm{hi}}$       & $0.5\,\mathrm{SD}$ \\
    Smoothing width, floor     & $\delta_{\mathrm{lo}}$       & $0.01\,\mathrm{SD}$ \\
    Ramp window                &                              & $\bar\sigma \in [0.03, 0.3]$ \\
    Ramp exponent              &                              & $\log_{10} 50 \approx 1.70$ \\
    Observable SD, frozen$^{\ddagger}$ (\texttt{tas}, \texttt{hurs}) & SD & $1.017$~K, $4.224\%$ \\
    Threshold                  & $\tau$                       & constant \\
    \midrule
    \multicolumn{3}{l}{\emph{Sampler}}\\
    Particles                  & $N$                          & $1024$ \\
    Integration steps          & $K$                          & $200$ \\
    Integrator                 &                              & exponential (noise-prediction) \\
 
    Resampling                 &                              & adaptive, $\mathrm{ESS}/N < 0.5$ \\

    Rejuvenation               & $M$                          & $0$ \\
    Noise range                & $[\bar\sigma_{\min}, \bar\sigma_{\max}]$ & $[0.01, 146.09]$ \\
    Posterior proxy            & $s_t^2$                      & $\bar\sigma^2/(1 + \bar\sigma^2)$ \\
    Look-ahead bandwidth       & $\delta_{\mathrm{eff}}^2$    & $\delta(\bar\sigma)^2 + s_t^2\|a\|^2$ \\
    \bottomrule
  \end{tabular}
\end{table}
 
\subsection{Experiment 1: Compound hot \& dry events}

\label{app:climemu_exp1}
This subsection supports the first climate experiment of Section~\ref{sec:exp_climate}: the
ablation behind the deployed recipe, and the fields the guided sampler produces at
each rarity. We begin with the marginal threshold exceedance problem, with the hot (\texttt{tas}) and dry (\texttt{hurs}) marginal results reported in Table~\ref{tab:marginal_hot_dry}, followed by the joint.

\begin{table}[t]
\centering
\setlength{\tabcolsep}{4pt}
\caption{Marginal extreme-event probabilities over the PNW target region in July at $\Delta T = +1.5$~K, for hot ($\phi_{\texttt{tas}} \ge \tau$) and dry ($\phi_{\texttt{hurs}} \le \tau$) events. For each variable, $\tau_{1}, \tau_{2}, \tau_{3}$ are chosen so that the event is successively rarer (exact values in App.~\ref{app:planned_climemu}). Ground-truth $p^\star$ are computed with $\sim 10^7$ MC samples. Unguided Monte Carlo and our guided method use the same computational budget: each unguided estimate uses $4813$ draws ($4.7\times$ the $1024$ guided particles, Table~\ref{tab:compute_cost}), over 16 disjoint replicates. Estimated probabilities, $\hat{p}$ are reported as a ratio $\hat{p}/p^*$ with mean $\pm 1$ standard deviation over 16 seeds. The variance ratio (VR) is $\Var(\hat{p}_{\text{unguided}}) / \Var(\hat{p}_{\text{guided}})$. The number of unguided replicates with no event is 1 (hot) and 0 (dry) of 16 at $\tau_1$, 9 and 12 at $\tau_2$, and 14 for both at $\tau_3$, where the other two replicates record a single event each, so the VR at $\tau_3$ is a rough estimate. $^\dagger$Includes one collapsed run (App.~\ref{app:climemu_resampling_ablation}). The hot results are shown in Fig.~\ref{fig:tas_ladder_uncond}.}
\label{tab:marginal_hot_dry}
\begin{tabular}{@{}l cccc cccc@{}}
\toprule
 & \multicolumn{4}{c}{Hot: $\phi_{\texttt{tas}}(x) \ge \tau$} & \multicolumn{4}{c}{Dry: $\phi_{\texttt{hurs}}(x) \le \tau$} \\
\cmidrule(lr){2-5} \cmidrule(l){6-9}
 & & \multicolumn{2}{c}{$\hat{p}/p^\star$} & & & \multicolumn{2}{c}{$\hat{p}/p^\star$} & \\
\cmidrule(lr){3-4} \cmidrule(lr){7-8}
 & $p^\star$ & Unguided & Guided & VR & $p^\star$ & Unguided & Guided & VR \\
\midrule
$\tau_{1}$ & $1.00\text{e-}3$ & $0.90 \pm 0.45$ & $\mathbf{0.98 \pm 0.08}$ & $34\times$
           & $1.00\text{e-}3$ & $1.03 \pm 0.44$ & $\mathbf{1.00 \pm 0.07}$ & $37\times$ \\
$\tau_{2}$ & $1.00\text{e-}4$ & $1.30 \pm 1.67$ & $\mathbf{1.01 \pm 0.08}$ & $411\times$
           & $1.00\text{e-}4$ & $0.52 \pm 0.93$ & $\mathbf{1.02 \pm 0.22}^{\dagger}$ & $18\times$ \\
$\tau_{3}$ & $1.01\text{e-}5$ & $2.57 \pm 7.03$ & $\mathbf{1.05 \pm 0.13}$ & $2722\times$
           & $1.01\text{e-}5$ & $2.57 \pm 7.03$ & $\mathbf{1.06 \pm 0.12}$ & $3567\times$ \\
\bottomrule
\end{tabular}
\end{table}

\subsubsection{Selection and rejuvenation}
\label{app:climemu_resampling_ablation}
Similar to Appendix~\ref{app:madm_curl}, we ablate the effect of resampling and rejuvenation on the hot and dry threshold exceedance events and the joint hot-and-dry event that we report the net speed-up on. This is because while resampling does not cost extra wall clock time, rejuvenation does. We evaluate whether the reduction in CV justifies the added computation cost. For this, we use a fixed $M = 5$ MADM steps using 2 degrees of quadrature. Additionally, we cache the estimate of $\log \Gamma_\pi$ in Eq.~\eqref{eq:madm_simpson} for either the current $x$ or proposal $\tilde{x}$ so that it may be reused between moves. For the same reasons as highlighted in Appendix~\ref{app:more_2dgmm}, we set the acceptance target at $0.72$ and gate the rejuvenation move to run only at noise levels $\bar{\sigma} \ge 0.30$.

Table~\ref{tab:climemu_ablation} captures this ablation. We report the CV which is the spread of the estimated $\hat{p}$ across 16 seeds; $\hat{k}$ is the generalised-Pareto shape parameter of the importance weights, fitted with the top 96  ($\min (0.2 N, 3 \sqrt{N})$) \cite{vehtari2024pareto}, within a run, reported as the median over seeds.  We find that without resampling the median $\hat{k}$ exceeds $0.5$ (no finite second moment) in seven of the nine cells. This suggests that the weights are very heavy-tailed and the convergence of the no-resampling estimator is slow to converge, with an accompanying CV that is optimistic. 

Resampling lowers the CV in eight of the nine cells; the exception, dry at $10^{-4}$, comes from a single collapsed run in which one lineage takes over the crossing particles (seed 12, $\hat p/p^\star = 1.72$ with 5 effective families); without it the CV is $0.124$. Rejuvenation does not change the CV consistently: it is lower for hot at $10^{-5}$ and for dry at $10^{-4}$, where no rejuvenated run collapses, and higher in the other seven cells, by at most $27\%$. Differences of this size are within the sampling error of a CV estimated from 16 seeds. $W_1$ is the weighted Wasserstein-1 distance of the quantity of interest $\phi(\cdot)$ between the generated samples and the Monte Carlo reference, reported against a null floor. All variants report a similar $W_1$ and there is no discernible issue with any. However, it is worth noting that rejuvenation adds ${\sim}1.1$--$1.25\times$ wall-clock time compared to resampling alone, hence there is a trade-off between a larger $N$ versus more rejuvenation steps at finite compute budget. In the experiments below, unless stated otherwise we run with zero rejuvenation steps.

\begin{table}[t]
\centering
\footnotesize
\setlength{\tabcolsep}{3pt}
\caption{Ablation of resampling and rejuvenation on the climate emulator for the
hot, dry, and joint hot-and-dry events. Same parameters as in the main text; details in App.~\ref{app:climemu_resampling_ablation}. Ratios are mean $\pm$ one standard error over 16 seeds.  $W_1$ is a 
weighted distance between the $\phi(\cdot)$ of generated  samples and draws from a reference, scaled against a null floor.
The relative error of $p^\star$ is $1.1/3.3/10.5\%$ at the three marginal rarities and
$0.6/2.3/8.9\%$ at the three joint rarities. CV is estimated across $\hat{p}$ estimates; $\hat{k}$ is estimated per seed and we report the median over the 16 seeds. $^\dagger$Includes one collapsed run; without it $\hat p/p^\star = 0.968$ and CV $= 0.124$. $\hat{k} > 0.7$ represents weights that have no finite second moment \cite{vehtari2024pareto}, so the CV estimates, and the speed-ups of those rows, are optimistic. Speed-up is over crude MC at matched CV, including each arm's measured wall-clock (App.~\ref{app:climemu_cost}); in brackets, its 95\% bootstrap CI (BCa, 20{,}000 resamples of the 16 seeds).
 }
\label{tab:climemu_ablation}
\begin{tabular}{@{}l l l r r c r r r@{}}
\toprule
 & $p^\star$ & arm & $\hat p/p^\star$ & CV $\downarrow$ &
 $\mathrm{ESS}/N$ \;\; $\hat k$ &
 \multicolumn{2}{c}{$W_1$} & Speed-up $\uparrow$ \\
\cmidrule(lr){7-8}
 & & & & & & \texttt{tas} & \texttt{hurs} & (95\% CI) \\
\midrule
\multirow{9}{*}{\rotatebox{90}{Hot}}
& $10^{-3}$ & no resampling & $1.024 \pm 0.051$ & $0.199$ & $0.06$ \; $0.49$ & $1.03$ & -- & $5.2\times$ ($2.8$--$12\times$) \\
& & adaptive resampling & $0.982 \pm 0.019$ & $0.078$ & -- & $1.54$ & -- & $34\times$ ($20$--$60\times$) \\
& & \quad + MADM moves & $0.954 \pm 0.020$ & $0.082$ & -- & $1.40$ & -- & $27\times$ ($17$--$44\times$) \\
\addlinespace[2pt]
& $10^{-4}$ & no resampling & $1.020 \pm 0.049$ & $0.191$ & $0.03$ \; $0.57$ & $1.07$ & -- & $56\times$ ($32$--$94\times$) \\
& & adaptive resampling & $1.010 \pm 0.021$ & $0.082$ & -- & $0.95$ & -- & $311\times$ ($129$--$616\times$) \\
& & \quad + MADM moves & $0.974 \pm 0.020$ & $0.083$ & -- & $1.17$ & -- & $257\times$ ($119$--$571\times$) \\
\addlinespace[2pt]
& $10^{-5}$ & no resampling & $1.091 \pm 0.062$ & $0.228$ & $0.02$ \; $0.88$ & $0.82$ & -- & $395\times$ ($191$--$815\times$) \\
& & adaptive resampling & $1.048 \pm 0.034$ & $0.129$ & -- & $0.87$ & -- & $1246\times$ ($778$--$2045\times$) \\
& & \quad + MADM moves & $1.019 \pm 0.023$ & $0.089$ & -- & $0.64$ & -- & $2179\times$ ($1264$--$3573\times$) \\
\midrule
\multirow{9}{*}{\rotatebox{90}{Dry}}
& $10^{-3}$ & no resampling & $0.931 \pm 0.036$ & $0.153$ & $0.06$ \; $0.53$ & -- & $1.15$ & $8.9\times$ ($4.0$--$23\times$) \\
& & adaptive resampling & $0.996 \pm 0.018$ & $0.072$ & -- & -- & $1.42$ & $40\times$ ($24$--$60\times$) \\
& & \quad + MADM moves & $0.962 \pm 0.017$ & $0.073$ & -- & -- & $1.44$ & $34\times$ ($17$--$58\times$) \\
\addlinespace[2pt]
& $10^{-4}$ & no resampling & $0.911 \pm 0.047$ & $0.206$ & $0.04$ \; $0.63$ & -- & $1.06$ & $49\times$ ($28$--$80\times$) \\
& & adaptive resampling & $1.015 \pm 0.055^{\dagger}$ & $0.217$ & -- & -- & $1.31$ & $44\times$ ($19$--$179\times$) \\
& & \quad + MADM moves & $0.964 \pm 0.032$ & $0.134$ & -- & -- & $1.09$ & $96\times$ ($37$--$362\times$) \\
\addlinespace[2pt]
& $10^{-5}$ & no resampling & $0.978 \pm 0.082$ & $0.333$ & $0.02$ \; $0.75$ & -- & $0.97$ & $185\times$ ($84$--$827\times$) \\
& & adaptive resampling & $1.055 \pm 0.029$ & $0.112$ & -- & -- & $0.68$ & $1655\times$ ($817$--$3271\times$) \\
& & \quad + MADM moves & $1.038 \pm 0.037$ & $0.142$ & -- & -- & $0.71$ & $815\times$ ($333$--$2381\times$) \\
\midrule
\multirow{9}{*}{\rotatebox{90}{Joint}}
& $10^{-2}$ & no resampling & $0.970 \pm 0.020$ & $0.082$ & $0.09$ \; $0.46$ & $1.25$ & $1.17$ & $6.0\times$ ($3.2$--$9.9\times$) \\
& & adaptive resampling & $0.965 \pm 0.016$ & $0.068$ & -- & $1.55$ & $1.42$ & $8.6\times$ ($3.5$--$22\times$) \\
& & \quad + MADM moves & $0.950 \pm 0.018$ & $0.074$ & -- & $1.50$ & $1.67$ & $6.6\times$ ($3.5$--$11\times$) \\
\addlinespace[2pt]
& $10^{-3}$ & no resampling & $0.967 \pm 0.043$ & $0.179$ & $0.03$ \; $0.67$ & $1.22$ & $1.09$ & $19\times$ ($9.8$--$42\times$) \\
& & adaptive resampling & $0.990 \pm 0.028$ & $0.113$ & -- & $1.34$ & $1.39$ & $47\times$ ($25$--$83\times$) \\
& & \quad + MADM moves & $1.047 \pm 0.032$ & $0.122$ & -- & $1.78$ & $1.54$ & $36\times$ ($14$--$90\times$) \\
\addlinespace[2pt]
& $10^{-4}$ & no resampling & $1.004 \pm 0.067$ & $0.268$ & $0.02$ \; $0.85$ & $1.28$ & $1.02$ & $126\times$ ($77$--$196\times$) \\
& & adaptive resampling & $1.044 \pm 0.021$ & $0.080$ & -- & $1.14$ & $0.86$ & $1413\times$ ($854$--$2551\times$) \\
& & \quad + MADM moves & $1.040 \pm 0.025$ & $0.096$ & -- & $1.00$ & $0.95$ & $842\times$ ($519$--$1372\times$) \\
\bottomrule
\end{tabular}
\end{table}

\subsubsection{Conditional fields for temperature events}
\label{app:climemu_physical}
We plot both the local (Fig.~\ref{fig:climate_plot}) and global (Fig.~\ref{fig:climate_plot_diff}) guided fields for the threshold exceedance problem to visualise representative draws from the conditional model obtained using guidance. In the example below, guidance is only applied to the region $R$ for the $\texttt{tas}$ variable, any structure or perturbation of other variables or away from $R$ is due to the unguided score.
\begin{figure}[tbp]
    \centering
    \includegraphics[width=1.0\linewidth]{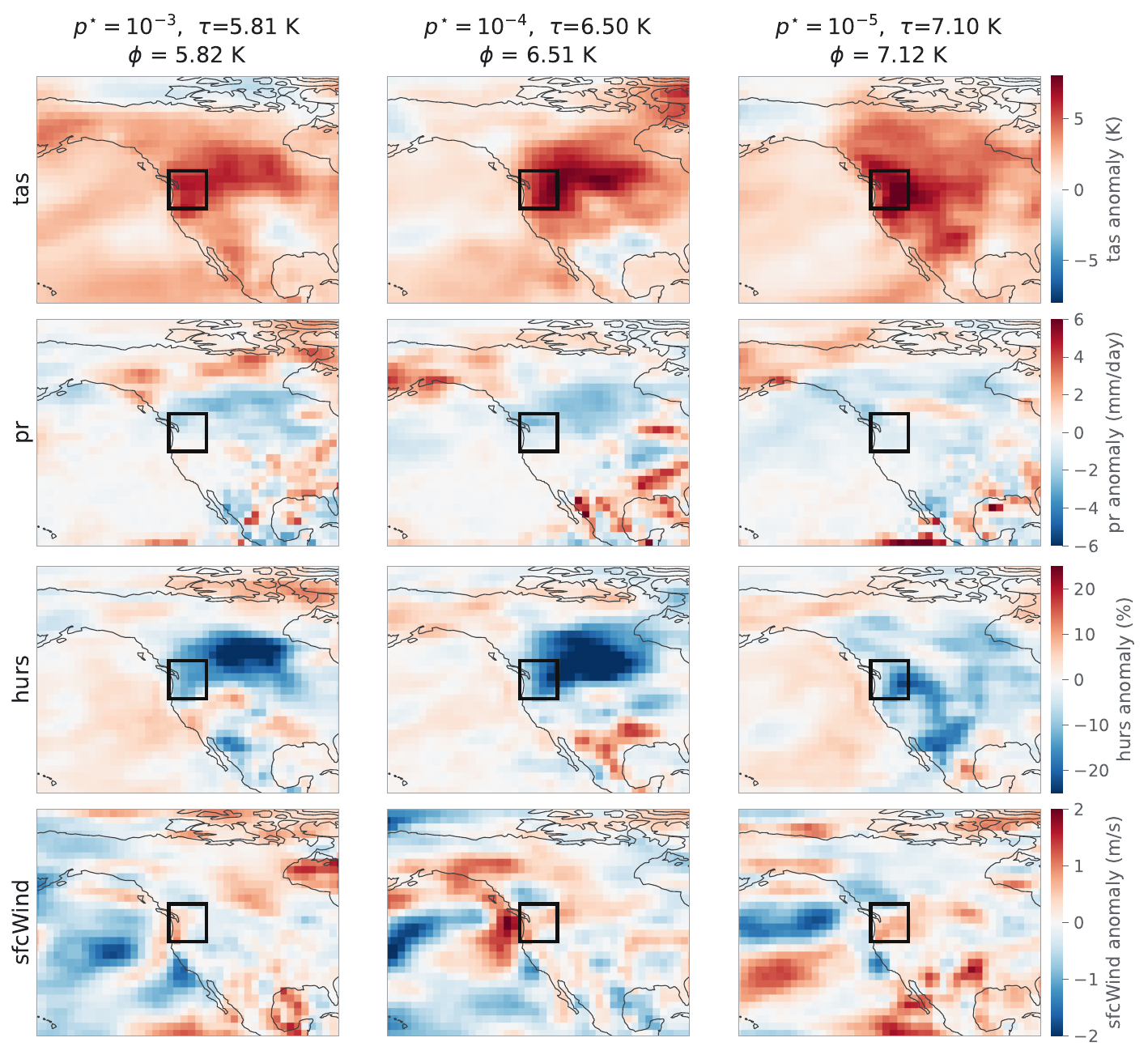}
    \caption{Top-weight crossing fields over the Pacific Northwest (PNW) target region produced by the guided sampler at different exceedance threshold $\tau$. Guided samples based on the $\texttt{tas}$ variable alone. Same setup as reference Table~\ref{tab:marginal_hot_dry}. $p$ denotes $p_0[E]$ and $\phi(\cdot)$ is the regional average temperature.}
    \label{fig:climate_plot}
\end{figure}

\begin{figure}[tbp]
    \centering
    \includegraphics[width=1.0\linewidth]{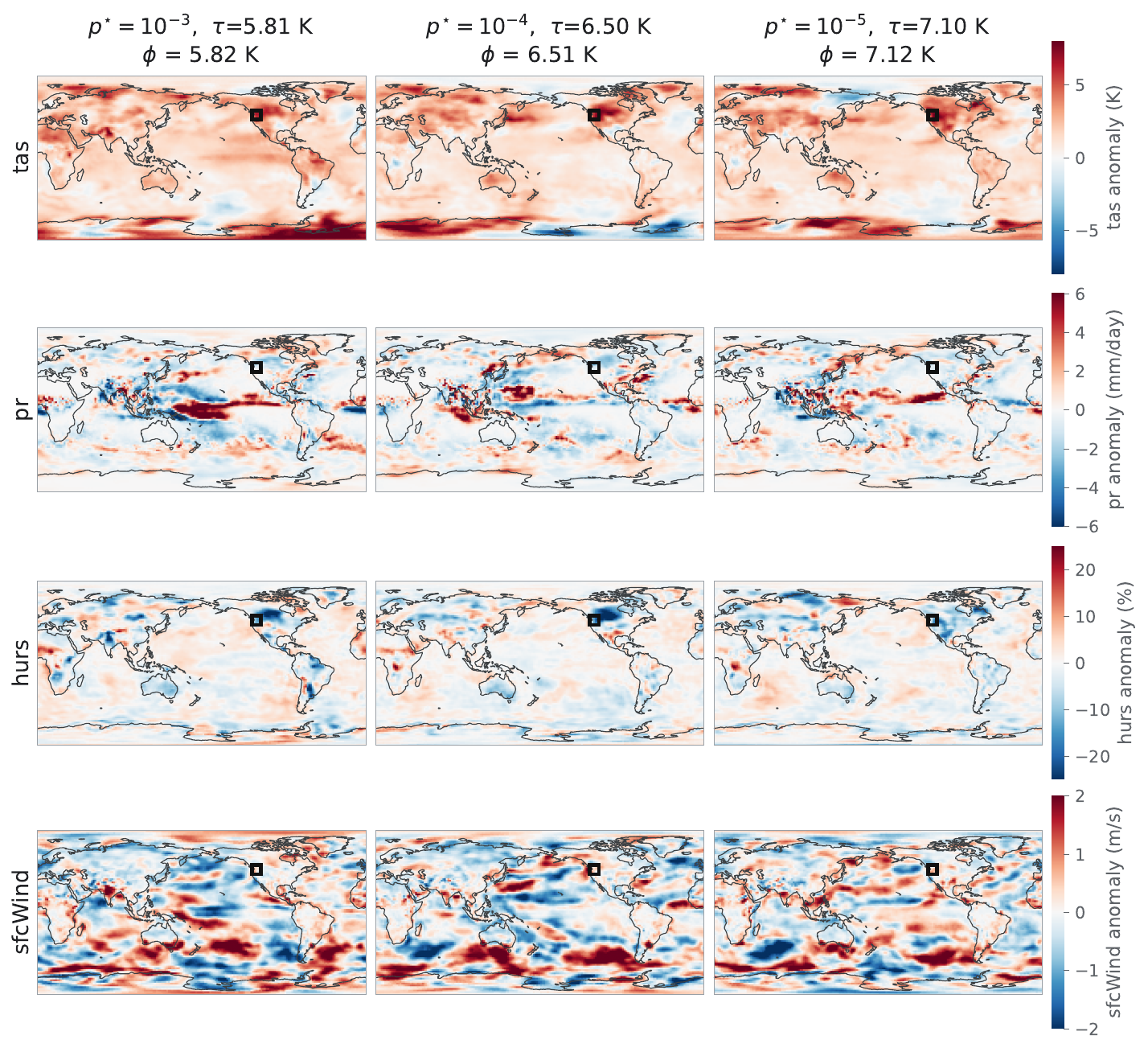}
    \caption{The same fields as Fig.~\ref{fig:climate_plot} on the global domain.}
    \label{fig:climate_plot_diff}
\end{figure}

\FloatBarrier

\subsubsection{Estimating the marginals and the joint}
\label{sec:estimating-marginals-joint}

We expand on the multivariate hazard experiment of Section~\ref{sec:exp_climate}. $d$-values are computed using three guided runs for the guided process and read out as a fraction of joint exceedance $p(\text{hot \& dry})$ from the Monte Carlo reference. For the joint guided runs we ablate between two variants for the CDF reward, one with a joint 2 dimensional CDF across the two variables $\Phi( \texttt{hurs}, \texttt{tas})$ and another as a product of two individual CDFs $\Phi( \texttt{hurs})\Phi(\texttt{tas})$. We replace the $\Pi$GDM's isotropic covariance estimate, $\text{Cov}(x_0 | x_t)$ with the model's by computing the Jacobian Vector Product of $a^\top \Sig a$ \cite{boys2023tweedie}, additionally we use a joint 2-dimensional bandwidth-inflated CDF $\Phi_2(\cdot)$ (Eq.~\eqref{eq:lookahead_multivariate}) as the reward estimated using Drezner's algorithm with a 16-point Gauss-Legendre quadrature \cite{drezner1990computation}.

Table~\ref{tab:dependence_scaling} summarises the probability estimate obtained via the Monte Carlo reference which we compare against. In this setting we define $p$ to be the probability a marginal exceeds a specific threshold, with the joint probability given by $p^2$ under an independence assumption.

\begin{table}[h]
  \centering
  \small
  \caption{Monte Carlo estimate and dependence for the compound hot-dry event at matched marginal quantiles. $p$ represents the compound hot-dry event at
  matched marginal quantiles: at each level $p$ both thresholds are the pool's
  own $p$-quantiles, so $p_{\rm indep} = p^2$ exactly; quoted errors are
  binomial, and the $10^{-4}$ row rests on $125$ joint crossers. }
  \label{tab:dependence_scaling}
  \begin{tabular}{c c c l l r}
    \toprule
    marginal $p$ & $\tau_{\mathrm{tas}}$ (K) & $\tau_{\mathrm{hurs}}$ (\%) & $p_{\mathrm{joint}}$ & $p_{\mathrm{indep}}$ & $d^\star$ \\
    \midrule
    $10^{-1}$ & 3.8908 & $-8.6085$  & $(5.028 \pm 0.007) \times 10^{-2}$ & $1.00 \times 10^{-2}$ & $5.03 \pm 0.01$ \\
$10^{-2}$ & 4.9907 & $-12.7822$ & $(3.138 \pm 0.019) \times 10^{-3}$ & $1.00 \times 10^{-4}$ & $31.4 \pm 0.2$ \\
$10^{-3}$ & 5.8128 & $-15.8382$ & $(2.121 \pm 0.049) \times 10^{-4}$ & $1.00 \times 10^{-6}$ & $212.0 \pm 4.9$ \\
$10^{-4}$ & 6.4977 & $-18.3871$ & $(1.387 \pm 0.124) \times 10^{-5}$ & $1.00 \times 10^{-8}$ & $1384 \pm 124$ \\
    \bottomrule
  \end{tabular}
  
\end{table}

\FloatBarrier

\subsubsection{Factorised versus multivariate guidance}
\label{app:climemu_joint_guidance}
We ablate the two reward variants - factorised (Eq.~\eqref{eq:factorised_lookahead}) and joint (Eq.~\eqref{eq:lookahead_multivariate}) of the compound weather event of  Section~\ref{sec:exp_climate}. Across rarities, we compute the guidance based on a factorised reward (i.e. a product of two independent 1-dimensional CDFs) and its correlated $\Phi_2$ variant. Both variants yield a different guidance drift as well as look-ahead reward $\rsoft$. Additionally, for the Gaussian surrogate of $p(x_0 | x_t) \approx N(\xh, \Sigma)$, where $\xh$ is Tweedie's estimate, we replace the isotropic $\Sigma$ approximation of \cite{song2023pseudoinverseguided}, with a JVP estimate of $\sigma^2A^TJA$, where $A$ is defined by the linear function, $J$ is the Jacobian of the network and $\sigma^2$ is the current noise level (Appendix~\ref{app:multiple_guidances}). The joint estimator costs roughly $1.64 \times$ wall-clock time (Appendix~\ref{app:climemu_cost}) of the factorised but yields a lower variance of $\pE$ at the rarest level, which we discuss below.

The experiment was conducted with identical seeds and settings as the main text and compared against the MC reference. We use the same 16 seeds across both variants.
\begin{table}[!h]
\centering
\small
\setlength{\tabcolsep}{6pt}
\caption{Comparison of two reward estimates for compound hot \& dry events. The difference between the two rewards is most prominent at the rarest level. Notably the $10^{-4}$ reference itself carries an
$8.9\%$ standard error. Setup remains identical to Table~\ref{tab:joint_model_exp}, MC reference:
Table~\ref{tab:dependence_scaling}.}
\label{tab:climate_exp2_full}
\begin{tabular}{@{}l rrr rrr@{}}
\toprule
 & \multicolumn{3}{c}{Factorised reward}
 & \multicolumn{3}{c}{Joint $\Phi_2$ reward} \\
\cmidrule(lr){2-4}\cmidrule(lr){5-7}
marginal $p$ & $\hat p/p^\star$ & CV $\downarrow$ & $\hat d/d^\star$
 & $\hat p/p^\star$ & CV $\downarrow$ & $\hat d/d^\star$ \\
\midrule
$10^{-2}$ & $\textbf{0.990} \pm \textbf{0.024}$ & $0.097$ & $1.07 \pm 0.03$
 & $0.965 \pm 0.016$ & $\textbf{0.068}$ & $\textbf{1.04} \pm \textbf{0.02}$ \\
$10^{-3}$ & $0.952 \pm 0.031$ & $0.132$ & $0.98 \pm 0.03$
 & $\textbf{0.990} \pm \textbf{0.028}$ & $\textbf{0.113}$ & $1.02 \pm 0.03$ \\
$10^{-4}$ & $1.124 \pm 0.087$ & $0.310$ & $1.12 \pm 0.08$
 & $\textbf{1.044} \pm \textbf{0.021}$ & $\textbf{0.080}$ & $\textbf{1.06} \pm \textbf{0.06}$ \\
\bottomrule
\end{tabular}

\end{table}
\FloatBarrier
The joint reward has a lower CV at every level, and the difference is most prominent at the rarest level: at $10^{-4}$ the factorised reward has CV $0.310$ against $0.080$ (variance ratio $14.8$, $p < 0.001$) and its $\hat d/d^\star$ is further from one ($1.12$ against $1.06$). At $10^{-2}$ and $10^{-3}$ the two are statistically indistinguishable, and the cheaper factorised reward reaches a similar net speed-up (Table~\ref{tab:compute_cost}). As $\texttt{tas}$ and $\texttt{hurs}$ are strongly correlated ($\rho \approx 0.72$), a product of marginal rewards undersamples the correlated region where $p(x_i)/q(x_i) > 1$; the few particles that reach it carry large weights, which inflates the variance at deeper rarities.

\subsubsection{Conditional fields for compound events}
\label{app:climemu_joint_fields}
\paragraph{The compound event as a field.}
Figure~\ref{fig:joint_hot_dry} contrasts the conditional mean field under the
hot-only constraint with the one under the joint hot \& dry constraint of
Table~\ref{tab:joint_model_exp}. Guiding on the temperature alone, reduces the humidity across the region, illustrating why the dependency ratio $d > 1$.

\begin{figure}[tbp]
    \centering
    \includegraphics[width=1.0\linewidth]{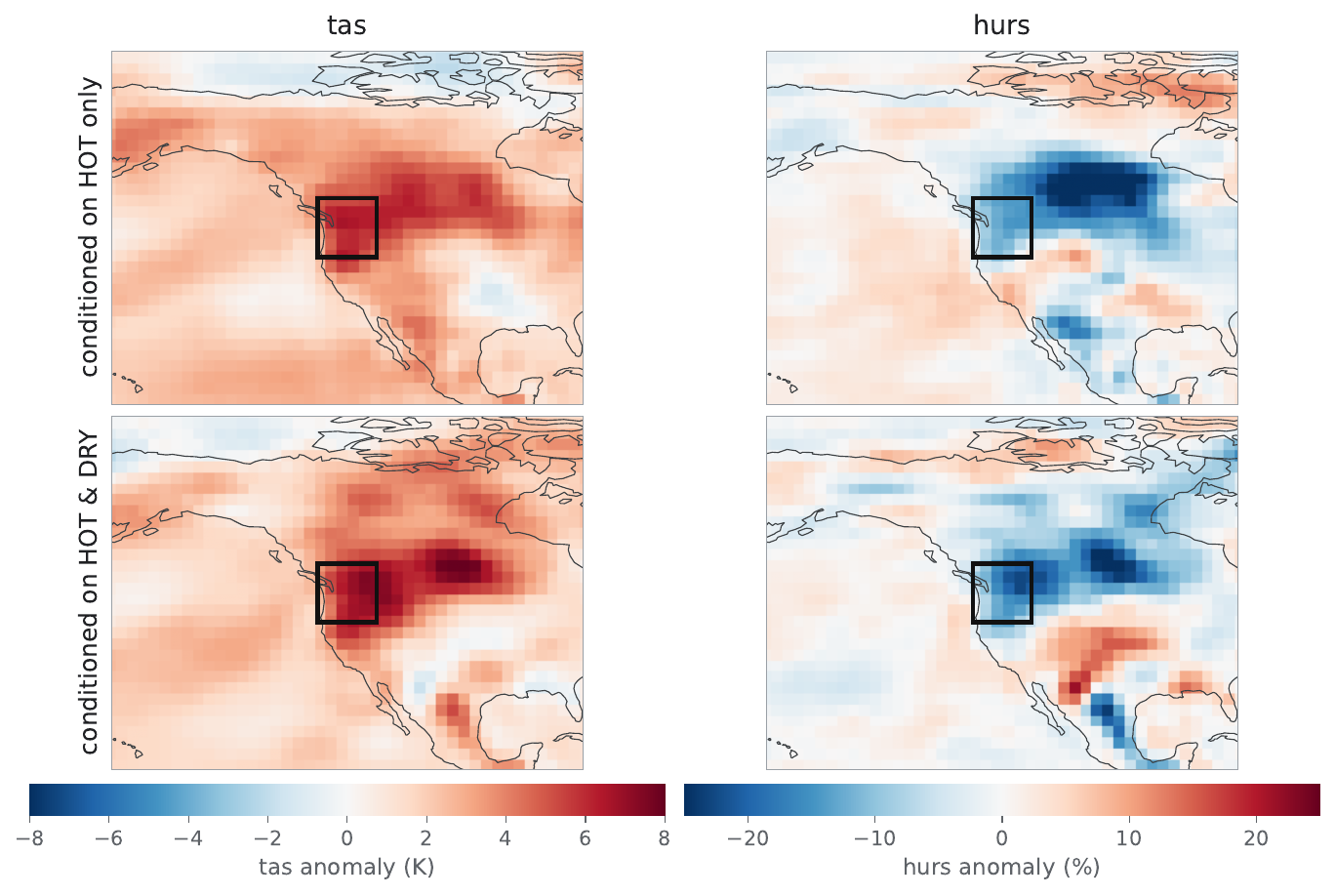}
    \caption{Conditional mean fields for the compound event of Table~\ref{tab:joint_model_exp}, over the PNW box in July at GMST $+1.5$~K.
    \emph{Top:} conditioned on the hot constraint alone. \emph{Bottom:}
    conditioned jointly on hot \& dry. \emph{Left:} near-surface temperature.
    \emph{Right:} relative humidity. Both constraints are at marginal
    $10^{-3}$; the box is the region of interest. Guiding on temperature alone reduces the humidity across a region.}
    \label{fig:joint_hot_dry}
\end{figure}

\FloatBarrier

\subsection{Experiment 2: A changing climate}
\label{app:climemu_warming}
\begin{table}[!h]
\centering
\small
\caption{Warming sweep at the frozen event $\tau = 5.81$~K over the region computed with 16 seeds per point. $p^*$  is obtained from a large Monte Carlo references, which exist only at
$\Delta T \ge 1.5$~K (relative errors $1.1\%$, $3.2\%$, $0.7\%$, $0.13\%$)
and the guided estimates match all four. }%
\label{tab:climate_exp3}
\begin{tabular}{cccccc}
\toprule
$\Delta T$ (K) & $\hat p$ & s.e. $\downarrow$ & CV $\downarrow$ & $p^\star$ & $\hat p/p^\star$ \\%
\midrule
1.5 & $9.79\!\times\!10^{-4}$ & $\pm0.18\!\times\!10^{-4}$ & 0.073 & $1.000\!\times\!10^{-3}$ & $0.978$ \\
2.0 & $1.51\!\times\!10^{-2}$ & $\pm0.03\!\times\!10^{-2}$ & 0.073 & $1.485\!\times\!10^{-2}$ & $1.016$ \\
3.0 & $3.93\!\times\!10^{-1}$ & $\pm0.02\!\times\!10^{-1}$ & 0.025 & $3.932\!\times\!10^{-1}$ & $1.000$ \\
4.0 & $9.51\!\times\!10^{-1}$ & $\pm0.02\!\times\!10^{-1}$ & 0.006 & $9.503\!\times\!10^{-1}$ & $1.001$ \\
\bottomrule
\end{tabular}

\end{table}
Table~\ref{tab:climate_exp3} reports the sweep at a frozen threshold against the large Monte Carlo reference, the generated fields behind it are in Fig.~\ref{fig:climate_change_rare_event_dT}. As the global mean surface temperature increases, the degree of warming across regions varies per region which is expected under the modelling condition of \cite{https://doi.org/10.1029/2025MS005558}. As such, the regional mean temperature is greater than the GMST which is reflected at $\Delta T = 4.0$K, where the estimated $\pE = 0.95$. Table~\ref{tab:climate_exp3_sweep} reports the performance of DireSMC compared against a compute normalised MC estimate. Due to computational constraints we could not validate the estimates at $\Delta T \le 1.2$~K; to obtain a MC probability estimate at a $5\%$ relative error we would require on the order of $10^8/10^7/10^6$ draws from the unconditional model at $\Delta T = 0.0/1.0/1.2$~K respectively.

\begin{table}[!h]
\centering
\small
\caption{Warming sweep at the frozen event $\tau = 5.81$~K over the region computed with 16 seeds per point. We compare against an unguided MC estimate at matched clock time allocating it $4.7\times$ particles, over 16 replicates. Zero counts the number of MC replicates that record no event. See Table~\ref{tab:climate_exp3} for DireSMC validated against a large MC pool. Var. Ratio = $\mathrm{Var}(\hat p_{\rm MC})/\mathrm{Var}(\hat p)$.}
\label{tab:climate_exp3_sweep}
\begin{tabular}{ccccccc}
\toprule
 & \multicolumn{2}{c}{DireSMC} & \multicolumn{3}{c}{Crude MC} & \\
\cmidrule(lr){2-3}\cmidrule(lr){4-6}
$\Delta T$ (K) & $\hat p$ & CV & $\hat p_{\rm MC}$ & CV & Zero & Var. Ratio \\
\midrule
0.0 & $1.06\!\times\!10^{-6}$ & 0.402 & $0$ & --- & 16/16 & N/A \\
0.5 & $1.65\!\times\!10^{-6}$ & 0.214 & $0$ & --- & 16/16 & N/A \\
1.0 & $3.51\!\times\!10^{-5}$ & 0.092 & $1.30\!\times\!10^{-5}$ & 4.000 & 15/16 & $257$ \\
1.2 & $1.46\!\times\!10^{-4}$ & 0.140 & $1.04\!\times\!10^{-4}$ & 1.789 & 11/16 & $83$ \\
1.5 & $9.79\!\times\!10^{-4}$ & 0.073 & $1.03\!\times\!10^{-3}$ & 0.594 & 0/16 & $73$ \\
2.0 & $1.51\!\times\!10^{-2}$ & 0.073 & $1.66\!\times\!10^{-2}$ & 0.138 & 0/16 & $4.4$ \\
2.5 & $1.10\!\times\!10^{-1}$ & 0.049 & $1.14\!\times\!10^{-1}$ & 0.049 & 0/16 & $1.05$ \\
3.0 & $3.93\!\times\!10^{-1}$ & 0.025 & $3.96\!\times\!10^{-1}$ & 0.015 & 0/16 & $0.39$ \\
3.5 & $7.52\!\times\!10^{-1}$ & 0.012 & $7.53\!\times\!10^{-1}$ & 0.010 & 0/16 & $0.65$ \\
4.0 & $9.51\!\times\!10^{-1}$ & 0.006 & $9.50\!\times\!10^{-1}$ & 0.004 & 0/16 & $0.33$ \\
\bottomrule
\end{tabular}
\end{table}

\begin{figure}[!h]
    \centering
    \includegraphics[width=0.95\linewidth]{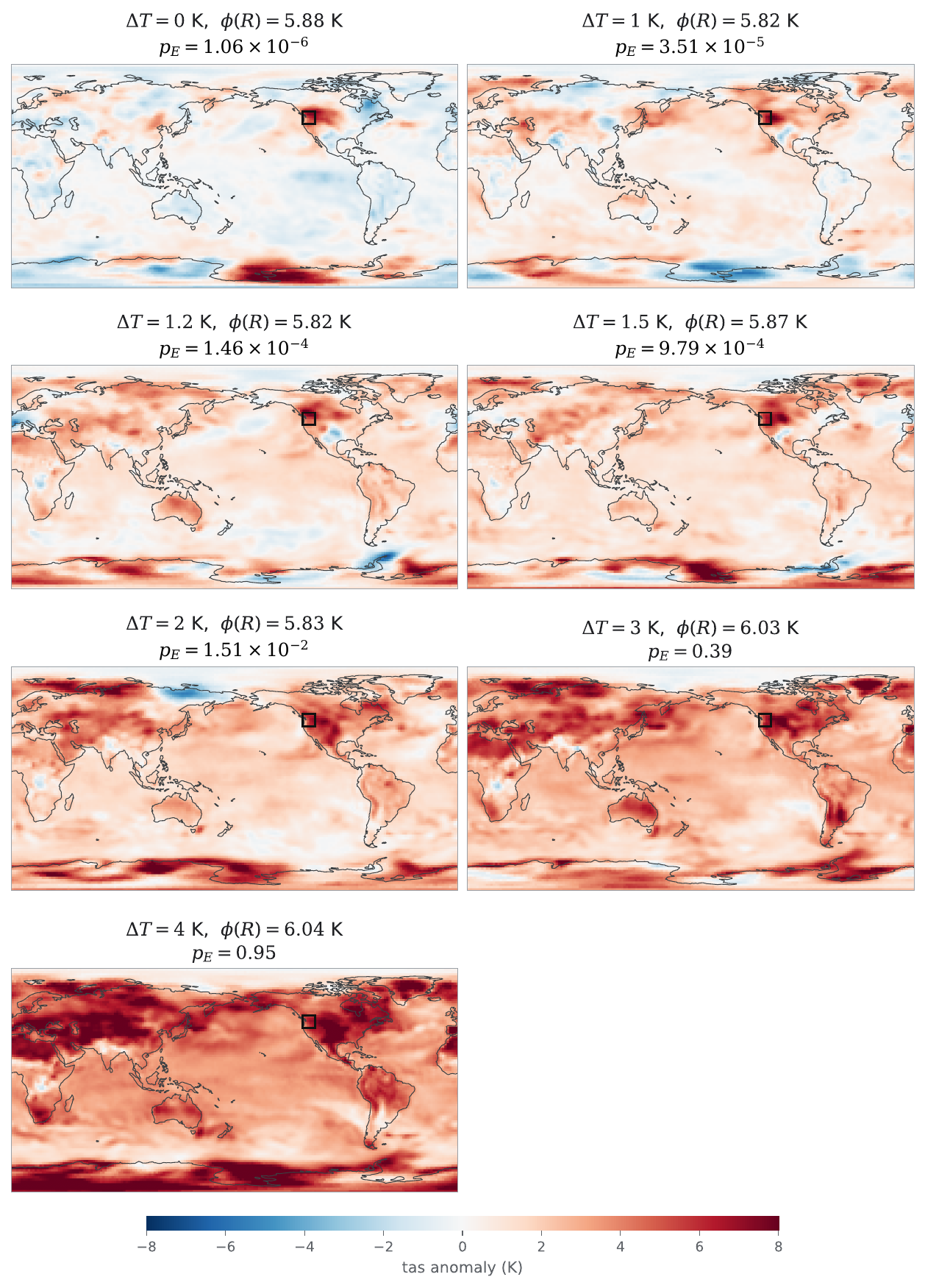}
   \caption{Highest weight samples under the no resampling regime.  We guide the sampler to the same $\tau_{\texttt{tas}} \ge 5.81$K for a range of global average temperatures $\Delta T$. See Table~\ref{tab:climate_exp3} for the qualitative analysis. }
\label{fig:climate_change_rare_event_dT}
\end{figure}

\FloatBarrier

\subsection{Computational cost}
\label{app:climemu_cost}
We measure the net-speed-up of the path-weighted estimator by comparing the net wall-clock time required for each example in 
Table~\ref{tab:compute_cost}. We compare our results against the number of $N_{MC}$ samples required to form a Monte Carlo estimate at equal-CV, discounted by the added cost of computing the guidance. An unguided draw is one denoiser call per step ($0.151$~s per sample at
$K = 200$); the production sampler measures
$4.7\times$ that per sample for a univariate or factorised reward (marginal threshold exceedances and multivariate exceedance under uncorrelated reward), and the correlated variant of the reward measures a larger
$7.7\times$. In general, any $m$-coupled reward requires $m$ vector Jacobian product calls per denoising step, which contributes a linearly growing computational cost to the wall-clock time.

\begin{table}[h]
\centering
\small
\setlength{\tabcolsep}{8pt}
\caption{Net speed-up over crude Monte Carlo at matched CV for $N = 1024$ guided draws against unguided draws. Computed using CVs obtained from
Tables~\ref{tab:marginal_hot_dry}, \ref{tab:joint_model_exp} and~\ref{tab:climate_exp2_full}. The dry $10^{-4}$ entry is low because of one collapsed run (App.~\ref{app:climemu_resampling_ablation}); without it the CV is $0.124$, close to the other $10^{-4}$ cells.}
\label{tab:compute_cost}
\begin{tabular}{@{}l rr rr@{}}
\toprule
 & \multicolumn{2}{c}{Threshold Exceedances} & \multicolumn{2}{c}{Joint hot \& dry} \\
\cmidrule(lr){2-3}\cmidrule(lr){4-5}
$p^*$ & hot & dry & factorised reward & Joint reward \\
\midrule
Per-sample cost & $4.7\times$ & $4.7\times$ & $4.7\times$ & $7.7\times$ \\
\midrule
\addlinespace[4pt]
$10^{-1}$ & $0.7\times$  & $1.0\times$  & $1.3\times$  & $1.2\times$ \\
$10^{-2}$ & $8.7\times$  & $3.9\times$  & $7.0\times$  & $8.6\times$ \\
$10^{-3}$ & $34\times$   & $40\times$   & $56\times$   & $47\times$ \\
$10^{-4}$ & $311\times$  & $44\times$   & $156\times$  & $1413\times$ \\
$10^{-5}$ & $1246\times$ & $1655\times$ & ---          & --- \\
\bottomrule
\end{tabular}

\end{table}
We find that for most of the generations, the net speedup of the path-weighted sampler out-weighs the added clock time for guided steps. Additionally, at deeper rarities where Monte Carlo has zero samples within the extreme event set of interest (e.g. for very rare events), the method becomes a point of feasibility rather than speed-up alone. These speed-ups scale with the quality of the guidance itself.

\subsection{Hyperparameter ablation}

We test a range of schedules of $(\lambda(\bar{\sigma}), \delta(\bar{\sigma}))$ on the threshold exceedance problem at a marginal $p^*  = 10^{-3}$. We caveat that the parameters we present were chosen heuristically, and we do not independently tune the hyperparameters for any other setting presented in our paper. While we tune on the marginal hot and dry $\tau_1$ ($p^* = 10^{-3}$) problems, we believe there are further performance improvements that can be made by better schedule tuning. All ablations are done with a single change to the default parameters described in App.~\ref{app:climemu_impl}. We report $\hpE/p^*$ against the Monte Carlo anchor, with errors displayed as the standard errors over the 16 seeds, which are shared across all arms.
\begin{table}[t]
  \centering
  \setlength{\tabcolsep}{4pt}
  \caption{One-knob ablation of the guidance schedules at $p^\ast = 10^{-3}$. We vary a single parameter reference in Table~\ref{tab:climemu-hparams}; the same $16$ seeds are used in every row. Collapsed runs are included: one hot run at $\lambda_{\max} = 2.0$, and two hot and two dry runs at $\delta_{\mathrm{hi}} = 0.25\,\mathrm{SD}$.}
    \label{tab:climemu-recipe-schedule}

  \begin{tabular}{lll rr rr}

    \toprule
    & & & \multicolumn{2}{c}{Hot (\texttt{tas})} & \multicolumn{2}{c}{Dry (\texttt{hurs})} \\
    \cmidrule(lr){4-5}\cmidrule(lr){6-7}
    Knob & Deployed & Ablated & $\hat p/p^\ast$ & CV & $\hat p/p^\ast$ & CV \\
    \midrule
    \emph{Deployed recipe} & & & $0.982 \pm 0.019$ & $0.078$ & $0.996 \pm 0.018$ & $0.072$ \\
    \midrule
    \multicolumn{7}{l}{Tempering schedule $\lambda(\bar\sigma)$}\\
    Peak $\lambda_{\max}$        & $1.25$ & $1.0$ (no overshoot) & $1.010 \pm 0.045$ & $0.180$ & $0.984 \pm 0.023$ & $0.095$ \\
    Peak $\lambda_{\max}$        & $1.25$ & $1.5$                & $0.984 \pm 0.018$ & $0.071$ & $0.984 \pm 0.022$ & $0.090$ \\
    Peak $\lambda_{\max}$        & $1.25$ & $2.0$                & $1.025 \pm 0.051$ & $0.199$ & $0.942 \pm 0.037$ & $0.159$ \\
    Onset $\bar\sigma_{\mathrm{on}}$ & $40$ & $20$ (later)      & $0.937 \pm 0.025$ & $0.106$ & $0.969 \pm 0.034$ & $0.139$ \\
    Onset $\bar\sigma_{\mathrm{on}}$ & $40$ & $80$ (earlier)    & $0.989 \pm 0.022$ & $0.088$ & $0.967 \pm 0.023$ & $0.094$ \\
    \midrule
    \multicolumn{7}{l}{Smoothing width $\delta(\bar\sigma)$}\\
    Transport $\delta_{\mathrm{hi}}$ & $0.5\,\mathrm{SD}$  & $0.25\,\mathrm{SD}$ & $1.088 \pm 0.112$ & $0.411$ & $0.777 \pm 0.084$ & $0.431$ \\
    Floor $\delta_{\mathrm{lo}}$     & $0.01\,\mathrm{SD}$ & $0.05\,\mathrm{SD}$ & $1.009 \pm 0.019$ & $0.074$ & $1.023 \pm 0.014$ & $0.055$ \\
    \bottomrule
  \end{tabular}
\end{table}
\FloatBarrier

\section{Metropolis-adjusted Langevin corrections}
\label{sec:madm_details}
 
We use Metropolis-Adjusted Diffusion Model (MADM, \cite{lam2026metropolis}) for the rejuvenation kernel $M_k$ of Section~\ref{sec:sequential}. 
From the predictor-corrector framework of \cite{song2021scorebasedgenerativemodelingstochastic} the guided score can be used as a predictor and $M$ MADM steps act as our corrector, targeting the tilted marginals $\pi_k$ of Eq.~\eqref{eq:target_sequence}.
The unadjusted Langevin algorithm (ULA) update is,
\begin{align}
\label{eq:mala_proposal}
    \tilde{x} \;=\; x + \tfrac{\rho}{2}\, g(x) + \sqrt{\rho}\, z, && z \sim \N(0, I), && g(x) := \nabla_x \log \pi_k(x) = \sth(x, t_k) + \hg{t_k}(x).
\end{align}
Here $g$ is the guided score. 
ULA accepts all proposals $x \rightarrow \tilde{x}$ and is biased for every finite step size $\rho$, and only leaves $\pi_k$ invariant in the $\rho \to 0$ limit.

A Metropolis-Hastings correction removes this bias while keeping the same proposal: accept $\tilde x$ with probability $\mathcal{A}(x, \tilde{x}) = 1 \wedge \big[ \Gamma_\pi(x, \tilde{x})\, H(x, \tilde{x}) \big]$, otherwise stay at $x$, where the acceptance ratio splits into a target ratio $\Gamma_\pi$ and a proposal ratio $H$. 
The proposal ratio is Gaussian and closed-form,
\begin{equation}
\label{eq:madm_proposal}
    \log H(x, \tilde{x}) \;=\; \frac{1}{2\rho} \big\| \tilde{x} - x - \tfrac{\rho}{2} g(x) \big\|^2 \;-\; \frac{1}{2\rho} \big\| x - \tilde{x} - \tfrac{\rho}{2} g(\tilde{x}) \big\|^2,
\end{equation}
while the target ratio is exactly the tilted-density ratio at $t_k$,
\begin{equation}
\label{app:madm_gamma_main}
    \Gamma_\pi(x, \tilde{x}) \;=\; \frac{\pi_k(\tilde{x})}{\pi_k(x)} \;\propto\; \frac{p_{t_k}(\tilde{x})\, e^{r_{t_k}(\tilde{x})}}{p_{t_k}(x)\, e^{r_{t_k}(x)}}.
\end{equation}
Only $\Gamma_\pi$ remains intractable for a diffusion model: the reward factor $e^{r_{t_k}}$ is available in closed form, but we do not have direct access to the base marginal $p_{t_k}$, only its score $\sth(x, t_k)$. Writing the target ratio as
\begin{equation}
    \log \Gamma_\pi(x, \tilde{x}) \;=\; \underbrace{\log\frac{p_{t_k}(\tilde{x})}{p_{t_k}(x)}}_{\text{intractable}} \;+\; \underbrace{\big(r_{t_k}(\tilde{x}) - r_{t_k}(x)\big)}_{\text{exact}}
\end{equation}
isolates the one hard piece: the reward term is known in closed form and needs no approximation at all. MADM \cite{lam2026metropolis} sidesteps the marginal-density problem by computing the remaining $p_{t_k}$-ratio, under a well trained network where $\sth = \nabla_x \log p_{t_k}$, using a line-integral,
\begin{equation}
\label{eq:madm_line_integral}
    \log \frac{p_{t_k}(\tilde{x})}{p_{t_k}(x)} \;=\; \int_0^1 \big\langle\, \sth\big(x + u(\tilde{x} - x),\, t_k\big),\; \tilde{x} - x \,\big\rangle\, \dd u,
\end{equation}
valid when $\sth$ is a gradient field, which holds exactly for an analytic score but is not enforced by the denoising objective for a learned network.
This is the conservation error of the learned score, which is small in high noise regions, where we apply rejuvenation. 
The line integral can be approximated by quadrature, for instance by Simpson's rule at the midpoint $\bar{x} = (x + \tilde{x})/2$,
\begin{equation}
\label{eq:madm_simpson}
    \widehat{\log \Gamma_\pi} \;=\; \tfrac{1}{6} \big\langle\, \sth(x,t_k) + 4\, \sth(\bar{x},t_k) + \sth(\tilde{x},t_k),\; \tilde{x} - x \,\big\rangle \;+\; \big(r_{t_k}(\tilde{x}) - r_{t_k}(x)\big),
\end{equation}
giving the practical acceptance probability $\widehat{\mathcal{A}}(x,\tilde x) = 1 \wedge \exp\big(\widehat{\log \Gamma_\pi}\big) H(x, \tilde{x})$. 
The endpoint scores $\sth(x,t_k), \sth(\tilde x,t_k)$ and rewards $r_{t_k}(x), r_{t_k}(\tilde x)$ are shared with the proposal and its acceptance test, so the quadrature costs additional \emph{unconditional} score evaluations.

The quadrature carries a truncation error of order $\mathcal{O}(\rho^{5/2})$ on the $p_{t_k}$-ratio, so the adjusted kernel is not exactly $\pi_k$-invariant but is accurate to far higher order than the $\mathcal{O}(\rho)$ bias of the unadjusted chain \cite{lam2026metropolis}.
An exact alternative replaces the Metropolis rule with Barker acceptance, at the cost of additional score evaluations \cite{lam2026metropolis}; but we use the Simpson-rule variant throughout.

We can summarise a single MADM move under the Simpson's-rule variant of \cite{lam2026metropolis} as (i) propose a transition with
Eq.~\eqref{eq:mala_proposal} and (ii) accept the proposal with an acceptance probability approximated with Eqs.~\eqref{eq:madm_proposal} and \eqref{eq:madm_simpson}.
Algorithm~\ref{alg:madm} states this as the rejuvenation kernel $M_k$ called by Algorithm~\ref{alg:rne_ais}.

\begin{algorithm}
\caption{One MADM rejuvenation move at level $k$, targeting $\pi_k \propto p_{t_k} e^{r_{t_k}}$.}
\label{alg:madm}
\begin{algorithmic}[1]
\Require state $x$, level $k$, score $\sth(\cdot,t_k)$, reward $r_{t_k}$, step size $\rho$.
\State $g(x) \gets \sth(x,t_k) + \nabla_x r_{t_k}(x)$ \Comment{Guided score}
\State $\tilde{x} \gets x + \tfrac{\rho}{2}\,g(x) + \sqrt{\rho}\,z$, \quad $z\sim\N(0,I)$ \Comment{Langevin proposal, Eq.~\eqref{eq:mala_proposal}}
\State $\log\widehat{\mathcal{A}} \gets \min\big\{0,\; \widehat{\log\Gamma_\pi}(x,\tilde{x}) + \log H(x,\tilde{x})\big\}$ \Comment{Eqs.~\eqref{eq:madm_simpson} and \eqref{eq:madm_proposal}}
\State With probability $\widehat{\mathcal{A}}$ set $x \gets \tilde{x}$; otherwise keep $x$.
\State \Return $x$
\end{algorithmic}
\end{algorithm}
\FloatBarrier

\section{Related Works Extended}
\label{app:manshausen_comparison}
We compare our method for generating the rare event probabilities against the method introduced in \cite{manshausen2026accurateextremeeventlikelihoods}. 
Both methods guide a pre-trained diffusion model towards the extremes, and reweight the resulting samples to estimate $\pE$ under the model density $p_\theta$. 
The methods differ at the integrator.
The method of \cite{manshausen2026accurateextremeeventlikelihoods} samples with an ODE and estimates the rare-event probability using the Probability Flow ODE (PF-ODE) \cite{song2021scorebasedgenerativemodelingstochastic} for both unguided and guided processes, which allows them to weight each sample by its exact terminal likelihood.  
We sample with a reverse SDE and weight each trajectory by the Girsanov density of the guided path measure, computed against the unguided process. Neither object is available to the other: a stochastic integrator has no PF-ODE likelihood, and a deterministic flow induces path measures that are mutually singular, so Girsanov's theorem is not available.

Resampling mid-trajectory is only available to the path-weights formulation as a terminal-only weight defines no intermediate targets to resample against. During the denoising process, we measure how far the sampled trajectory $\Qm$ deviates from one defined by the tilted process $\Pmr$, and adaptively resample within a Sequential Monte Carlo scheme. Under the operating guidance of \cite{manshausen2026accurateextremeeventlikelihoods}, this yields us a far larger fraction of particles from the rare-event set $E$.

\paragraph{Notation.} The \emph{arrow} on a density is the direction in which it is evaluated:
backward ($\leftarrow$) means along the sampling run itself
(i.e. from a Gaussian prior $\to$ the data distribution); forward ($\rightarrow$) is the noising direction (i.e. data $\to$ prior) obtained via integrating the ODE
from the generated sample back to noise. $K$ is the number of
integrator steps; arrow-free symbols ($\pz$, $\qz$) are the continuous-time
objects, untouched by the discretisation.  The key point is that, at a finite $K$, the two directions assign \emph{different} values to the same density (dependent on the integrator of choice).
\begin{center}
\small
\begin{tabular}{@{}l p{0.72\linewidth}@{}}
\toprule
$\pz,\ \qz$ & model's true density and its guided (tilted) counterpart; target $\pE=\int_{E}\pz$\\
$\qbwd$ & discretised density of the guided draws $x \sim \qbwd$.\\
$\qfwd$ & the same guided density evaluated by the forward ODE run - their denominator\\
$\pbwd$ & discretised density of the \emph{unguided} process - our estimator targets this quantity\\
$\pfwd$ & the unguided density evaluated by a forward ODE - their numerator\\
$\Ez=\log(\pfwd/\pbwd)$ & \emph{arrow gap} of the unguided density: forward and backward evaluation of the same process.\\
$\Eg=\log(\qfwd/\qbwd)$ & arrow gap of the guided density\\
$\Dmap=\log(\pbwd/\pz)$ & the discrete chain's gap to the model\\
\bottomrule
\end{tabular}
\end{center}
The three discrepancies $\Ez,\Eg,\Dmap$ are all $O(1/K)$ discretisation
quantities and vanish as $K\to\infty$.

\subsection{The PF-ODE Estimator}
\label{app:manshausen_method}

Both methods target the same quantity, $\pE$ by importance sampling with a guided proposal: draw $x\sim\qbwd$, then reweight,
\begin{equation}
\label{eq:is_identity_pfode}
    \pE \;=\; \int \pz\,\inds{E
}\,\dd x
    \;=\; \E_{x\sim\qbwd}\!\left[\frac{\pz}{\qbwd}\,\inds{E}\right].
\end{equation}
The two methods differ only in how this ratio is estimated. 
We compute the Girsanov path ratio $\dd\Pm/\dd\Qm$ using transition kernels which is accumulated during the run itself. It is a functional of the whole trajectory $X_{[0,1]}$, not of its endpoint alone.
The method of \cite{manshausen2026accurateextremeeventlikelihoods} is a ratio of two \emph{forward-evaluated} endpoint densities:
\begin{equation}
\label{eq:ours_vs_theirs}
    \text{ours: } \E_{X_{[0,1]}\sim\Qm}\!\Big[\tfrac{\dd\Pm}{\dd\Qm}(X_{[0,1]})\,\inds{E}(X_0)\Big],
    \qquad
    \text{\cite{manshausen2026accurateextremeeventlikelihoods} :  } \E_{x\sim\qbwd}\!\big[\tfrac{\pfwd}{\qfwd}(x)\,\inds{E}(x)\big].
\end{equation}
Both estimators target the integral of Eq.~\eqref{eq:is_identity_pfode}, ours by the tower property in conditional expectations, $\E_\Qm\big[\dd\Pm/\dd\Qm \,\big|\, X_0\big] = \pbwd(X_0)/\qbwd(X_0)$ and theirs by the ratio $\pfwd/\qfwd \approx \pz/\qbwd$. Our path weight uses the backward or denoising process, simulated under the guided chain. Note that importance sampling asks us to divide by the density of the draws, i.e. the \emph{backward} object $\qbwd$ instead of $\qfwd$. We note that both estimators need not have the same variance \cite{del2006sequential, chopin2020introduction}.

The method of \cite{manshausen2026accurateextremeeventlikelihoods} first generates a sample by integrating a deterministic ODE from $t=1$ down to data at $t=0$,
\begin{equation}
    x_0 \;=\; x_1 \;-\; \int_0^1
    \Big(f_t\,x_t \;-\; \tfrac{1}{2}\sigma_t^2
    \big(\sth(x_t,t) + \gamma\,h_t(x_t) \ind{t \in \Delta}\big)\Big)\,\dd t,
\end{equation}
where $h_t$ is their classifier guidance term and $\gamma$ is their guidance strength. 
After sampling, samples that belong to the rare event set of interest $x_E \sim \qbwd \in E$ are selected and their likelihoods are evaluated: from each sample $x_E$, they re-integrate the same dynamics in the noising direction $t = 0\to1$, accumulating the integral,
\begin{equation}
\label{eq:cnf_identity}
    \log\pfwd(x_E) \;=\; \log p_1(x_1)
    \;+\; \int_0^1 \Big(d\,f_t
    \;-\; \tfrac{1}{2}\sigma_t^2\,\nabla\!\cdot \sth(x_t,t)\Big)\,\dd t,
\end{equation}
where $d$ is the dimension and $x_1$ the endpoint of the re-integration (noise), $\log \qfwd$ is obtained similarly via replacing the score in Eq.~\ref{eq:cnf_identity} with its guided variant, $\sth \to \sth + \gamma h_t(x_t)\ind{t \in \Delta}$.
Their importance weight is given by,
\begin{equation}
\label{eq:their_logweight}
\begin{split}
    \log\frac{p_1(x_1)}{p_1(x_1^{\gamma})} 
    - \frac{1}{2}\int_0^1 \sigma_t^2\,
    \Big(\nabla\!\cdot \sth(x_t,t)
    \;-\; \nabla\!\cdot \sth(x_t^{\gamma},t)\Big)\,\dd t 
    + \frac{\gamma}{2}\int_{\Delta} \sigma_t^2\,
    \nabla\!\cdot h_t(x_t^{\gamma})\,\dd t,
\end{split}
\end{equation}
where $x_t$ and $x_t^{\gamma}$ are the unguided and guided re-integration trajectories from the same sample $x_E$, with endpoints $x_1$ and
$x_1^{\gamma}$.

\paragraph{The assumption of \cite{manshausen2026accurateextremeeventlikelihoods}.} The weight is exact iff (i) $\pfwd = \pz$, or when the forward density equals the true model density, where $\Dmap = 0$, and (ii) $\qfwd = \qbwd$, where the forward re-integration of the guided path reproduces the density of the draws, i.e, $\Eg = 0$. Neither holds at finite $K$ for the \emph{explicit} Euler integrator used, as the reverse of this integrator is the \emph{implicit} Euler which is not used. 
Hence the forward density $\qfwd$ does not retrace the same trajectory as the backward process $\qbwd$.

The bias (the expectation over guided draws in
$E$) of this is given by,
\begin{equation}
\label{eq:manshausen_bias}
    \frac{\hpE}{\pE}
    \;=\; \E\!\left[\exp\big(\Ez-\Eg+\Dmap\big)\right],
\end{equation}
which is unbiased iff the right-hand side is one. This happens trivially, when $\Ez-\Eg=0$ (i.e. no guidance)  \emph{and} $\Dmap=0$, i.e.\ $\pbwd=\pz$. The non-trivial solution requires $\pbwd = \pfwd = \pz$ and $\qfwd=\qbwd$, such that $\Ez = \Eg = 0$ which occurs exactly at $K\to\infty$. 
In a non-analytical setting, $\Dmap$ is unknown and cannot be evaluated, hence, both methods assume $\pbwd = \pz$.  
At finite $K$, an accurate estimator therefore requires the three terms to cancel. 
Within the finite limit, the system becomes unbiased if $\Eg = \Ez + \Dmap$ such that individual biases cancel under addition. 
Our counterpart is a single condition, $\pbwd\to \pz$: the law of the $K$-step chain must approach the model's, or $\Dmap \to 0$. 
We too need large $K$ for this to be satisfied. In the absence of $\pz$ and $\Dmap$, our condition is trivially satisfied, whereas \cite{manshausen2026accurateextremeeventlikelihoods} requires $\Ez = \Eg$ or the cancellation of individual biases.

\subsubsection{Empirical results}
As the laws induced by the discretised ODE and SDE integrators need not coincide generally for a learned score, see \citet{deveney2025closing}, we evaluate the bias using analytical scores where the true rare-event probability $p_0[E]$ is known without any discretisation bias.  We compare these methods using the 2D GMM threshold exceedance problem at $\tau = 10$. As both methods use first order integrators, we use the exponential integrator for both ODE and SDE variants with analytical scores and Doob's guidance on the tested equal $\ell$ discretisation (see App.~\ref{app:em_bias}). In the ODE setting of \cite{manshausen2026accurateextremeeventlikelihoods}, this is equivalent to the DDIM \cite{song2020denoising, lu2022dpm, gonzalez2023seeds} integrator, but under the EDM's parametrisation \cite{karras2022elucidatingdesignspacediffusionbased} this corresponds to the Euler integrator used in \cite{manshausen2026accurateextremeeventlikelihoods}. Additionally, we clip the guidance to ensure that the crossing fraction remains ~$10\%$ for both methods (under no resampling) which mirrors the operational regime of \cite{manshausen2026accurateextremeeventlikelihoods}.  
We calibrate the clipping for both methods at each $K$ separately.
\paragraph{Remark} To isolate the effect of discretisation bias, we provide the true $\nabla \cdot s(x_t, t)$ in Eq.~\eqref{eq:their_logweight} for the PF-ODE computation of \cite{manshausen2026accurateextremeeventlikelihoods} instead of approximating it via the Hutchinson trace estimator. However, in practice the trace estimator contributes variance to the estimate, shrinking with the number of probes in exchange for a larger computational overhead.

\begin{table}[h]
\centering\small
\caption{Signed percentage bias for the rare-event probability estimated against the analytical ground truth (obtained exactly without discretisation) for the 2D GMM threshold exceedance problem at $\tau = 10$. Our method tracks a lower bias at finite discretisation $K$ compared to  \cite{manshausen2026accurateextremeeventlikelihoods}. Positive $\%$ reflects an overestimate and are presented as a percentage against the true rare-event probability.}
\begin{tabular}{llrr}
\toprule
guidance & $K$ & $\pfwd/\qfwd$ (\cite{manshausen2026accurateextremeeventlikelihoods}) $\downarrow$ & ours $\downarrow$\\
\midrule
exact Doob & 50  & $+155.7\%{\scriptstyle\pm0.3}$ & $\mathbf{-2.2\%}{\scriptstyle\pm0.7}$\\
exact Doob & 100 & $+65.3\%{\scriptstyle\pm0.2}$  & $\mathbf{-2.6\%}{\scriptstyle\pm0.4}$\\
exact Doob & 200 & $+31.7\%{\scriptstyle\pm0.2}$  & $\mathbf{-1.9\%}{\scriptstyle\pm0.3}$\\
exact Doob & 500 & $+14.7\%{\scriptstyle\pm0.1}$  & $\mathbf{-1.0\%}{\scriptstyle\pm0.2}$\\
\midrule
clipped Doob, ${\sim}10\%$ cross & 50  & $+42.8\%{\scriptstyle\pm4.2}$ & $\mathbf{+2.7\%}{\scriptstyle\pm2.5}$\\
clipped Doob, ${\sim}10\%$ cross & 100 & $+21.2\%{\scriptstyle\pm3.4}$ & $\mathbf{-2.1\%}{\scriptstyle\pm1.9}$\\
clipped Doob, ${\sim}10\%$ cross & 200 & $+10.7\%{\scriptstyle\pm3.0}$ & $\mathbf{-1.6\%}{\scriptstyle\pm1.8}$\\
clipped Doob, ${\sim}10\%$ cross & 500 & $+4.4\%{\scriptstyle\pm2.8}$  & $\mathbf{-1.9\%}{\scriptstyle\pm2.2}$\\
\bottomrule
\end{tabular}
\label{tab:us_v_manshausen_comparision}

\end{table}

Under the exact Doob's guidance, the discretisation bias from $\pbwd/\qbwd$ survives, as the guidance is employed throughout the trajectory. 
We report the discretisation bias against $K$ in Table~\ref{tab:us_v_manshausen_comparision}. 
Exact guidance highlights that this is a discretisation bias itself which converges as $\pfwd\rightarrow \pbwd$ and $\Dmap = 0$. 
For our method, any error lives in $\Dmap$. 
The discretisation bias of \cite{manshausen2026accurateextremeeventlikelihoods} is dependent on a specific choice of $\gamma$ and $h_t$, reflected in $\Eg$.  
Hence the two proposals in Table~\ref{tab:us_v_manshausen_comparision} give a different estimate for $\pE$. 
Additionally, as we adaptively resample, our crossing fraction remains high even with the clipped guidance at $86.5\%/80.4\%/81.3\%/79.4\%$ (averaged across seeds) for $K = 50/100/200/500$ respectively compared to theirs at $10\%$, which results in a higher particle population concentration on the event of interest. We need not discard $90\%$ of samples in this regime, and this is reflected in a lower seed-level standard error of our method. 

\begin{table}[h]
\centering
\caption{Network evaluations per particle. For a $K$ step denoising process, with $R$ rejuvenation steps and $M$ MCMC moves of Algorithm~\ref{alg:rne_ais}. $P$ denotes the number of Hutchinson's probes in Eq.~\eqref{eq:their_logweight}.}
\label{tab:cost}
\small
\begin{tabular}{lcccc}
\toprule
 & score function & $r$ & $\nabla r$ & JVP \\
\midrule
ours (Alg.~\ref{alg:rne_ais}) & $K{+}2RM$ & $K{+}RM$ & $K{+}RM$ & -- \\
odds ratio (Eq.~\ref{eq:their_logweight}) & $3K$ & $2K$ & $2K$ & $2PK$ \\
\bottomrule
\end{tabular}
\end{table}
\textbf{Remark on computational costs.} Let $N$ be the number of particles, $R$ the number of resampling steps, $M$ the number of MCMC move of Algorithm~\ref{alg:rne_ais}, and $P$ the number of Hutchinson trace estimator probes used per step in Eq.~\eqref{eq:their_logweight}. Both samplers pay a different number of function evaluations per particle, and have a different compute budget. Table~\ref{tab:cost} summarises this. $K$ is fixed and set by the discretisation bias tolerance versus compute. Which  has the same weak-order convergence for both methods (first-order) \cite{gonzalez2023seeds, manshausen2026accurateextremeeventlikelihoods}.

As an illustrative example, take the $\tau = 10$ example of Section~\ref{sec:experiments}, coupled with the guidance term of Eq.~\eqref{eq:lookahead_reward} and the $\Pi$GDM $\Sigma_t$ proxy \cite{song2023pseudoinverseguided}. Averaged over 24 seeds, resampling is triggered $R=3.5$ times, followed by $M = 3$ MADM rejuvenation steps. If we set $P=1$, which is a conservative estimate as typically a larger $P$ results in a lower variance estimate of the Jacobian trace. If we were to count every single score evaluation, reward and reward-gradient evaluation as one call, our estimator uses $2K + 3RM \approx 2032$ function calls per particle, of which rejuvenation moves account for $32$, whereas the odds-ratio estimator with a single probe uses $5 K + 2PK = 7000$. As a reward and its gradient calls are often paired together, at an equal step count \citet{manshausen2026accurateextremeeventlikelihoods} receive $N/3.5$ particles at the same compute budget, which is the comparison drawn in Fig~\ref{fig:table_1_replacement}.

\subsubsection{Detailed Comparison}
\label{app:detailed_comparision_fig1}

Below we collate the tables used to generate Fig.~\ref{fig:table_1_replacement}. In the first two panels, we use a learned score network, wheras the final uses an analytical score with Doob's transition kernels. We motivate why this is selected in Section~\ref{sec:exp_2dgmm}. All hyperparameters are kept the same unless they are ablated. The only exception is that each arm uses the best $\lambda(t)$ tempering parameter as selected in App.~\ref{app:gmm_schedule_ablations}. 

\paragraph{CV vs N.} We keep the same $K = 1000$ and ablate at $\tau = 10$ with learned scores. Table~\ref{tab:fig2_cv_vs_n} summarises these results, where we find that the CV of DireSMC (w/ rejuvenation) outperforms \citet{manshausen2026accurateextremeeventlikelihoods} at all $K$.
\begin{table}[h]
\centering
\small
\setlength{\tabcolsep}{5pt}
\caption{CV of $\hat p$ against particle count $N$ (Fig.~\ref{fig:table_1_replacement}, left), at $\tau=10$, $K=1000$, with the learned score, over 24 seeds. Measured against model's tail mass $p^*$ (Table~\ref{tab:step1_recovery}).}
\label{tab:fig2_cv_vs_n}
\begin{tabular}{@{}r c ccc@{}}
\toprule
 & & \multicolumn{3}{c}{DireSMC} \\
\cmidrule(l){3-5}
$N$ & \citet{manshausen2026accurateextremeeventlikelihoods} & base & + resampling & + MADM \\
\midrule
100  & 0.631 & 0.999 & 0.611 & \textbf{0.321} \\
250  & 0.385 & 0.506 & 0.394 & \textbf{0.184} \\
500  & 0.259 & 0.349 & 0.303 & \textbf{0.134} \\
750  & 0.206 & 0.266 & 0.168 & \textbf{0.082} \\
1000 & 0.167 & 0.213 & 0.166 & \textbf{0.079} \\
2000 & 0.140 & 0.184 & 0.147 & \textbf{0.074} \\
4000 & 0.094 & 0.161 & 0.076 & \textbf{0.057} \\
\bottomrule
\end{tabular}
\end{table}

\textbf{CV vs Rarity.} We keep the same $K = 1000$ and $N = 1000$ and ablate the rarities using approxiamte guidance and learned scores. We discount the number of particles available based on number of function evaluations compared to DireSMC base. For \citet{manshausen2026accurateextremeeventlikelihoods} this is $3.5\times$ less particles (286) on this problem. Table~\ref{tab:fig2_cv_vs_rarity} summarises these results, where we find that at matched compute all variants of DireSMC outperform \citet{manshausen2026accurateextremeeventlikelihoods}.

\begin{table}[h]
\centering
\small
\setlength{\tabcolsep}{6pt}
\caption{CV of $\hat{p}$ relative to models threshold exceedance probability $p^\star$ (Table~\ref{tab:step1_recovery}).
$\hat p$ across rarities at equal compute (Fig.~\ref{fig:table_1_replacement}, middle), with $K=1000$, the learned score and 24 seeds. We use $N=1000$ and \citet{manshausen2026accurateextremeeventlikelihoods} use $n=N/3.5=286$ (Table~\ref{tab:cost}). All variants of DireSMC outperform \citet{manshausen2026accurateextremeeventlikelihoods}.}
\label{tab:fig2_cv_vs_rarity}
\begin{tabular}{@{}c c c ccc@{}}
\toprule
 & & & \multicolumn{3}{c}{DireSMC} \\
\cmidrule(l){4-6}
$\tau$ & $p^\star$ & \citet{manshausen2026accurateextremeeventlikelihoods} & base & + resampling & + MADM \\
\midrule
7  & $7.13\times10^{-3}$ & 0.129 & 0.070 & 0.093 & \textbf{0.060} \\
8  & $1.38\times10^{-3}$ & 0.167 & 0.092 & 0.093 & \textbf{0.064} \\
9  & $1.87\times10^{-4}$ & 0.221 & 0.127 & 0.159 & \textbf{0.090} \\
10 & $1.72\times10^{-5}$ & 0.349 & 0.213 & 0.166 & \textbf{0.079} \\
\bottomrule
\end{tabular}
\end{table}

\textbf{Bias vs K.} We measure the discretisation bias using analytical scores, Doob's transition kernels of DireSMC with no resampling, selected to isolate the errors introduced in Eq.~\ref{eq:their_logweight}. Table~\ref{tab:fig2_bias_vs_k} summarises these results.

\begin{table}[h]
\centering
\small
\setlength{\tabcolsep}{6pt}
\caption{Discretisation bias against the step count $K$ (Fig.~\ref{fig:table_1_replacement}, right), shown as the mean $\hat p/p^\star \pm$ s.e.\ over 24 seeds. Computed using analytical socres, Doob's transition kernels and withut resampling at $N = 1000$. $p^\star$ is $\pdata[E]$. \citet{manshausen2026accurateextremeeventlikelihoods} are given the exact divergence.}
\label{tab:fig2_bias_vs_k}
\begin{tabular}{@{}c c r cc@{}}
\toprule
$\tau$ & $p^\star$ & $K$ & \citet{manshausen2026accurateextremeeventlikelihoods} & DireSMC \\
\midrule
\multirow{4}{*}{7}  & \multirow{4}{*}{$6.88\times10^{-3}$} & 50  & $1.466 \pm 0.001$ & $\mathbf{0.950} \pm 0.006$ \\
                    &                                      & 100 & $1.233 \pm 0.001$ & $\mathbf{0.972} \pm 0.003$ \\
                    &                                      & 200 & $1.128 \pm 0.001$ & $\mathbf{0.981} \pm 0.003$ \\
                    &                                      & 500 & $1.069 \pm 0.001$ & $\mathbf{0.990} \pm 0.002$ \\
\addlinespace
\multirow{4}{*}{8}  & \multirow{4}{*}{$1.28\times10^{-3}$} & 50  & $1.701 \pm 0.002$ & $\mathbf{0.957} \pm 0.005$ \\
                    &                                      & 100 & $1.336 \pm 0.001$ & $\mathbf{0.970} \pm 0.003$ \\
                    &                                      & 200 & $1.178 \pm 0.001$ & $\mathbf{0.980} \pm 0.003$ \\
                    &                                      & 500 & $1.091 \pm 0.001$ & $\mathbf{0.990} \pm 0.002$ \\
\addlinespace
\multirow{4}{*}{9}  & \multirow{4}{*}{$1.62\times10^{-4}$} & 50  & $2.049 \pm 0.002$ & $\mathbf{0.963} \pm 0.006$ \\
                    &                                      & 100 & $1.473 \pm 0.002$ & $\mathbf{0.972} \pm 0.004$ \\
                    &                                      & 200 & $1.240 \pm 0.001$ & $\mathbf{0.981} \pm 0.003$ \\
                    &                                      & 500 & $1.117 \pm 0.001$ & $\mathbf{0.990} \pm 0.002$ \\
\addlinespace
\multirow{4}{*}{10} & \multirow{4}{*}{$1.40\times10^{-5}$} & 50  & $2.557 \pm 0.003$ & $\mathbf{0.978} \pm 0.007$ \\
                    &                                      & 100 & $1.653 \pm 0.002$ & $\mathbf{0.974} \pm 0.004$ \\
                    &                                      & 200 & $1.317 \pm 0.002$ & $\mathbf{0.981} \pm 0.003$ \\
                    &                                      & 500 & $1.147 \pm 0.001$ & $\mathbf{0.990} \pm 0.002$ \\
\bottomrule
\end{tabular}
\end{table}

\end{document}